\documentclass{article}
\usepackage[accepted]{styles/icml2021/icml2021}

\usepackage{times}
\usepackage{booktabs}       % professional-quality tables
\usepackage{amsfonts}       % blackboard math symbols
\usepackage{nicefrac}       % compact symbols for 1/2, etc.
\usepackage{enumitem}
\usepackage{nicefrac}

\usepackage{amsmath,amssymb,mathrsfs}

\usepackage{footnote}
\usepackage{algorithm}
\usepackage[algo2e, linesnumbered, vlined,ruled]{algorithm2e}
\usepackage{algorithmic}

\usepackage{amsthm}

\allowdisplaybreaks

\makeatletter
\def\set@curr@file#1{\def\@curr@file{#1}} %temp workaround for 2019 latex release; otherwise the command \includegraphics clashes...
\makeatother

\usepackage{graphicx}
\usepackage{mathtools}
\usepackage{footnote}
\usepackage{float}
\usepackage{xspace}
\usepackage{multirow}
\usepackage{wrapfig}
\usepackage{framed}
\usepackage{bbm}
\usepackage{footnote}
\usepackage{nicefrac}
\usepackage{makecell}
\usepackage[T1]{fontenc}
\makesavenoteenv{tabular}
\makesavenoteenv{table}

\definecolor{Green}{rgb}{0.13, 0.65, 0.3}
\usepackage[]{color-edits}

\addauthor{HL}{red}
\addauthor{LC}{cyan}

\newcommand{\calA}{{\mathcal{A}}}
\newcommand{\calB}{{\mathcal{B}}}

\newcommand{\calX}{{\mathcal{X}}}
\newcommand{\calS}{{\mathcal{S}}}
\newcommand{\calF}{{\mathcal{F}}}
\newcommand{\calI}{{\mathcal{I}}}

\newcommand{\calP}{{\mathcal{P}}}

\newcommand{\V}{\text{\rm Var}}

\DeclareMathOperator*{\argmin}{argmin}
\DeclareMathOperator*{\argmax}{argmax}
\DeclarePairedDelimiter\abs{\lvert}{\rvert}

\newcommand{\eat}[1]{}

\newcommand{\inner}[2]{\left\langle #1, #2 \right\rangle}
\newcommand{\inners}[2]{\langle #1, #2\rangle}

\newcommand{\rbr}[1]{\left(#1\right)}
\newcommand{\sbr}[1]{\left[#1\right]}
\newcommand{\cbr}[1]{\left\{#1\right\}}

\newcommand{\abr}[1]{\left|#1\right|}
\newcommand{\bigO}[1]{\order\left( #1 \right)}
\newcommand{\tilO}[1]{\otil\left( #1 \right)}
\newcommand{\lowO}[1]{\lorder\left( #1 \right)}

\newcommand{\tilo}[1]{\otil( #1 )}
\newcommand{\lowo}[1]{\lorder( #1 )}
\DeclarePairedDelimiter\ceil{\lceil}{\rceil}

\newcommand{\Tmax}{\ensuremath{T_{\max}}}

\newcommand{\T}{\ensuremath{T_\star}}
\newcommand{\cmin}{\ensuremath{c_{\min}}}
\newcommand{\propers}{\ensuremath{\Pi_{\rm proper}}}

\newcommand{\reg}{\textsc{Reg}}
\newcommand{\bias}{\textsc{Bias}}

\newcommand{\err}{\textsc{Err}}

\newcommand{\aux}{\chi}

\newcommand{\PSD}{\propers}
\newcommand{\SA}{\Gamma}
\newcommand{\qfeat}{\phi} % (1+\lambda h)q

\newcommand{\regz}{\psi} % regularizer

\newcommand{\sums}{\sum_{s\in\calS}}
\newcommand{\suma}{\sum_{a\in\calA_s}}

\newcommand{\sumtilsaf}[1][s, a]{\sum_{(#1)\in\tilSA}}
\newcommand{\sumsah}{\sum_{(s, a), h}}
\newcommand{\sumsa}[1][s, a]{\sum_{(#1)}}

\newcommand{\sumh}{\sum_{h=1}^H}

\newcommand{\sumk}{\sum_{k=1}^K}

\newcommand{\hatb}{\widehat{b}}
\newcommand{\hatc}{\widehat{c}}

\newcommand{\hatq}{\widehat{q}}
\newcommand{\optpi}{\pi^\star}
\newcommand{\optq}{q_{\optpi}}
\newcommand{\tiloptq}{q_{\tiloptpi}}

\newcommand{\tilN}{\widetilde{N}}
\newcommand{\tilR}{\widetilde{R}}

\newcommand{\hatP}{\widehat{P}}
\newcommand{\tils}{\widetilde{s}}
\newcommand{\tila}{\widetilde{a}}
\newcommand{\tilc}{\widetilde{c}}
\newcommand{\tilf}{\widetilde{f}}

\newcommand{\tilM}{\widetilde{M}}
\newcommand{\tilP}{\widetilde{P}}
\newcommand{\tilA}{\widetilde{\calA}}
\newcommand{\tilS}{\widetilde{\calS}}
\newcommand{\tilSA}{\widetilde{\SA}}
\newcommand{\tilD}{\widetilde{D}}
\newcommand{\tilpi}{{\widetilde{\pi}}}
\newcommand{\tiloptpi}{{\widetilde{\pi}^\star}}
\newcommand{\tilDelta}{{\widetilde{\Delta}}}
\newcommand{\tilPi}{\widetilde{\Pi}}
\newcommand{\N}{\mathbf{N}} % counter N
\newcommand{\M}{\mathbf{M}}
\newcommand{\Nk}{\N^+_k}

\newcommand{\Nc}{\N^c_k}
\newcommand{\Mk}{\M_k}

\newcommand{\h}[1]{\vec{h}\circ #1}
\newcommand{\unk}{U} % unknown indices
\newcommand{\A}{A}

\newcommand{\Ak}{\A_k}
\newcommand{\lfs}{\calX} % loop-free state space
\newcommand{\tcs}{\Lambda_P} % transition confidence set

\newcommand{\RUN}{\text{RUN}}
\newcommand{\TransEst}{\text{TransEst}}

\newcommand{\Qdtd}{H^3S^2A} % double transition difference
\newcommand{\Qumq}{H^3S^3A^2} % u - q

\newcommand{\field}[1]{\mathbb{#1}}

\newcommand{\fR}{\field{R}}

\newcommand{\fN}{\field{N}}
\newcommand{\E}{\field{E}}
\newcommand{\fV}{\field{V}}

\newcommand{\Ind}{\field{I}}

\newcommand{\defeq}{\stackrel{\rm def}{=}}

\newtheorem{lemma}{Lemma}
\newtheorem{theorem}{Theorem}

\newtheorem{definition}{Definition}

\newcommand{\order}{\ensuremath{\mathcal{O}}}
\newcommand{\lorder}{\ensuremath{\Omega}}
\newcommand{\otil}{\ensuremath{\tilde{\mathcal{O}}}}

\usepackage{prettyref}
\newcommand{\pref}[1]{\prettyref{#1}}
\newcommand{\pfref}[1]{Proof of \prettyref{#1}}
\newcommand{\savehyperref}[2]{\texorpdfstring{\hyperref[#1]{#2}}{#2}}
\newrefformat{eq}{\savehyperref{#1}{Eq.~\textup{(\ref*{#1})}}}
\newrefformat{eqn}{\savehyperref{#1}{Equation~\ref*{#1}}}
\newrefformat{lem}{\savehyperref{#1}{Lemma~\ref*{#1}}}
\newrefformat{def}{\savehyperref{#1}{Definition~\ref*{#1}}}
\newrefformat{line}{\savehyperref{#1}{Line~\ref*{#1}}}
\newrefformat{thm}{\savehyperref{#1}{Theorem~\ref*{#1}}}
\newrefformat{corr}{\savehyperref{#1}{Corollary~\ref*{#1}}}
\newrefformat{cor}{\savehyperref{#1}{Corollary~\ref*{#1}}}
\newrefformat{sec}{\savehyperref{#1}{Section~\ref*{#1}}}
\newrefformat{subsec}{\savehyperref{#1}{Section~\ref*{#1}}}
\newrefformat{app}{\savehyperref{#1}{Appendix~\ref*{#1}}}
\newrefformat{assum}{\savehyperref{#1}{Assumption~\ref*{#1}}}
\newrefformat{ex}{\savehyperref{#1}{Example~\ref*{#1}}}
\newrefformat{fig}{\savehyperref{#1}{Figure~\ref*{#1}}}
\newrefformat{alg}{\savehyperref{#1}{Algorithm~\ref*{#1}}}
\newrefformat{rem}{\savehyperref{#1}{Remark~\ref*{#1}}}
\newrefformat{conj}{\savehyperref{#1}{Conjecture~\ref*{#1}}}
\newrefformat{prop}{\savehyperref{#1}{Proposition~\ref*{#1}}}
\newrefformat{proto}{\savehyperref{#1}{Protocol~\ref*{#1}}}
\newrefformat{prob}{\savehyperref{#1}{Problem~\ref*{#1}}}
\newrefformat{claim}{\savehyperref{#1}{Claim~\ref*{#1}}}
\newrefformat{que}{\savehyperref{#1}{Question~\ref*{#1}}}
\newrefformat{op}{\savehyperref{#1}{Open Problem~\ref*{#1}}}
\newrefformat{fn}{\savehyperref{#1}{Footnote~\ref*{#1}}}
\newrefformat{tab}{\savehyperref{#1}{Table~\ref*{#1}}}
\newrefformat{fig}{\savehyperref{#1}{Figure~\ref*{#1}}}

\DeclareOldFontCommand{\rm}{\normalfont\rmfamily}{\mathrm}
\DeclareOldFontCommand{\it}{\normalfont\itshape}{\mathit}
\icmltitlerunning{Finding the Stochastic Shortest Path with Low Regret}
\usepackage{times}

\usepackage{microtype}
\usepackage{graphicx}
\usepackage{subfigure}
\usepackage{booktabs} % for professional tables

\usepackage[colorlinks=true, linkcolor=blue, citecolor=blue, breaklinks=true]{hyperref}

\begin{document}
%\SetAlgoVlined
%\DontPrintSemicolon
%\maketitle

\twocolumn[
\icmltitle{Finding the Stochastic Shortest Path with Low Regret: \\ The Adversarial Cost and Unknown Transition Case}



\icmlsetsymbol{equal}{*}

\begin{icmlauthorlist}
\icmlauthor{Liyu Chen}{usc}
\icmlauthor{Haipeng Luo}{usc}
%\icmlauthor{Chen-Yu Wei}{usc}
\end{icmlauthorlist}

\icmlaffiliation{usc}{University of Southern California}

\icmlcorrespondingauthor{Liyu Chen}{liyuc@usc.edu}

\icmlkeywords{Machine Learning, ICML}

\vskip 0.3in
]



\printAffiliationsAndNotice{}  % leave blank if no need to mention equal contribution

\begin{abstract}
We make significant progress toward the stochastic shortest path problem with adversarial costs and unknown transition. 
Specifically, we develop algorithms that achieve $\tilo{\sqrt{S^2AD\T K}}$ regret for the full-information setting and $\tilo{\sqrt{S^3A^2D\T K}}$ regret for the bandit feedback setting, where $D$ is the diameter, $\T$ is the expected hitting time of the optimal policy, $S$ is the number of states, $A$ is the number of actions, and $K$ is the number of episodes.
Our work strictly improves~\citep{rosenberg2020adversarial} in the full information setting, 
extends~\citep{chen2020minimax} from known transition to unknown transition, 
and is also the first to consider the most challenging combination: bandit feedback with adversarial costs and unknown transition.
% Our algorithms are built on top of previous work~\cite{chen2020minimax} that develop minimax regret in the known transition setting. 
To remedy the gap between our upper bounds and the current best lower bounds constructed via a stochastically oblivious adversary,
we also propose algorithms with near-optimal regret for this special case.
\end{abstract}

\section{Introduction}\label{sec:intro}
% !TEX root = main.tex

\renewcommand{\arraystretch}{1.3}
\begin{table*}[t]
	\centering
	\caption{Summary of our results. Here, $D, S, A$ are the diameter, number of states, and number of actions of the MDP, $\T$ is the expected hitting time of the optimal policy, and $K$ is the number of episodes.
	All algorithms can be implemented efficiently.
	Our results strictly improve that of \citep{rosenberg2020adversarial} in the full information setting, and are the first to consider the bandit setting with unknown transition. % as well as stochastic costs in both settings.
	Lower bounds here are a direct combination of lower bounds for stochastic costs and known transition~\citep{chen2020minimax} %(which is $\lowo{\sqrt{D\T K}}$ for the full information setting and $\lowo{\sqrt{SAD\T K}}$ for the bandit feedback setting) 
	and the lower bound for fixed costs and unknown transition~\citep{cohen2020near}. %(which is $\lowo{D\sqrt{SAK}}$).
	%We also assume $D\leq\T$ without loss of generality.
	}
	\label{tab:unknown}
	\vspace{5pt}
%	\resizebox{\textwidth}{!}{%
	\begin{tabular}{| c | c | c | c | }
		\hline
		 & Adversarial costs & Stochastic costs (\pref{thm:iid}) & Lower bounds \\
		 \hline
		 Full information & $\tilo{\sqrt{S^2AD\T K}}$ (\pref{thm:full-unknown}) & $\tilo{\sqrt{D\T K}+DS\sqrt{AK}}$ & $\lowo{ \sqrt{D\T K} + D\sqrt{SAK} }$ \\
		 \hline
		 Bandit feedback & $\tilo{\sqrt{S^3A^2D\T K}}$  (\pref{thm:bandit-unknown}) & $\tilo{\sqrt{SAD\T K}+DS\sqrt{AK}}$ & $\lowo{ \sqrt{SAD\T K} + D\sqrt{SAK} }$ \\
		 \hline
	\end{tabular}
%	}
\end{table*}

We study the stochastic shortest path (SSP) problem, where a learner aims to find the goal state with minimum total cost.
The environment dynamics are modeled as a Markov Decision Process (MDP) with $S$ states, $A$ actions, and a fixed and unknown transition function.
The learning proceeds in $K$ episodes,
where in each episode, starting from a fixed initial state, the learner sequentially selects an action, incurs a cost, and transits to the next state sampled from the transition function.
The episode ends when the learner reaches a fixed goal state.
We focus on regret minimization in SSP and measure the performance of the learner by the difference between her total cost over the $K$ episodes and that of the best fixed policy in hindsight.

The special case of SSP where an episode is guaranteed to end within a fixed number of steps is extensively studied in recent years (often known as episodic finite-horizon reinforcement learning or loop-free SSP).
The general (and also the more practical) case, on the other hand, has only been recently studied:
\citet{tarbouriech2019no} and~\citet{cohen2020near} develop algorithms with sub-linear regret for the case with fixed or i.i.d. costs.
Adversarial costs is later studied by
\citet{rosenberg2020adversarial}
in the full-information setting (where the cost is revealed at the end of each episode). %Their algorithm achieves $\tilo{\frac{D}{\cmin}S\sqrt{AK}}$ regret where $D$ is the diameter of the MDP and $\cmin\in(0,1]$ is a global lower bound of the cost for any state-action pair.
%When $\cmin=0$, they achieve a weaker regret $\tilo{\T S\sqrt{A}K^{3/4}+D^2\sqrt{K}}$, where $\T$ is the expected time for the optimal policy to reach the goal state.
The minimax regret for adversarial costs and known transition is then fully characterized in a recent work by \citet{chen2020minimax}, in both the full-information setting and the bandit feedback setting (where only the cost of visited state-action pairs is revealed).

%develops algorithms achieving $\tilo{\sqrt{D\T K}}$ regret for the full information setting and $\tilo{\sqrt{SAD\T K}}$ regret for the bandit feedback setting (where the agent only observes cost of visited state-action pairs) assuming that transition function is known.
%Both bounds are minimax optimal. 

In this work, we further extend our understanding of general SSP with adversarial costs and unknown transition, for both the full-information setting and the bandit setting.
More specifically, our results are (see also \pref{tab:unknown}):
\begin{itemize}
\item
(\pref{sec:full-unknown}) In the full-information setting, we develop an algorithm that achieves $\tilo{\sqrt{S^2AD\T K}}$ regret with high probability, where $D$ is the diameter of the MDP and $\T$ is the expected time for the optimal policy to reach the goal state.
This improves over the best existing bound $\tilo{\frac{1}{\cmin}\sqrt{S^2AD^2K}}$ or $\tilo{\sqrt{S^2A\T^2}K^{3/4}+D^2\sqrt{K}}$ from~\citep{rosenberg2020adversarial},
where $\cmin\in[0,1]$ is a global lower bound of the cost for any state-action pair (it can be shown that $\T \leq D/\cmin$).
%Under adaptive adversaries, we develop algorithms that achieve high probability regret bound of order $\tilo{\sqrt{S^2AD\T K} + DS\sqrt{AK} }$ and $\tilo{\sqrt{S^3A^2D\T K} + DS\sqrt{AK}}$ (ignoring logarithmic terms) for the full information setting and the bandit feedback setting respectively, with no dependence on $1/\cmin$ (it can be shown that $\T \leq D/\cmin$).

\item
(\pref{sec:bandit-unknown})
In the bandit setting, we develop another algorithm that achieves $\tilo{\sqrt{S^3A^2D\T K}}$ regret with high probability, 
which, as far as we know, is the first result for this most challenging setting (bandit feedback, adversarial costs, and unknown transition).

\item
(\pref{sec:iid})
By combining previous results, it can be shown that the lower bound for the full-information and the bandit setting are $\lowo{ \sqrt{D\T K} + D\sqrt{SAK}}$ and $\lowo{ \sqrt{SAD\T K} + D\sqrt{SAK}}$ respectively, 
establishing a gap from our upper bounds.
Noting that these lower bounds are constructed with a stochastically oblivious adversary,
we propose another algorithm for this special case with near-optimal regret bounds that are only $\sqrt{S}$ factor larger than the lower bounds, a gap that is still open even for loop-free SSP~\citep{rosenberg2019online,jin2019learning}.
Note that this setting is slightly different from and harder than existing i.i.d. cost settings; see discussions in ``Related work'' below.

%\item
%Next, we consider a weaker oblivious adversary whose cost functions are i.i.d. samples from a distribution, a type of adversary commonly used in proving lower bounds.
%We show that leveraging the i.i.d. assumption, it is possible to have one simple algorithm that achieves near optimal regret for both the full information setting and the bandit feedback setting.
%The gap between the upper bounds and lower bounds is only a $\sqrt{S}$ factor, which already exists in the simpler episodic loop-free SSP setting~\citep{rosenberg2019online,jin2019learning}.
\end{itemize}

\paragraph{Technical contributions}
Our algorithms are largely based on those from the recent work of~\citep{chen2020minimax} for the known transition setting.
However, learning with unknown transition and carefully controlling the transition estimation error requires several new ideas.
First, we extend the loop-free reduction of \citep{chen2020minimax} to the unknown transition setting (\pref{sec:reduction}).
Then, combining a Bellman type law of total variance~\citep{azar2017minimax} and a linear form of the variance of actual costs,  we show that, importantly, the bias introduced by transition estimation is well controlled via the so-called skewed occupancy measure proposed by~\citet{chen2020minimax}.
This leads to our algorithm for the full information setting.
For the bandit setting, apart from the techniques above and those from~\citep{chen2020minimax}, we further propose and utilize two optimistic cost estimators inspired by the idea of upper occupancy bounds from~\citet{jin2019learning} for loop-free SSP.

Finally, for the weaker stochastically oblivious adversaries, 
we further augment the loop-free reduction to allow the learner to switch to a fast policy at any time step if necessary, which is crucial to ensure the near-optimal regret for our simple optimism-based algorithm.

\paragraph{Related work}
The SSP problem was studied earlier mostly from the control aspect where the goal is to find the optimal policy efficiently with all parameters known~\citep{bertsekas1991analysis,bertsekas2013stochastic}.
Regret minimization in SSP was first studied in~\citep{tarbouriech2019no, cohen2020near}, with fixed and known costs and unknown transition.
Although their results can be generalized to i.i.d. costs as discussed in~\citep[Appendix I.1]{tarbouriech2019no},
this is in fact different from our stochastic cost setting.
Indeed, in their setting, the cost of each state-action pair is drawn (independently of other pairs and other episodes) every time it is visited, and is revealed to the learner immediately.
On the other hand, in our stochastic setting, the costs of all state-action pairs in each episode are jointly drawn from a fixed distribution (independently of other episodes; but costs of different pairs could be correlated) and fixed throughout the episode,
and any information about the costs is only revealed after the episode ends.
As argued in~\citep[Section 3.1]{chen2020minimax}, our setting is information-theoretically harder as an extra dependence on $\T$ is unavoidable here,
and thus our bounds for stochastic costs are incomparable to these two works.
To distinguish these two different settings, we sometime refer to ours as a setting with a stochastically oblivious adversary.

\citep{rosenberg2020adversarial} is the first work that studies SSP with adversarial costs with either known or unknown transition, but only in the full-information setting.
Later, \citep{chen2020minimax} develops efficient and minimax optimal algorithms for both the full-information setting and the bandit feedback setting, but only with known transition.
As mentioned, our results significantly improve and extend these two works.
One of the key technical contributions of~\citep{chen2020minimax} is the loop-free reduction, which, as discussed by the authors, is readily applied to the unknown transition case, but leads to suboptimal bounds with unnecessary dependence on other parameters if applied directly.
%Specifically, directly applying loop-free reduction to the unknown transition case results in a bound with dependency on parameters not appeared in the lower bound (similar to $\cmin$). 
Our algorithms are built on top of an extension of this loop-free reduction, and we overcome the technical difficulty they run into via a more careful analysis showing that the transition estimation error can in fact be well controlled using their idea of skewed occupancy measure.

%and achieve strictly improvement compared to \citep{rosenberg2020adversarial} and \citep{chen2020minimax}.
%We also develop an algorithm that achieves near-optimal bounds assuming that cost functions are i.i.d. samples from a distribution.

As mentioned, the special case of loop-free SSP has been extensively studied in recent years, for both fixed or i.i.d. costs (see e.g., \citep{azar2017minimax,jin2018q,zanette2019tighter,efroni2020optimistic}) and adversarial costs (see e.g.,  \citep{neu2012adversarial,zimin2013online,rosenberg2019online,jin2019learning,shani2020optimistic,cai2020provably}).
In particular, the idea of upper occupancy bound from~\citep{jin2019learning}, used to construct an optimistic cost estimator with a confidence set of the transition, is also one key technique we adopt in the bandit setting.

\section{Preliminaries}\label{sec:prelim}
% !TEX root = main.tex

We largely follow the notations of~\citep{chen2020minimax}.
An SSP instance consists of an MDP $M=\left( \calS, s_0, g, \calA, P \right)$ and a sequence of $K$ cost functions $\{c_k\}_{k=1}^K$.
Here, $\calS$ is a finite state space, $s_0 \in \calS$ is the initial state,
$g \notin \calS$ is the goal state,
$\calA=\{\calA_s\}_{s\in\calS}$ is a finite action space where $\calA_s$ is the available action set at state $s$.
Let $\SA = \{(s,a): s\in \calS, a \in \calA_s\}$ be the set of all valid state-action pairs.
The transition function $P: \SA\times(\calS\cup\{g\})\rightarrow [0, 1]$ is such that $P(s'| s, a)$ specifies the probability of transiting to the next state $s'$ after taking action $a \in \calA_s$ at state $s$, and we have $\sum_{s'\in\calS\cup\{g\}}P(s'|s, a)=1$ for each $(s,a) \in \SA$. 
Finally, $c_k: \SA \rightarrow [0,1]$ is the cost function that specifies the cost for each state-action pair during episode $k$.
We denote by $S=\lvert\calS\rvert$ and $A=(\sums |\calA_s|)/S$ the total number of states and the average number of available actions respectively.

The learner interacts with the MDP through $K$ episodes, not knowing the transition function $P$ nor the cost functions $\{c_k\}_{k=1}^K$ ahead of time.
In each episode $k = 1, \ldots, K$,
the adversary first decides the cost function $c_k$, which,
for the majority of this work, can depend on the learner's algorithm and the randomness before episode $k$ in an arbitrary way (known as an adaptive adversary).
Only in \pref{sec:iid}, we switch to a weaker stochastically oblivious adversary who draws $c_k$ independently from a fixed but unknown distribution.
In any case, without knowing $c_k$,  the learner decides which action to take in each step of the episode, starting from the initial state $s_0$ and ending at the goal state $g$.
More precisely, in each step $i$ of episode $k$, the learner observes its current state $s_k^i$ (with $s_k^1 = s_0$).
If $s_k^i \neq g$, the learner selects an action $a_k^i \in \calA_{s_k^i}$ and transits to the next state $s_k^{i+1}$ sampled from $P(\cdot| s_k^i, a_k^i)$;
otherwise, the episode ends, and we let $I_k$ be the number of steps in this episode such that $s_k^{I_k+1}= g$.

After each episode $k$ ends, the learner receives some feedback on the cost function $c_k$.
In the {\it full-information} setting, the learner observes the entire $c_k$,
while in the more challenging {\it bandit feedback} setting, the learner only observes the costs of the visited state-action pairs, that is, $c_k(s_k^i, a_k^i)$ for $i = 1, \ldots, I_k$.

\paragraph{Important concepts}
We introduce several necessary concepts before discussing the goal of the learner.
A stationary policy is a mapping $\pi$ such that $\pi(a|s)$ specifies the probability of taking action $a \in \calA_s$ in state $s$.
It is deterministic if for all $s$, $\pi(a|s)=1$ holds for some action $a$ (in which case we write $\pi(s)=a$). 
A policy is {\it proper} if executing it in the MDP starting from any state ensures that the goal state is reached within a finite number of steps with probability $1$ (and improper otherwise).
We denote by $\propers$ the set of all deterministic and proper policies, and make the basic assumption $\propers\neq \emptyset$ following~\citep{rosenberg2020adversarial,chen2020minimax}.

We denote by $T^{\pi}(s)$ the expected hitting time it takes for a stationary policy $\pi$ to reach $g$ starting from state $s$.
%If $\pi$ is proper, then $T^{\pi}(s)<\infty$ for any state $s$.
The \textit{fast policy} $\pi^f$ is a deterministic policy that achieves the minimum expected hitting time starting from any state (among all stationary policies).
The diameter of the MDP is defined as $D=\max_{s\in\calS}\min_{\pi\in\PSD}T^{\pi}(s)=
\max_{s\in\calS} T^{\pi^f}(s)$, which is the ``largest shortest distance'' between any state and the goal state.

Given a transition function $P$, a cost function $c$, and a proper policy $\pi$, we define the cost-to-go function $J^{P, \pi, c}:\calS\rightarrow [0,\infty)$ such that 
$J^{P, \pi, c}(s)=\E\left[\left.\sum_{i=1}^{I}c(s^i, a^i)\right| P, \pi, s^1=s\right]$,
where the expectation is over the randomness of the action $a^i$ drawn from $\pi(\cdot|s^i)$, the state $s^{i+1}$ drawn from $P(\cdot|s^i, a^i)$, and the number of steps $I$ before reaching $g$.
Similarly, we also define the state-action value function $Q^{P,\pi,c}: \SA\rightarrow [0,\infty)$ such that $Q^{P, \pi, c}(s, a)=\E\left[\left.\sum_{i=1}^{I}c(s^i, a^i)\right| P, \pi, s^1=s, a^1=a\right]$.
We use $J_k^{P, \pi}$ and $Q_k^{P, \pi}$ to denote the cost-to-go and state-action function with respect to the cost $c_k$.
When there is no confusion, we also ignore the dependency on the transition function (especially when $P$ is the true transition function of the MDP) and write $J^{P, \pi, c}$ as $J^{\pi, c}$, $J_k^{P, \pi}$ as $J_k^{\pi}$, $Q^{P, \pi, c}$ as $Q^{\pi, c}$, and $Q_k^{P, \pi}$ as $Q_k^{\pi}$.

\paragraph{Learning objective}
The learner's goal is to minimize her {\it regret}, defined as the difference between her total cost and the total expected cost of the best deterministic proper policy in hindsight:
\begin{align*}
	R_K = \sumk\sum_{i=1}^{I_k}c_k(s_k^i, a_k^i) - \sumk J^{\optpi}_k(s_0),
\end{align*}
where $\optpi\in\argmin_{\pi\in\PSD}\sumk J^{\pi}_k(s_0)$ is the optimal stationary and proper policy, which is referred to as optimal policy in the rest of the paper. %If there is an episode $k$ such that $I_k=\infty$, we define $R_K=\infty$.
By the Markov property, $\optpi$ is in fact also the optimal policy starting from any other state, that is, $\optpi\in\argmin_{\pi\in\PSD}\sumk J^{\pi}_k(s)$ for any $s \in \calS$.
As in~\citep{chen2020minimax},
the following two quantities related to $\optpi$ play an important role:
its expected hitting time starting from the initial state $\T = T^{\optpi}(s_0)$
and its largest expected hitting time starting from any state $\Tmax = \max_s T^{\optpi}(s)$.
\citet{chen2020minimax} show that $\Tmax \leq \frac{D}{\cmin}$ where 
$\cmin = \min_{k}\min_{(s,a)} c_k(s,a)$ is the minimum possible cost. 
%and $\smax \in \calS$ be such that $\Tmax = T^{\optpi}(\smax)$.
%Then, we have $K\Tmax\cmin \leq \sumk J^{\optpi}_k(\smax) \leq \sumk J^{\pi^f}_k(\smax) \leq KD$.
%We have $\Tmax \cmin \leq J^{\optpi}_k(\smax)$ and $J^{\pi^f}_k(\smax) \leq D$ by definition.
%Together with the fact $\sumk J^{\optpi}_k(\smax) \leq \sumk J^{\pi^f}_k(\smax)$,
%This implies $\T \leq \Tmax \leq \frac{D}{\cmin}$ if $\cmin > 0$ (which is one of the reasons why $\cmin$ shows up in existing results).
For ease of presentation, we assume $D\leq \T$ to simplify our bounds.

\paragraph{Knowledge on key parameters}
Our algorithms require the knowledge of $\T$ and $\Tmax$, similarly to most algorithms of~\citep{chen2020minimax}.
This requirement is seemingly restrictive, especially when against an adaptive adversary, in which case $\T$ and $\Tmax$ depend on the behavior of both the algorithm itself and the adversary.
However, we argue that our results are still meaningful:
First, for an oblivious adversary, $\T$ and $\Tmax$ are fixed unknown quantity independent of the learner's behavior.
Many works in online learning indeed start with assuming knowledge on such quantities to get a better understanding of the problem (and to tune these
hyperparameters empirically), before one can eventually develop a fully parameter-free algorithm. Thus, as the first step, we believe that our work is still valuable.
%Many works in online learning indeed start with such an assumption to get a better understanding of the problem;
Second, in the lower bound construction~\cite{chen2020minimax}, $\T$ is also known to the learner, meaning that knowing $\T$ does not make the problem any easier information-theoretically.
Finally, to emphasize the difficulty of removing this requirement, we note that this is still open even with known transition when considering high-probability bounds.
\citet{chen2020minimax} were able to resolve this for expected regret bounds, but extending their techniques to high-probability bounds is related to deriving
a high-probability bound for the so-called multi-scale expert problem, which is also still open~\citep[Appendix A]{chen2021impossible}.

On the other hand, we also emphasize that our main improvement compared to \citep{rosenberg2020adversarial} is not due to the knowledge of these parameters.
Indeed, under the same setup where these parameters are unknown, we can still run our algorithms by replacing $\T$ with its upper bound $D/\cmin$ and $\Tmax$ with some lower order term $o(K)$, and this still leads to better results compared to~\citep{rosenberg2020adversarial}.
Details are deferred to \pref{app:tune T}.

Finally, for simplicity, we also assume that $D$ is known, but our results can be extended even if $D$ is unknown; see \pref{app:tune-d}.

\paragraph{Occupancy measure}
% Given a state space $\calS$, action space $\calA$ and the set of state-action pairs $\SA$,
Occupancy measure plays a key role in solving SSP with adversarial costs, in both the loop-free case~\citep{neu2012adversarial,zimin2013online,rosenberg2019online,jin2019learning} and the general case~\citep{rosenberg2020adversarial, chen2020minimax}.
A proper stationary policy $\pi$ and a transition function $P$ induce an occupancy measure $q_{P, \pi} \in \fR_{\geq 0}^{\SA\times(\calS\cup \{g\})}$ such that $q_{P, \pi}(s, a, s')$ is the expected number of visits to state-action-afterstate triplet $(s, a, s')$ when executing $\pi$ in an MDP with transition $P$, that is:
$
	q_{P, \pi}(s, a, s') = \E\left[\left. \sum_{i=1}^I \Ind\{s^i=s, a^i=a, s^{i+1}=s' \}\right| P, \pi, s^1=s_0 \right].
$
When $P$ is clear from the context (which is usually the case if it is the true transition), we omit the $P$ dependence and only write $q_\pi$.
We also let $q_{\pi}(s, a)=\sum_{s'}q_{\pi}(s, a, s')$ be the expected number of visits to  state-action pair $(s, a)$ and $q_{\pi}(s) = \suma q_{\pi}(s, a)$ be the expected number of visits to state $s$ when executing $\pi$.
%When there is no confusion, we ignore the dependency on $P$ and write $q_{P, \pi}$ as $q_{\pi}$.
Note that, given a function $q: \SA\times(\calS\cup \{g\}) \rightarrow [0, \infty)$, if it corresponds to an occupancy measure, then the corresponding policy $\pi_q$ can be obtained via $\pi_q(a| s) \propto q(s,a)$, and the corresponding transition function can be obtained via $P_q(s'|s, a)\propto q(s, a, s')$.
Also note that $T^{\pi}(s_0) = \sumsa q_\pi(s,a) = \sums q_\pi(s)$.

Occupancy measures allow one to turn the problem into a form of online linear optimization where Online Mirror Descent is a standard tool.
Indeed, we have $J^{\pi}_k(s_0) = \sum_{(s,a)\in\SA} q_{\pi}(s, a)c_k(s, a) = \inner{q_{\pi}}{c_k}$,
and if the learner executes a stationary proper policy $\pi_k$ in episode $k$, then the expected regret can be written as 
%\begin{align}\label{eq:regret_linear_form}
$\E[R_K] = \E\left[\sumk J^{\pi_k}_k(s_0) - J^{\optpi}_k(s_0) \right]
= \E\left[ \sumk\inner{q_{\pi_k} - q_{\optpi}}{c_k} \right]$,
%\end{align}
exactly in the form of online linear optimization.
%This converts the problem into a form of online linear optimization and makes Online Mirror Descent a natural solution to the problem.
%Our algorithms all make use of this framework, but differ in the set of occupancy measures that the algorithm operates on.

%Note that if $\pi$ is proper, then $q_{\pi}(s, a)<\infty,\forall s, a$.
%Moreover, the mapping between proper policies and finite occupancy measures is bijective, and its inverse for an occupancy measure $q$ is given by $\pi_q(a| s)=q(s, a)/q(s)$, where $q(s)=\suma q(s, a)$.

%Using the equivalence between proper policies and occupancy measures, the SSP problem can be transformed into an online linear optimization problem.
%Notice that the cost-to-go function of $\pi$ is linear w.r.t $q_{\pi}$:
%\begin{align*}
%	J^{\pi}_k(s_0) = \sums\suma q_{\pi}(s, a)c_k(s, a) = \inner{q_{\pi}}{c_k}.
%\end{align*}
%Therefore, the expected regret can also be written as follows:
%\begin{align*}
%	\E[R_K] = \E\left[\sumk J^{\pi_k}_k(s_0) - J^{\optpi}_k(s_0) \right] = \E\left[ \sumk\inner{q_{\pi_k} - q_{\optpi}}{c_k} \right].
%\end{align*}

\paragraph{Other notations}
%Define $\Ind_k(s, a) = \Ind\{\exists i=1,\ldots, I_k, \;(s_k^i, a_k^i)=(s, a)\}$, $\Ind_k(s, a, i) = \Ind\{s_k^i=s, a_k^i=a\}$, and $N_k(s, a)=\sum_{i=1}^{I_k}\Ind_k(s, a, i)$ (the number of visits to pair $(s,a)$ in episode $k$).
%Similarly, for a policy $\pi$, 
%define $N_{\pi}(s, a)=\sum_{i=1}^{I}1\{s^i=s, a^i=a\}$ as the (random) number of visits to pair $(s,a)$ when executing policy $\pi$.

We let $N_k(s, a)$ denote the (random) number of visits of the learner to $(s,a)$ during episode $k$, so that the regret can be re-written as $R_K = \sumk \inner{N_k - q_{\optpi}}{c_k}$.
Denote by $\Ind_k(s, a)$ the indicator of whether $c_k(s, a)$ is revealed to the learner in episode $k$, so that in the full information setting $\Ind_k(s, a)=1$ always holds, 
and in the bandit feedback setting $\Ind_k(s, a)$ is also the indicator of whether $(s, a)$ is ever visited by the learner.
Throughout the paper, we use the notation $\inner{f}{g}$ as a shorthand for $\sums f(s)g(s)$, $\sumsa f(s,a)g(s,a)$, $\sumh\sumsa f(s,a,h)g(s,a,h)$, or $\sumsa\sum_{s'}\sumh f(s, a, s', h)g(s, a, s', h)$ when $f$ and $g$ are functions in $\fR^{\calS}$, $\fR^{\SA}$, $\fR^{\SA\times[H]}$ or $\fR^{\SA\times(\calS\cup\{g\})\times[H]}$ (for some $H$) respectively.
Denote $\odot$ as the Hadamard product of tensors, so that $(u\odot v)_i=u_i \cdot v_i$ (e.g. the feedback on cost for both settings is thus $c_k \odot \Ind_k$).
Let $\calF_k$ denote the $\sigma$-algebra of events up to the beginning of episode $k$,
and $\E _k$ be a shorthand of $\E[\cdot|\calF_k]$.
To be specific, $c_k$ and the learner's policy in episode $k$ is already determined at the beginning of episode $k$, and the randomness in $\E[\cdot|\calF_k]$ is w.r.t the learner's actual trajectory in episode $k$.
%For a martingale difference sequence $X_{1:n}$ with respect to a filtration $\calF_{1:n}$ such that $\E[X_i| \calF_i]=0$, we denote by $\E_i$ a shorthand of $\E[\cdot|\calF_i]$.
%When there is no confusion,  %and $\E _k$ is a shorthand of $\E[\cdot|\calF_k]$.
%The Kullback-Leibler divergence is denoted by $\KL(p, q)=\sum_xp(x)\log\frac{p(x)}{q(x)}$, where the summation is over the support of $p$.
For a convex function $\psi$, the Bregman divergence between $u$ and $v$ is defined as: $D_{\psi}(u, v)=\psi(u)-\psi(v)-\inner{\nabla\psi(v)}{u-v}$.
For an integer $n$, $[n]$ denotes the set $\{1, \ldots, n\}$.

\section{Loop-free Reduction with Unknown Transition}\label{sec:reduction}
% !TEX root = main.tex

When the transition is known, \cite{chen2020minimax} show that it is possible to approximate a general SSP by a loop-free SSP in a way such that any policy in the loop-free instance can be transformed to a policy in the original instance with only $\tilo{1}$ additional overhead in the final regret.
More importantly, this loop-free reduction provides simpler forms for some variance-related quantities, which is the key in achieving high probability bounds and dealing with bandit feedback.
As the first step, we extend this loop-free reduction to the unknown transition setting, and show that the additional regret is also very small.

\paragraph{Loop-free instance}
The construction of the converted loop-free SSP instance is essentially the same as that in \cite{chen2020minimax}: for the first $H_1$ steps, we duplicate each state by attaching it with a time step $h$, then we connect all states to some virtual {\it fast} state that lasts for another $H_2$ steps.
We show the definition below for completeness (with slight modifications for our purposes),
and then discuss what the necessary changes are to complete the reduction using this loop-free SSP when the transition is unknown.
\begin{definition}{\citep[Definition 5]{chen2020minimax}}\label{def:reduction}
	For an SSP instance $M=(\calS, s_0, g, \calA, P)$ with cost functions $c_{1:K}$,
	we define, for horizon parameters $H_1, H_2 \in \fN$, another loop-free SSP instance $\tilM=(\tilS, \tils_0, g, \tilA, \tilP)$ with cost function $\tilc_{1:K}$ as follows: 
	\begin{itemize}
	\item $\tilS=\lfs\times [H]$ where $\lfs=\calS\cup\{s_f\}$, $s_f$ is an artificially added ``fast'' state, and $H = H_1 + H_2$.
	\item $\tils_0 = (s_0, 1)$, and the goal state $g$ remains the same.
	\item $\tilA = \calA\cup\{a_f\}$, where $a_f$ is an artificially added action. The available action set at $(s,h)$ is $\calA_s$ for all $s\neq s_f$ and $h\in [H]$, and the only available action at $(s_f, h)$ for $h\in[H]$ is $a_f$.
	\item Transition from $(s,h)$ to $(s', h')$ is only possible when $h'=h+1$:
	for the first $H_1$ layers, the transition follows the original MDP in the sense that $\tilP((s', h+1)|(s,h),a) = P(s'|s,a)$ and $\tilP(g|(s,h),a) = P(g|s,a)$ for all $h < H_1$ and $(s,a)\in\SA$;
	from layer $H_1$ to layer $H$, all states transit to the fast state: $\tilP((s_f, h+1)|(s,h),a) = 1$ for all $H_1 \leq h < H$ and $(s,a)\in \tilSA \triangleq \SA \cup \{(s_f, a_f)\}$;
	finally, the last layer transits to the goal state always: $\tilP(g|(s,H),a)=1$ for all $(s,a) \in \tilSA$. For notational convenience, we also write $\tilP((s', h+1)|(s, h), a)$ as $P(s'|s, a, h)$, and $\tilP(g|(s, h), a)$ as $P(g| s, a, h)$.
	\item 
	Cost function is such that $\tilc_k((s,h),a) = c_k(s,a)$ and $\tilc_k((s_f,h),a_f) = 1$ for all $(s,a)\in\SA$ and $h\in[H]$. For notational convenience, we also write $\tilc_k((s,h),a)$ as $c_k(s,a,h)$.
	\end{itemize}
\end{definition}

For notations related to the loop-free version, we often use a tilde symbol to distinguish them from the original counterparts (such as $\tilM$ and $\tilS$),
and for a function $\tilf((s,h),a)$ or $\tilf((s, h), a, (s', h+1))$ that takes a state-action pair or a state-action-afterstate triplet in $\tilM$ as input,
we often simplify it as $f(s, a, h)$ (such as $c_k$) or $f(s, a, s', h)$ (such as $q$ and $P$).
For such a function, we will also use the notation $\h{f} \in \fR^{\tilSA\times[H]}$ (or $\h{f} \in \fR^{\tilSA\times\lfs\times[H]}$) such that $(\h{f})(s, a, h)=h\cdot f(s, a, h)$ (or $(\h{f})(s, a, s', h)=h\cdot f(s, a, s', h)$).
Similarly, for a function $f \in \fR^{\tilSA}$, we use the same notation $\h{f} \in \fR^{\tilSA\times[H]}$ such that $(\h{f})(s, a, h)=h\cdot f(s, a)$.
Finally, for a occupancy measure $q \in [0,1]^{\tilSA\times \lfs \times [H]}$ of $\tilM$,
we write $q(s,a,h) = \sum_{s' \in \lfs} q(s,a,s',h)$ and $q(s,a) = \sum_{h=1}^H q(s,a,h)$.

\paragraph{The reduction}
Now, we are ready to describe the reduction, that is, how one can convert an algorithm for $\tilM$ to an algorithm for $M$.
Specifically, given policies $\tilpi_1, \ldots, \tilpi_K$ for $\tilM$,
we define a sequence of {\it non-stationary} policies $\sigma(\tilpi_1), \ldots, \sigma(\tilpi_K)$ for $M$ as follows. 
For each episode $k$, during the first $h \leq H_1$ steps, we follow $\tilpi(\cdot|(s,h))$ when at state $s$.
After the first $H_1$ steps (if not reaching $g$ yet), 
\citet{chen2020minimax} simply execute the fast policy $\pi^f$, available since the transition is known, to reach the goal state as soon as possible.
In our case with unknown transition, we propose to approximate the fast policy's behavior by running the Bernstein-base algorithm of~\citep{cohen2020near} designed for the fixed cost setting and pretending that all costs are $1$.
More precisely, we initialize a copy of their algorithm (that we call Bernstein-SSP) for $M$ (not $\tilM$) ahead of time, and whenever the learner does not reach the goal within $H_1$ steps in some episode, we invoke Bernstein-SSP as if this is a new episode for this algorithm, follow its decisions until reaching $g$, and always feed it a cost of $1$ for all state-action pairs.\footnote{%
This means that Bernstein-SSP is dealing with different initial states for different episodes, which is not exactly the same setting as the original work of~\citep{cohen2020near} but makes no real difference in their regret guarantee as pointed out in~\citep[Appendix C]{tarbouriech2020provably}.
}
We describe this converted policy in the procedure RUN (\pref{alg:unknown-run}).

\begin{algorithm}[t]
\caption{RUN$(\tilpi, \calB)$}
\label{alg:unknown-run}
	\textbf{Input:} a policy $\tilpi$ for $\tilM$ and a Bernstein-SSP instance $\calB$.
	
	\textbf{Initialize:} $s^1 = s_0$ and $h=1$.
	
	\While{$s^h \neq g$ and $h \leq H_1$}{
		Draw action $a^h \sim \tilpi(\cdot |(s^h,h))$.
		If $a^h = a_f$, break.\footnotemark 
		
		Play $a^h$, observe $s^{h+1}$, increment $h \leftarrow h+1$.
	}
	
    \If{$s^h\neq g$}{	
	Invoke $\calB$ with a new episode starting with state $s^h$, follow its decision until reaching $g$, and always feed it cost $1$ for all state-action pairs. 
	}
	\textbf{Return:} trajectory $\{s^1, a^1, s^2, a^2, \ldots, a^{h-1}, s^h\}$.
\end{algorithm}
\footnotetext{This if statement is only necessary for \pref{sec:iid}. \label{fn:break}}

The rationale of using Bernstein-SSP in this way is simply because when the costs are all $1$, the fast policy is exactly the optimal policy, and since Bernstein-SSP guarantees low regret against the optimal policy in the fixed cost setting, it behaves similarly to the fast policy in the long run in our reduction.

%execute Bernstein-SSP~\citep{cohen2020near} with artificial cost $c'(s, a)=1$ in $M$ until reaching the goal state $g$.
%The role of Bernstein-SSP is essentially approximating the fast policy in $M$, since executing Bernstein-SSP with $c'$ ensures sub-linear regret compared to fast policy, which is the optimal policy under cost $c'$.

This allows us to mostly preserve the properties of the reduction of~\citep{chen2020minimax}.
To state these properties, we need the following notations.
When executing $\sigma(\tilpi_k)$ in $M$ for episode $k$, we adopt the notation $\tilN_{k}$ and let $\tilN_{k}(s, a, h)$ be $1$ if $(s,a)$ is visited at time step $h \leq H_1$, or $0$ otherwise; and $\tilN_{k}(s_f, a_f, h)$ be $1$ if $H_1 < h \leq H$ and the goal state $g$ is not reached within $H_1$ steps, or 0 otherwise.
Clearly, $\tilN_{k}$ for $\tilM$ is the analogue of $N_k$ for $M$, and $\tilN_{k}(s, a, h)$ follows the same distribution as the number of visits to state-action pair $((s, h),a)$ when executing $\tilpi$ in $\tilM$. 
In addition, define a deterministic policy $\tiloptpi$ for $\tilM$ that mimics the behavior of $\optpi$ in the sense that $\tiloptpi(s,h) = \optpi(s)$ for $s\in\calS$ and $h\leq H_1$ (for larger $h$, $s$ has to be $s_f$ and the only available action is $a_f$).
With these notations, the next lemma shows that the reduction introduces little regret overhead when the horizon parameters $H_1$ and $H_2$ are set appropriately.

\begin{lemma}\label{lem:loop-free}
Suppose $H_1 \geq 8\Tmax\ln K$, $H_2 = \ceil{2D}$, $K\geq D$,
and $\tilpi_1, \ldots, \tilpi_K$ are policies for $\tilM$.
%Then for large enough $K$,
Then with probability at least $1-\delta$, the regret of executing $\sigma(\tilpi_1), \ldots, \sigma(\tilpi_K)$ in $M$ satisfies:
\[
	R_K \leq \sumk\inners{\tilN_{k}-\tiloptq}{c_k} + \tilO{D^{3/2}S^2A\rbr{\ln\tfrac{1}{\delta}}^2}.
	%R_K \leq \sumk\inners{\tilN_{k}-\tiloptq}{c_k} + \tilO{1}.
\]
\end{lemma}

\paragraph{Reduction alone is not enough}
While all of our algorithms make use of this reduction, it is worth emphasizing that the reduction alone is not enough.
Put differently, applying existing loop-free algorithms to $\tilM$ directly only leads to sub-optimal bounds with dependence on $H=\tilo{\Tmax}$.
This is true already in the known transition case~\citep{chen2020minimax}, and is even more so in our unknown transition case where one needs to estimate the transition.
On the other hand, what the reduction accomplishes is to make sure that some 
important variance-related quantities take a simple form that is linear in both the occupancy measure and the cost function.
For example, we will make use of the following important lemma, which is essentially taken from~\citep{chen2020minimax} but includes an extra intermediate result (the first inequality) important for \pref{sec:iid}.
In \pref{sec:full-unknown}, we will see another important property of the reduction.

\begin{lemma}\label{lem:deviation_loop_free}
Consider executing a policy $\sigma(\tilpi)$ in episode $k$.
%and let $\tilN_{k}(s, a, h) \in \{0,1\}$ denote the number of visits to state-action pair $((s, h),a)$.
Then 
	$\E_k[\inners{\tilN_{k}}{c_k}^2] \leq 2\inners{q_{\tilpi}}{c_k\odot Q^{\tilpi}_k} \leq 2\inners{q_{\tilpi}}{J^{\tilpi}_k} = 2\inners{q_{\tilpi}}{\h{c_k}}$.
\end{lemma}

\section{Adversarial Costs with Full Information}
\label{sec:full-unknown}
% !TEX root = main.tex

%To introduce our algorithm for the full-information setting,
%we first briefly review .

In the full-information setting, the algorithm of~\citep{chen2020minimax} maintains a sequence of occupancy measures $q_1, \ldots, q_K$ for $\tilM$, obtained via Online Mirror Descent (OMD) over a sophisticated {\it skewed occupancy measure} space.
In their analysis, the regret for $\tilM$ from \pref{lem:loop-free} is decomposed as $\sumk\inners{\tilN_k - \tiloptq}{c_k} = \sumk\inners{\tilN_k-q_k}{c_k} + \sumk\inners{q_k - \tiloptq}{c_k}$,
where the first term is the sum of a martingale difference sequence whose variance can be bounded using \pref{lem:deviation_loop_free},
and the second term is controlled by the standard OMD analysis.
Importantly, due to the skewed occupancy measure, the bound for the second term contains a negative bias in terms of $-\inners{q_{k}}{\h{c_k}}$, which can then cancel the variance from the first term in light of \pref{lem:deviation_loop_free}.

When the transition is unknown, we follow the ideas of the SSP-O-REPS algorithm~\citep{rosenberg2020adversarial} and maintain a confidence set of plausible transition functions, which contains the true transition $P$ with high probability.
This step is conducted via the procedure $\TransEst$ (\pref{alg:update_transition}), which takes a trajectory returned by $\RUN$ (along with other statistics) and outputs an updated confidence set  based on standard concentration inequalities.
We defer the details to \pref{sec:trans-est}.

With a confidence set $\calP$ at hand, we define the set of plausible occupancy measures $\tilDelta(T, \calP)$ as follows, which is parameterized by $\calP$ and a size parameter $T$
(recall the shorthand $q(s,a,h) = \sum_{s'} q(s,a,s',h)$):
\begin{align}
	&\Bigg\{ q \in [0,1]^{\tilSA\times \lfs \times [H]}: \sumh\sumtilsaf q(s, a, h)\leq T; \notag\\
	&\sum_{a\in\tilA_{(s,h)}} q(s, a, h) = \sumtilsaf[s', a'] q(s', a', s, h-1),\; \forall h > 1;\notag\\ 
	&\sum_{a\in \tilA_{(s,1)}} q(s, a, 1) = \Ind\{s=s_0\}, \;\forall s\in\calX;  \; P_q \in \calP  \Bigg\}. \label{eq:constraint-set}
\end{align}
When $\calP = \{P\}$, this is equivalent to the set used by~\citep{chen2020minimax}, where the first inequality constraint makes sure that the induced policy reaches the goal within $T$ steps in expectation, the equality constraints make sure that $q$ is a valid occupancy measure, and the last constraint $P_q = P$ makes sure that the induced transition is consistent with the true one.
We naturally generalize the set to the unknown transition case by enforcing the induced transition $P_q$ to be within a given confidence set.

Then, in each episode $k$, with $\calP_k$ being the current confidence set, 
we define the skew occupancy measure space for some parameter $\lambda$ as 
\begin{align}
	\Omega_k = \cbr{\qfeat=q + \lambda \h{q} : q\in \tilDelta(T, \calP_{k})}. \label{eq:skewed_occupancy_measure}
\end{align}
which is again a direct generalization of~\citep{chen2020minimax} from $\{P\}$ to $\calP_k$.
Our algorithm then maintains a sequence of skewed occupancy measures $\qfeat_1, \ldots, \qfeat_K$ based on the standard OMD framework:
\[
\qfeat_{k+1} = \argmin_{\qfeat\in\Omega_{k+1}}\inner{\qfeat}{c_k} + D_{\regz}(\qfeat, \qfeat_k)
\]
where $\regz$ is the negative entropy regularizer.
In each episode, extracting $\hatq_k$ from $\qfeat_{k} = \hatq_k + \lambda \h{\hatq_k}$, we obtain a policy $\tilpi_{\hatq_k}$ for $\tilM$, and then execute it via the $\RUN$ procedure (\pref{alg:unknown-run}).
The complete pseudocode of our algorithm is presented in \pref{alg:full-unknown},
which can be efficiently implemented (see related discussion in~\citep{rosenberg2020adversarial}).

\begin{algorithm}[t]
\caption{SSP-O-REPS with Loop-free Reduction and Skewed Occupancy Measure}
\label{alg:full-unknown}
	\textbf{Input:} Upper bound on expected hitting time $T$, horizon parameter $H_1$, confidence level $\delta$
	
	\textbf{Parameters:} $\eta=\min\cbr{\frac{1}{8}, \sqrt{\frac{T}{DK}}}, \lambda=4\sqrt{\frac{S^2A}{DTK}}, H_2=\ceil{2D}, H=H_1+H_2$
	
	\textbf{Define:} regularizer $$\regz(\qfeat) = \frac{1}{\eta}\sum_{h=1}^{H}\sum_{(s,a)\in\tilSA}\sum_{s'\in\calX\cup\{g\}} \qfeat(s, a, s', h)\ln \qfeat(s, a, s', h)$$
	
	%\textbf{Define:} $\hatq_k + \lambda \h{\hatq_k} = \qfeat_k$
	%, decision set $\Omega_k = \{\qfeat=q + \lambda \h{q} : q\in \tilDelta(T, \calP_{i_k})\}$ (with $\tilDelta(T, \calP)$ defined in \pref{eq:constraint-set})
	
	\textbf{Initialize:} $\N_1(s, a)=\M_1(s, a, s')=0$ for all $(s, a, s')\in\SA\times(\calS\cup\{g\})$, a Bernstein-SSP instance $\calB$, $\calP_1$ is the
set of all possible transition functions, $\qfeat_1 = \argmin_{\qfeat\in\Omega_1} \regz(\qfeat)$ (where $\Omega_k$ is defined in \pref{eq:skewed_occupancy_measure}).
	
	\For{$k=1,\ldots,K$}{
		Extract $\hatq_k$ from $\qfeat_k = \hatq_k + \lambda \h{\hatq_k}$ and let $\tilpi_k = \tilpi_{\hatq_k}$.
	
		Execute policy $\tilpi_k$: $\tau_k=\RUN(\tilpi_k, \calB)$, receive $c_k$.
	
		Update $\calP_{k+1}=\TransEst(\N, \M, \delta, H_1, H_2, \tau_k)$.
		
		Update $\qfeat_{k+1} = \argmin_{\qfeat\in\Omega_{k+1}}\inner{\qfeat}{c_k} + D_{\regz}(\qfeat, \qfeat_k)$.
	}
\end{algorithm}

\paragraph{Analysis}
Let $q_k$ be the occupancy measure with respect to the policy $\tilpi_{k}$ and the true transition $P$.
We can then decompose the regret from \pref{lem:loop-free} as
$\sumk\inners{\tilN_k - \tiloptq}{c_k} = \sumk\inners{\tilN_k-q_k}{c_k}  + \sumk\inners{\hatq_k - \tiloptq}{c_k}+ \sumk\inners{q_k-\hatq_k}{c_k}$,
where the last term measures the difference between $q_k$ and $\hatq_k$ due to the transition estimation error and is the only extra term compared to the known transition case discussed at the beginning of this section.
One of our key technical contributions is to prove that, thanks to the structure of the loop-free instance $\tilM$, this term is in fact also bounded by the variance term seen earlier in \pref{lem:deviation_loop_free}:
\begin{equation}\label{eq:q_k_qhat_k}
	\sumk\inner{q_k-\hatq_k}{c_k}=\tilO{\sqrt{S^2A\sumk\E_k[\inners{\tilN_{k}}{c_k}^2]}}.
\end{equation}
See \pref{lem:transition-bias} for the complete statement, whose proof makes use of a Bellman type law of total variance for Bernstein-based confidence sets (\pref{lem:var}).

With this result and \pref{lem:deviation_loop_free}, one can see that just like the first term $\sumk\inners{\tilN_k-q_k}{c_k}$, the extra transition error term can also be handled by the negative bias introduced by the skewed occupancy measure space as discussed earlier.
This leads to our final regret guarantee of \pref{alg:full-unknown}.

\begin{theorem}
	\label{thm:full-unknown}
	If $T\geq\T+1, H_1\geq 8\Tmax\ln K$, and $K\geq 16S^2AH^2$, then with probability at least $1-6\delta$, \pref{alg:full-unknown} ensures $R_K=\tilo{\sqrt{S^2ADTK} + H^3S^2A}$.
\end{theorem}

We emphasize that our way to handle the transition estimation error $\sumk\inner{q_k-\hatq_k}{c_k}$ is novel.
Specifically, all previous works directly upper bound this error using the definition of confidence interval, which in our case introduces an undesirable $\Tmax$ dependency.
Instead, we derive a specific upper bound (\pref{eq:q_k_qhat_k}) of the transition estimation error that can be cancelled out by the negative term introduced by the skewed occupancy measure.
This technique is especially useful in obtaining data-dependent bound in the unknown transition case, since it replaces the error by a term related to the optimal policy, which is hard to achieve if we directly upper bound the error.

Besides this new way to handle the transition estimation error,
another source of improvement compared to the analysis of~\citep{rosenberg2020adversarial} is to make use of the fact $\sumk\inner{q_{\optpi}}{c_k}\leq DK$ in the OMD analysis.
Again, we emphasize that even without the knowledge of $\T$ or $\Tmax$,
our analysis leads to better bounds compared to theirs; see \pref{app:tune T}.

%We also emphasize that the improvement over \cite{rosenberg2020adversarial} of \pref{alg:full-unknown} is not simply from the knowledge of $\T$ or $\Tmax$.
%In fact, it achieves a better bound asymptomatically compared to \cite{rosenberg2020adversarial} even if it is instantiated without knowing $\T$ and $\Tmax$; see \pref{sec:tune T}.
%This improvement comes from two places: first in the OMD analysis, we derive a better bound by utilizing the fact $\sumk\inner{q_{\optpi}}{c_k}\leq DK$; second, our bound on transition estimation error (\pref{eq:q_k_qhat_k}) is tighter and data-dependent.

Since \citet{chen2020minimax} show a lower bound of $\lowo{\sqrt{D\T K}}$ for stochastic costs and known transition, and \citet{cohen2020near} show a lower bound of $\lowo{D\sqrt{SAK}}$ for fixed costs and unknown transition, we know that in our setting, $\lowo{ \sqrt{D\T K} + D\sqrt{SAK}}$ is a lower bound, which shows a gap of $\sqrt{S\T / D}$ from our upper bound.
Closing the $\sqrt{S}$ gap is still open even for the loop-free case~\citep{rosenberg2019online,jin2019learning}.
On the other hand, closing the $\sqrt{\T/D}$ gap also seems rather challenging for adversarial costs, but is indeed possible for stochastic costs as we show in \pref{sec:iid} (note that the lower bound is indeed constructed with stochastic costs).

\section{Adversarial Costs with Bandit Feedback}
\label{sec:bandit-unknown}
% !TEX root = main.tex

We now consider the bandit feedback setting which, even when the transition is known, is quite challenging already and requires several new techniques as shown by~\citet{chen2020minimax}.
Our algorithm is built on top of their Log-barrier Policy Search algorithm with the transition estimation component integrated in a similar way as in \pref{sec:full-unknown}.
We defer most details to \pref{app:bandit-unknown} but only highlight two important new ingredients below.

A standard technique to deal with adversarial costs and bandit feedback in online learning  is to feed the OMD algorithm with importance-weighted cost estimators (since $c_k$ is now only partially observed).
Specifically, the Log-barrier Policy Search algorithm of~\citet{chen2020minimax} feeds OMD with cost $\hatc_k-\gamma \hatb_k$ (for some parameter $\gamma$),
where $\hatc_k(s, a) = \frac{\tilN_k(s, a)}{q_k(s, a)}c_k(s, a)$ and $\hatb_k(s, a)=\frac{\sum_hh\cdot q_k(s, a, h)\hatc_k(s, a)}{q_k(s, a)}$ are two importance-weighted estimators.
Here, $q_k(s, a)$ is defined as $\sumh q_k(s,a,h)$ and $\tilN_k$ is defined above \pref{lem:loop-free} with mean $q_k(s, a)$, so that $\hatc_k$ is an unbiased estimator of $c_k$.
The reason of having $\hatb_k$, on the other hand, is relatively technical, but it eventually serves as a way of reducing variance by introducing a negative bias.
The immediate challenge to generalize these estimators to the unknown transition setting is that $q_k$, the occupancy measure with respect to the policy $\tilpi_k$ for episode $k$ and the true transition $P$, is now unknown.

To address this issue for $\hatc_k$, we follow the idea of~\citep{jin2019learning} and construct the following \emph{optimistic} biased estimator: $\hatc_k(s, a) = \frac{\tilN_k(s, a)}{u_k(s, a)}c_k(s, a)$ where
$
	u_k(s, a) = \max_{\hatP\in\calP_{k}}q_{\hatP, \tilpi_k}(s, a)
$,
called the \emph{upper occupancy bound},
is the largest possible expected number of visits to $(s,a)$ of policy $\tilpi_k$ under a plausible transition from the confidence set $\calP_k$.
Clearly, $q_k(s,a) \leq u_k(s,a)$ holds (with high probability), making $\hatc_k(s, a)$ an optimistic underestimator which is important in reducing variance as shown in~\citet{jin2019learning}.
Note that $u_k$ can be efficiently computed since it boils down to solving a linear program.\footnote{%
To see this, note that $u_k(s, a)$ is equivalent to $\max_q q(s,a)$ where the maximization is over the set $\{q \in \tilDelta(\infty, \calP_{k}): \pi_q =\tilpi_k \}$, which consists of polynomially many linear constraints. \label{fn:LP}
}

%To efficiently compute $u_k(s, a)$, note that $u_k(s, a)=q^\star(s, a)$, where
%\[
%	q^\star = \argmax_{q\in\{q'\in \tilDelta(T, \calP_{i_k}):\pi_{q'}=\tilpi_k\}}\sum_{h=1}^Hq(s, a, h).
%\] 
%Thus, computing $q^\star$ boils down to solving a linear program, which can be done in polynomial time.

On the other hand, $\hatb_k$ does not appear before in the loop-free setting of~\citet{jin2019learning} and requires some more careful thinking.
Other than replacing $q_k$ in the denominator with $u_k$, we also need to deal with $q_k(s, a, h)$ in the numerator.
It turns out that the right generalization is to let 
\[
	\hatb_k(s, a) = \frac{\max_{\hatP\in\calP_{k}}\sum_h h \cdot q_{\hatP, \tilpi_k}(s, a, h)\hatc_k(s, a)}{u_k(s, a)},
\]
so that $\sum_h h \cdot q_k(s,a,h) \hatc_k(s, a) \leq u_k(s,a) \hatb_k(s, a)$ holds (with high probability), which in turn makes sure that the bias introduced by $\hatb_k$ is large enough to cancel some important variance term, as shown in \pref{lem:simpler_form_N_times_c}.
Similarly, $\hatb_k$ can also be computed efficiently (\emph{c.f.} \pref{fn:LP}).

%$\hatb_k(s, a) = \frac{\sum_h h \cdot u'_k(s, a, h)\hatc_k(s, a)}{u_k(s, a)}$, where $u'_k(s, a, h) \triangleq q_{P_k^{(s,a)}, \tilpi_k }(s,a,h)$ and $P_k^{(s, a)} = \argmax_{\hatP\in\calP_{k}}\sum_{h=1}^Hh\cdot q_{\hatP, \tilpi_k}(s, a, h)$.

%For $\hatb_k$, we similarly replace $q_k(s, a)$ in the denominator by $u_k(s, a)$, and $q_k(s, a, h)$ in the numerator by $u'_k(s, a, h)$, such that $\sum_hh\cdot u'_k(s, a, h)$ is optimistic (see \pref{lem:simpler_form_N_times_c} for the reason).
%Formally, we define $\hatb_k$ as follows:
%\begin{align*}
%	\hatb_k(s, a) &= \frac{\sum_h h  u'_k(s, a, h)\hatc_k(s, a)}{u_k(s, a)},\\
%	u'_k(s, a, h) &= q_{P_k^{u', (s, a)}, \tilpi_k}(s, a, h),\\
%	P_k^{u', (s, a)} &= \argmax_{\hatP\in\calP_{i_k}}\sum_{h=1}^Hh\cdot q_{\hatP, \tilpi_k}(s, a, h).
%\end{align*}
%$u'_k(s, a, h)$ can also be cast as the solution of a linear program and thus efficiently computable.

Our final algorithm is summarized in \pref{alg:bandit-unknown} of \pref{app:bandit-unknown}.
Noting that the bias introduced by the upper occupancy bounds is eventually also related to the transition estimation error that has been analyzed in \pref{lem:transition-bias},
we are able to prove the following regret guarantee.

\begin{theorem}
	\label{thm:bandit-unknown}
	If $T\geq\T+1, H_1\geq 8\Tmax\ln K$, and $K$ is large enough ($K \gtrsim S^3A^2H^2$), then with probability at least $1-30\delta$, \pref{alg:bandit-unknown} ensures $R_K = \tilO{\sqrt{S^3A^2DTK}+H^3S^3A^2}$.
\end{theorem}

Compared to the full-information setting, here we pay an extra $\sqrt{SA}$ factor in the regret bound, a price that does not exist in the loop-free setting~\citep{rosenberg2019online,jin2019learning}.
This comes from a technical lemma on bounding $\sumk\inner{u_k-q_k}{c_k}$ in terms of $\sumk\inners{q_k}{\h{c_k}}$ so that it can be canceled by the skew occupancy measure; see \pref{lem:optimistic-transition-bound}.
Removing this extra factor is an important future direction.
On the other hand, by combing the lower bounds of~\citep{chen2020minimax} and~\citep{cohen2020near} again, we have that $\lowo{ \sqrt{SAD\T K} + D\sqrt{SAK}}$ is the best existing lower bound for this setting. 

%Note that there is an extra $\sqrt{SA}$ factor in the bias of transition estimation compared to the full information setting.
%This is due to the use of optimistic cost estimator: we are not able to aggregate bias terms in different $(s, a)$ since the transition function of each $u_k(s, a)$ is different.
%We suspect that it is the limitation of the optimistic cost estimator, and we left it as a future work to remove this extra factor.

\section{Stochastically Oblivious Adversary}
\label{sec:iid}
% !TEX root = main.tex

Given the gap between our upper and lower bounds, in this section, we consider a weaker stochastically oblivious adversary and develop a simple algorithm with regret bounds only $\sqrt{S}$ times larger than the aforementioned lower bounds.
Specifically, in this setting the adversary generates ahead of the time the cost functions  $c_1, \ldots, c_K$ as i.i.d. samples from a fixed and unknown distribution with mean $c: \SA \to [0,1]$.
The regret measure is also changed to the more standard pseudo-regret 
$\tilR_K = \sumk\inner{N_k}{c_k} - \inner{\optq}{c}$ where $\optpi \in \argmin_\pi J^{\pi, c}(s_0)$.\footnote{We can get a bound for the standard regret with an extra cost of order $\tilo{\sqrt{D\T K}}$. Therefore, the standard regret and the pseudo regret are of the same order. We use the latter only for simplicity and convention.}
We remind the readers that the lower bound is indeed for the pseudo-regret and is constructed via this weaker adversary, and also that this is slightly different from the setting studied in~\citep{tarbouriech2019no, cohen2020near} as mentioned in \pref{sec:intro}.

%\paragraph{Transform from pseudo regret to standard regret}
%Using standard regret, we need $T \geq \T$, where $\T$ is still the expected hitting time of optimal policy, but the optimal policy here maximizes $c_{1:K}$. Note that in the proof, we get dependency $D$ by $\inner{\hatq_k}{\hatc_k} \leq D$ and it has nothing to do with whether the optimal policy optimize $c_{1:K}$ or $c$. Therefore, the proof still proceeds. The only thing we need to deal with is to bound $\sum_k \inner{\optq}{c-c_k}$.
%Note that:
%\begin{align*}
%	&\sumk\inner{\optq}{c-c_k}\\
%	&= \sumsa \optq(s, a)\sumk (c(s, a) - c_k(s, a))\\
%	&=\tilO{ \sumsa \optq(s, a)\sqrt{Kc(s, a)} } \tag{Bernstein}\\
%	&=\tilO{ \sqrt{\sumsa\optq(s, a)}\sqrt{K\inner{\optq}{c}} }\\
%	&=\tilO{\sqrt{\T\sumk\inner{\optq}{c}}}.
%\end{align*}
%Moreover,
%\begin{align*}
%	&\sumk\inner{\optq}{c} = \sumk\inner{\optq}{c-c_k} + \sumk\inner{\optq}{c_k}\\
%	&\leq \sqrt{\T\sumk\inner{\optq}{c}} + DK,
%\end{align*}
%which gives $\sumk\inner{\optq}{c} = \tilO{DK+\T}$.
%Plug this back, we get $\sumk\inner{\optq}{c-c_k}=\tilO{\sqrt{D\T K}+\T}$.

Our algorithm is based on the well-known \emph{optimism in face of uncertainty} principle, which finds the best policy among all plausible MDPs subject to some additional constraints.
First, we compute an optimistic cost function $\hatc_k$ defined via $\hatc_k(s,a)$ being\footnote{This is not to be confused with the estimator used in \pref{sec:bandit-unknown} with the same notation overloaded.}
\begin{equation}\label{eq:optimistic_cost}
\max\left\{\bar{c}_k(s, a)-2\sqrt{\A^c_k(s, a)\bar{c}_k(s, a)}-7\A^c_k(s, a), 0\right\},
\end{equation}
where $\bar{c}_k(s, a)=\frac{\sum_{j=1}^{k-1}c_j(s, a)\Ind_j(s, a)}{\Nc(s, a)}$ is the empirical cost mean, $\Nc(s, a)=\max\left\{\sum_{j=1}^{k-1}\Ind_j(s, a), 1\right\}$ is the number of times the cost at $(s, a)$ was revealed (covering both the full-information and the bandit settings), and $\A^c_k(s, a)=\frac{\ln(2SAK/\delta)}{\Nc(s, a)}$.
Then, we find the best occupancy measure with respect to this optimistic cost, with the same constraint $\tilDelta(T, \calP_k)$ as in previous sections:
\begin{align}
	\hatq_k = \argmin_{q\in\tilDelta(T, \calP_k)}\inner{q}{\hatc_k}, \label{eq:evi}
\end{align}
and finally execute the induced policy $\tilpi_k = \tilpi_{\hatq_k}$ as before.

There is, however, one caveat in the approach above.
Our analysis relies on one crucial property of $\tilpi_k$: $J^{P_k, \tilpi_k, \hatc_k}(s, h) \leq D$, that is, its state value with respect to the optimistic transition/cost is always no more than the diameter $D$.
This holds automatically if we did not impose the hitting time constraint in \pref{eq:evi}, due to the existence of the fast policy $\pi^f$ whose state value is never worse than $D$.
With the hitting time constraint, however, this might not hold anymore.
To address this, we slightly modify the loop-free instance $\tilM$ and give every state $(s,h)$ (for $h \leq H_1$) a \emph{shortcut} to directly transit to $(s_f, H_1+1)$ by taking action $a_f$, which is equivalent to allowing the learner to switch to Bernstein-SSP (whose role is similar to the fast policy) at any state and any time (\emph{c.f.} \pref{fn:break}).
This ensures $J^{P_k,\tilpi_k, \hatc_k}(s,h) \lesssim D$ as desired; see \pref{lem:fast-action}.
This modification can be implemented by a small change to the definition of $\tilDelta$, and we defer the details to \pref{app:iid}.
With this in mind, our final algorithm is presented in \pref{alg:iid}.

\begin{algorithm}[t]
\caption{A near-optimal algorithm for stochastically oblivious adversary}
\label{alg:iid}
	\textbf{Input:} Upper bound on expected hitting time $T$, horizon parameter $H_1$ and confidence level $\delta$

	\textbf{Parameters:} $H_2=\ceil{2D}, H=H_1+H_2$.
	
	%\textbf{Define:} $\hatc_k(s, a)=\max\left\{\bar{c}_k(s, a)-\epsilon^c_k(s, a), 0\right\}$, where $\epsilon^c_k(s, a)=2\sqrt{\A^c_k(s, a)\bar{c}_k(s, a)}+7\A^c_k(s, a)$, $\bar{c}_k(s, a)=\frac{\sum_{j=1}^{k-1}c_j(s, a)\Ind_j(s, a)}{\Nc(s, a)}$, $\A^c_k(s, a)=\frac{\log(2SAK/\delta)}{\Nc(s, a)}$, $\Nc(s, a)=\max\left\{\sum_{j=1}^{k-1}\Ind_j(s, a), 1\right\}$.
	
	%\textbf{Initiation:} Set epoch index $i=1$. Set counters $\N_0(s, a)=\N_1(s, a)=\M_0(s, a, s')=\M_1(s, a, s')=0$.
	
	\textbf{Initialization:} $\N_1(s, a)=\M_1(s, a, s')=0$ for all $(s, a, s')\in\SA\times(\calS\cup\{g\})$, a Bernstein-SSP instance $\calB$, $\calP_1$ is the set of all possible transition functions.
	
	\For{$k=1,\ldots,K$}{
		Compute the optimistic cost $\hatc_k$ (\pref{eq:optimistic_cost}).
	
		Compute $\hatq_k=\argmin_{q\in\tilDelta(T, \calP_k)}\inner{q}{\hatc_k}$.
	
		Execute $\tilpi_k=\tilpi_{\hatq_k}$: $\tau_k = \RUN(\tilpi_k, \calB)$, receive $c_k\odot\Ind_k$.
	
		Update $\calP_{k+1}=\TransEst(\N, \M, \delta, H_1, H_2, \tau_k)$.
	}
\end{algorithm}

\paragraph{Analysis}
The key reason that we can improve our regret bounds in this stochastic setting is as follows. First, since the estimated cost converges to the true cost fast enough, the previous dominating term $\sumk\inner{q_k-\hatq_k}{c_k}$ can now be replaced by $\sumk\inner{q_k-\hatq_k}{\hatc_k}$.
Then, similar to \pref{eq:q_k_qhat_k}, the latter is in the order of $\sqrt{S^2A\sumk\E_k[\inners{\tilN_{k}}{\hatc_k}^2]}$, which is further bounded by $\sqrt{S^2A\sumk\inner{q_k}{\hatc_k \odot Q^{\tilpi_k, \hatc_k}} }$ according to the first inequality of \pref{lem:deviation_loop_free}.
Finally, we make use of the aforementioned property $J^{P_k, \tilpi_k, \hatc_k}(s, h) \leq D$ to show that $\inner{q_k}{\hatc_k \odot Q^{\tilpi_k, \hatc_k}}$ is roughly $D^2$, leading to a final bound of $\otil(\sqrt{ S^2AD^2K })$ and improving over the $\tilO{\sqrt{ S^2AD\T K }}$ bound in \pref{thm:full-unknown}.
We summarize our results in the following theorem.

\begin{theorem}
	\label{thm:iid}
	If $T\geq\T+1, H_1\geq 8\Tmax\ln K$, and $K\geq H^2$, then \pref{alg:iid} ensures with probability at least $1-30\delta$, $\tilR_K=\tilo{\sqrt{DTK}+DS\sqrt{AK}+H^3S^3A^2}$ in the full information setting and $\tilR_K=\tilo{\sqrt{DTSAK}+DS\sqrt{AK}+H^3S^3A^2}$ in the bandit feedback setting.
\end{theorem}

Comparing with the lower bounds, one sees that our bounds are only $\sqrt{S}$ factor larger, a gap that also appears in other settings such as~\citep{cohen2020near}.
Unfortunately, we are not able to obtain the same improvement in the general adversarial setting, and we in fact conjecture that the lower bound there can be improved to at least $\lowO{\sqrt{ SAD\T K }}$, which, if true, would require a lower bound construction that is actually adversarial, instead of being stochastic as in most existing lower bound proofs.

\section*{Acknowledgements}
We thank Aviv Rosenberg, Chen-Yu Wei, and the anonymous reviewers for many helpful discussions and feedback.
This work is supported by NSF Award IIS-1943607 and a Google Faculty Research Award.

\bibliographystyle{plainnat}
%{\small\bibliography{bib}}
\bibliography{bib}

%%% appendix
\newpage
\onecolumn
\appendix
% !TEX root = main.tex

%\section{Extra Notations}
%Denote by $q_{\pi, P, s}$ the occupancy measure of policy $\pi$ with transition $P$ and initial state $s$. When there is no confusion, we ignore dependency on $P$ and write it as $q_{\pi, s}$.
%Define $\Ind_k(s, a, h)$ the indicator of whether $(s, h), a$ is visited in $\tilM$ in episode $k$.

\section{Loop-free reduction}
\label{app:loop-free}
%!TEX root = main.tex

In this section, we give full proofs of lemmas related to the proposed loop-free reduction.

\subsection{\pfref{lem:loop-free}}
\begin{proof}
Denote by $N'_k(s, a)$ the number of visits to $(s, a)$ during episode $k$ before switching to Bernstein-SSP,
by $N''_k(s, a)$ the number of visits to $(s, a)$ after switching to Bernstein-SSP,
and by $N_f$ the number of episodes where Bernstein-SSP is invoked. %=\sumk\Ind\{s^k_{H_1+1}\neq g\}$.
We have: $N_k(s, a)=N'_k(s, a)+N''_k(s, a)$ and $\sumk\inner{\tilN_k}{c_k}=\sumk\inner{N'_k}{c_k}+H_2N_f$.
Recall that the regret of running Bernstein-SSP~\citep{cohen2020near} for $K'$ episodes under uniform cost is of $\bigO{ DS\sqrt{AK'}\ln\frac{K'DSA}{\delta} + \sqrt{D^3S^4A^2}\ln^2\frac{K'DSA}{\delta} }$ with probability at least $1-\delta$.
Conditioned on the event above, we have:
\begin{align*}
	\sumk\inner{N''_k}{c_k} - H_2N_f &\leq \rbr{\sumk\inner{N''_k}{c_k} - DN_f} - DN_f \tag{$H_2\geq 2D$}\\
	&\leq \bigO{DS\sqrt{AN_f}\ln\frac{KDSA}{\delta} + \sqrt{D^3S^4A^2}\ln^2\frac{KDSA}{\delta}} - DN_f\\
	&= \bigO{D^{3/2}S^2A\ln^2\frac{KDSA}{\delta}},
\end{align*}
where in the last inequality we solve for the maximum of a quadratic function with variable $N_f$.
Therefore,
\begin{align*}
	\sumk\inner{N_k}{c_k} = \sumk\inner{N'_k}{c_k} + H_2N_f + \sumk\inner{N''_k}{c_k} - H_2N_f \leq \sumk\inner{\tilN_k}{c_k} + \tilO{D^{3/2}S^2A\ln^2\frac{1}{\delta}}.
\end{align*}
On the other hand, by \pref{lem:hitting}, the probability that the goal state is not reached within $H_1$ steps when executing $\optpi$ is at most $2e^{-\frac{H_1}{4\Tmax}} \leq \frac{2}{K^2}$.
Hence, the expected cost of $\optpi$ in $M$ and the expected cost of $\tiloptpi$ in $\tilM$ is very similar:
\begin{align*}
J^{\tiloptpi}_k(\tils_0)  \leq J^{\optpi}_k(s_0) + \frac{2H_2}{K^2} =  
J^{\optpi}_k(s_0) + \tilO{\frac{1}{K}}.
\end{align*}
	Putting everything together, and by $K\geq D$, we get:
	\begin{align*}
		R_K = \sum_{k=1}^K\inner{N_k}{c_k} - J^{\optpi}_k(s_0) \leq \sumk\inner{\tilN_k - \tiloptq}{c_k} +\tilO{D^{3/2}S^2A\ln^2\frac{1}{\delta}}.
	\end{align*}
%	It thus remains to bound the deviation $\sumk \inner{\tilN_k - q_k}{c_k}$, which is the sum of a martingale difference sequence.
%We apply Freedman's inequality \pref{lem:freedman} directly:
%the variable $\inner{\tilN_k}{c_k}$ is bounded by $H$ always, and its conditional variance is bounded by $2\inner{q_k}{\h{c_k}}$ as shown in \pref{lem:deviation_loop_free}, which means for any $\lambda \in (0, 2/H]$,
%\[
%\sumk \inner{\tilN_k - q_k}{c_k}\leq \lambda\sumk\inner{q_k}{\h{c_k}} + \frac{2\ln\left(\nicefrac{2}{\delta}\right)}{\lambda}
%\]
%holds with probability at least $1-\frac{\delta}{2}$.
%Applying another union bound finishes the proof.
\end{proof}

\subsection{\pfref{lem:deviation_loop_free}}
\begin{proof}
	With the inequality $(\sum_{i=1}^I a_i)^2\leq 2\sum_{i}a_i(\sum_{i'= i}^I a_{i'})$, we proceed as
	\begin{align*}
		&\E_k\left[\left( \sum_{(s, a)\in\tilSA, h} \tilN_{k}(s, a, h)c_k(s, a, h) \right)^2\right] \\
		&\leq 2\E_k\left[\sum_{h=1}^{H}\sumtilsaf \tilN_{k}(s, a, h)c_{k}(s, a, h)\left( \sum_{h'= h}^{H}\sumtilsaf[s', a'] \tilN_{k}(s', a', h')c_k(s', a', h') \right)\right] \\
		&= 2\E_k\left[\sum_{h=1}^{H}\sumtilsaf \tilN_{k}(s, a, h)c_{k}(s, a, h)\E\left[\left. \sum_{h'=h}^{H}\sumtilsaf[s', a'] \tilN_{k}(s', a', h') c_k(s', a', h') \right| \tils^h_k=(s, h), a^h_k=a\right]\right]\\
		&= 2\E_k\left[\sum_{h=1}^{H}\sumtilsaf \tilN_{k}(s, a, h)c_k(s, a, h)Q^{\tilpi}_k(s, a, h)\right] = 2\inner{q_{\tilpi}}{c_k\odot Q^{\tilpi}_k}\\
		&\leq 2\E_k\left[\sum_{h=1}^{H}\sumtilsaf \tilN_{k}(s, a, h)Q^{\tilpi}_k(s, a, h)\right] = 2 \sum_{h=1}^{H}\sumtilsaf q_{\tilpi}(s, a, h)Q^{\tilpi}_k(s, a, h)\\
		&= 2\sum_{h=1}^H\sum_{s\in\calS\cup\{s_f\}} q_{\tilpi}(s, h)\sum_a\tilpi(a|s)Q^{\tilpi}_k(s, a, h) = 2\sum_{h=1}^H\sum_{s\in\calS\cup\{s_f\}}q_{\tilpi}(s, h)J^{\tilpi}_k(s, h) = 2\inner{q_{\pi}}{J^{\pi}_k}.
	\end{align*}
	This proves the first two inequalities.
	Denote by $q_{\tilpi, (s, h)}$ the occupancy measure of policy $\tilpi$ with initial state $(s, h)$, so that 
	\[
	J^{\tilpi}_k(s, h) = \sum_{(s', a')\in\tilSA}\sum_{h'\geq h}q_{\tilpi,(s, h)}(s', a', h')c_k(s', a', h').
	\]
	Then, we continue with the following equalities:
	\begin{align*}
		&\sum_{h=1}^H\sum_{s\in\calS\cup\{s_f\}}q_{\tilpi}(s, h)J^{\tilpi}_k(s, h) \\
		&= \sum_{h=1}^{H}\sum_{s\in\calS\cup\{s_f\}} q_{\tilpi}(s, h)\sum_{(s', a')\in\tilSA}\sum_{h'\geq h}q_{\tilpi,(s, h)}(s', a', h')c_k(s', a', h')\\
		&= \sum_{h=1}^{H}\sum_{(s', a')\in\tilSA}\sum_{h'\geq h}\left(\sum_{s\in\calS\cup\{s_f\}} q_{\tilpi}(s, h)q_{\tilpi,(s, h)}(s', a', h')\right) c_k(s', a', h')\\
		&= \sum_{h=1}^{H}\sum_{(s', a')\in\tilSA}\sum_{h'\geq h}q_{\tilpi}(s', a', h')c_k(s', a', h') = \sum_{h'=1}^{H}\sum_{(s', a')\in\tilSA}\sum_{h\leq h'}q_{\tilpi}(s', a', h')c_k(s', a', h') \\
		&= \sum_{h'=1}^{H}\sum_{(s', a')\in\tilSA} h'\cdot q_{\tilpi}(s', a', h')c_k(s', a', h')
		= \inner{q_{\tilpi}}{\h{c_k}},
	\end{align*}
	where in the third line we use the equality $\sum_{s\in\calS\cup\{s_f\}} q_{\tilpi}(s, h)q_{\tilpi,(s, h)}(s', a', h') = q_{\tilpi}(s', a', h')$ by definition (since both sides are the probability of visiting $(s', a', h')$).
	This proves the last equality and completes the proof.
\end{proof}

\begin{lemma}{\citep[Lemma E.1]{rosenberg2020adversarial}}\label{lem:hitting}
	Let $\pi$ be a policy with expected hitting time at most $\tau$ starting from any state.
	Then, the probability that $\pi$ takes more than $m$ steps to reach the goal state is at most $2e^{-\frac{m}{4\tau}}$.
\end{lemma}

\section{Omitted details for \pref{sec:full-unknown}}
\label{app:full-unknown}
% !TEX root = main.tex

In this section, we provide all omitted algorithms and proofs for \pref{sec:full-unknown}.
We first introduce a lemma of a Bellman type law of total variance~\citep{azar2017minimax}, which is the key in obtaining a regret bound without $\Tmax$ dependency in the dominating term.
Then in \pref{sec:trans-est}, we provide details on transition estimation and prove the main lemma (\pref{lem:transition-bias}) that gives a data dependent upper bound on the transition estimation error.
Finally we prove \pref{thm:full-unknown} in \pref{sec:thm-full-unknown}.

In the rest of this section, we use the shorthand $\V_k$ for $\V[\cdot|\calF_k]$.
\begin{lemma}\label{lem:var}
	Consider executing policy $\tilpi_k$ induced by occupancy measure $q_k$ in $\tilM$ in episode $k$ with an arbitrary  cost function $c_k:\SA\rightarrow[0, \infty)$, and define $\fV_k(s, a, h)=\V_{S'\sim P(\cdot|s, a, h)}[J^{\tilpi_k}_k(S', h+1)]$.
	Then, $\inner{q_k}{\fV_k} \leq \V_k\sbr{\inner{\tilN_k}{c_k}}$.
\end{lemma}
\begin{proof}
	%Denote $\tilN_k(s, a)=\sum_{h=1}^H\tilN(s, a, h)$ the number of visits to $(s, a)$ in $\tilM$ by executing $\tilpi_k$. Then,
	First, we have
	\begin{align*}
		&\V_k\sbr{\inner{\tilN_k}{c_k}} = \E_k\sbr{ \rbr{\sum_{h=1}^H c_k(s^h, a^h, h) - J^{\tilpi_k}_k(s^1, 1)}^2}\\
		&= \E_k\sbr{ \rbr{ \sum_{h=2}^H c_k(s^h, a^h, h) - J^{\tilpi_k}_k(s^2, 2) + c_k(s^1, a^1, 1) + J^{\tilpi_k}_k(s^2, 2) - J^{\tilpi_k}_k(s^1, 1) }^2 }.
	\end{align*}
	Note that $c_k(s^1, a^1, 1) + J^{\tilpi_k}_k(s^2, 2) - J^{\tilpi_k}_k(s^1, 1)\in\sigma(s^1, a^1, s^2)$ and $\E_k\sbr{\left. \sum_{h=2}^H c_k(s^h, a^h, h) - J^{\tilpi_k}_k(s^2, 2) \right| s^1, a^1, s^2}=0$.
	Thus,
	\begin{align*}
		&\E_k\sbr{ \rbr{ \sum_{h=2}^H c_k(s^h, a^h, h) - J^{\tilpi_k}_k(s^2, 2) + c_k(s^1, a^1, 1) + J^{\tilpi_k}_k(s^2, 2) - J^{\tilpi_k}_k(s^1, 1) }^2 }\\
		&= \E_k\sbr{ \rbr{\sum_{h=2}^H c_k(s^h, a^h, h) - J^{\tilpi_k}_k(s^2, 2)}^2 } + \E_k\sbr{\rbr{ c_k(s^1, a^1, 1) + J^{\tilpi_k}_k(s^2, 2) - J^{\tilpi_k}_k(s^1, 1) }^2 }.
	\end{align*}
	Moreover, $c_k(s^1, a^1, 1) + \sum_{s'}P(s'|s^1, a^1, 1)J^{\tilpi_k}_k(s', 2) - J^{\tilpi_k}_k(s^1, 1)\in\sigma(s^1, a^1)$ and 
	$$\E_k\sbr{\left. J^{\tilpi_k}_k(s^2, 2) - \sum_{s'}P(s'|s^1, a^1, 1)J^{\tilpi_k}_k(s', 2) \right| s^1, a^1}=0.$$
	Thus,
	\begin{align*}
		&\E_k\sbr{\rbr{ c_k(s^1, a^1, 1) + J^{\tilpi_k}_k(s^2, 2) - J^{\tilpi_k}_k(s^1, 1) }^2 } \\
		&= \E_k\sbr{\rbr{ c_k(s^1, a^1, 1) + \sum_{s'}P(s'|s^1, a^1, 1)J^{\tilpi_k}_k(s', 2) - J^{\tilpi_k}_k(s^1, 1) }^2}  + \E_k\sbr{\rbr{ J^{\tilpi_k}_k(s^2, 2) - \sum_{s'}P(s'|s^1, a^1, 1)J^{\tilpi_k}_k(s', 2) }^2} \\
		&\geq \E_k\sbr{\rbr{ J^{\tilpi_k}_k(s^2, 2) - \sum_{s'}P(s'|s^1, a^1, 1)J^{\tilpi_k}_k(s', 2) }^2} = \E_k[\fV_k(s^1, a^1, 1)].
	\end{align*}
	Plugging these back, we obtain:
	\begin{align*}
		\V_k\sbr{\inner{\tilN_k}{c_k}} &= \E_k\sbr{ \rbr{\sum_{h=1}^H c_k(s^h, a^h, h) - J^{\tilpi_k}_k(s^1, 1)}^2}\\
		&\geq \E_k\sbr{\rbr{\sum_{h=2}^H c_k(s^h, a^h, h) - J^{\tilpi_k}_k(s^2, 2)}^2} + \E_k[\fV_k(s^1, a^1, 1)].
	\end{align*}
	Doing this repeatedly, we get:
	\begin{align*}
		\V_k\sbr{\inner{\tilN_k}{c_k}} \geq \E_k\sbr{\sum_{h=1}^H \fV_k(s^h, a^h, h)} = \sumsah q_k(s, a, h)\fV_k(s, a, h) = \inner{q_k}{\fV_k},
	\end{align*}
	finishing the proof.
%	Therefore, by \pref{lem:deviation_loop_free},
%	\begin{align*}
%		\inner{q_k}{\fV_k} \leq \V\sbr{\inner{\tilN_k}{c_k} } \leq \E\sbr{\inner{\tilN_k}{c_k}^2 } \leq 2\inner{q_k}{\h{c_k}}.
%	\end{align*}
\end{proof}

\subsection{Omitted Details for Transition Estimation}
\label{sec:trans-est}
In this subsection, we introduce a sub-procedure $\TransEst$ (\pref{alg:update_transition}) used in all algorithms proposed in this paper for transition estimation. It takes the learner's trajectory as input and outputs a confidence set $\calP_{k+1}$ (and additionally the counter $\tilN_k$ for the bandit setting) at the end of episode $k$.
%Then, we present the related lemmas for bounding estimation error.
We first introduce some notations. %for transition estimation.
Denote by $\tcs$ the set of valid transitions $\tilP$ for $\tilM$ based on its layer structure.
Also define the entries of unknown transition in $\tilM$ as $\unk=\Gamma\times(\calS\cup \{g\})\times[H_1-1]$, such that $\tilP(s'|s, a, h)=P(s'|s, a),\forall (s, a), s', h\in\unk$.

We implement the confidence sets of transition function by maintaining a separate (empirical) Bernstein confidence bound for each state, action, and next state in $M$ around its empirical estimate:
\begin{align*}
	\epsilon_k(s, a, s') &= 4\sqrt{\bar{P}_k(s'| s, a)\Ak(s, a)} + 28\Ak(s, a), %\label{eq:transition epsilon}
\end{align*}
for any $(s, a)\in\SA, s'\in\calS\cup\{g\}$, where $\bar{P}_k(s'| s, a) = \frac{\Mk(s, a, s')}{\Nk(s, a)}$ is the empirical transition estimation, $\Nk(s, a)=\max\{1, \N_k(s, a)\}$, $\N_k(s, a)$ is the number of visits to $(s, a)$ before episode $k$, $\Mk(s, a. s')$ is the number of visits to $(s, a), s'$ before episode $k$, and $\Ak(s, a) = \frac{\ln(\frac{HKSA}{\delta})}{\Nk(s, a)}$.
We then define the confidence set in episode $k$ as follows:
\begin{align}
	\label{eq:conf-set}
	\calP_k &= \Big\{\hatP\in \tcs: \lvert \hatP(s'| s, a, h) - \bar{P}_k(s'| s,a) \rvert \leq \epsilon_k(s, a, s'), \forall(s, a), s', h\in \unk \Big\}.
\end{align}

\begin{algorithm}[t]
\caption{$\TransEst(\N, \M, \delta, H_1, H_2, \tau)$}
\label{alg:update_transition}
	\textbf{Input:} counters $\N, \M$, confidence parameter $\delta$, horizon parameters $H_1, H_2$, trajectory $\tau=\{s^1, a^1, s^2, a^2, \ldots, a^{h-1}, s^h\}$.
	
	\textbf{Initialization:} $\tilN(s, a)=0$, $\N_{k+1}(s, a)=\N_{k}(s, a)$, and $\M_{k+1}(s, a)=\M_{k}(s, a)$ for any $(s, a)\in\tilSA$.
	
	\For{$t=1,\ldots, h-1$}{
		Update counters: $\N_{k+1}(s^t, a^t)\leftarrow \N_k(s^t, a^t)+1$, $\M_{k+1}(s^t, a^t, s^{t+1})\leftarrow \M_k(s^t, a^t, s^{t+1}) + 1$, $\tilN(s^t, a^t)\leftarrow\tilN(s^t, a^t)+1$.
	}
		
	Compute confidence set $\calP_{k+1}$ defined in \pref{eq:conf-set}.
	
	\If{$s^h\neq g$}{
		$\tilN(s_f, a_f)\leftarrow H_2$.
	}
	
	\textbf{Return:} $\calP_{k+1}$ (and additionally $\tilN_k$ for the bandit setting).
\end{algorithm}

Next, we introduce several lemmas useful for bounding the bias of transition estimation.

\begin{lemma}{(Lemma 4.2 in \citep{cohen2020near})}\label{lem:transition-interval}
	With probability at least $1-\delta$, $P\in\calP_k$ for all $k$.
\end{lemma}

\begin{lemma}
	\label{lem:transition-bound}
	Assume $\hatP\in\calP_k$.
	Then, under the event of \pref{lem:transition-interval} we have
	\begin{align*}
		|\hatP(s'|s, a, h) - P(s'|s, a, h)| \leq \Ind\{(s, a), s', h\in U\}\rbr{ 8\sqrt{P(s'|s, a, h)\Ak(s, a)} + 136\Ak(s, a)} \defeq \epsilon^\star_k(s, a, s').
	\end{align*}
\end{lemma}
\begin{proof}
	First note that $\hatP(s'|s, a, h)=P(s'|s, a, h)$ when $(s, a), s', h\notin \unk$.
	When $(s, a), s', h\in\unk$, the result follows from $P(s'|s, a, h)=P(s'|s, a)$, definition of $\epsilon_k(s, a, s')$ and  \citep[Lemma B.13]{cohen2020near}.
\end{proof}

\begin{lemma}\label{lem:om-diff}
	For any occupancy measure $q$ and $\hatq$ induced by the same policy $\pi$ but different transition functions $P$ and $\hatP$ respectively, we have:
	\begin{align*}
		\hatq(s, a, h) - q(s, a, h) = \sum_{(s', a'), s''}\sum_{m=1}^{h-1}q(s', a', m)(\hatP(s''|s', a', m) - P(s''|s', a', m))\hatq_{(s'', m+1)}(s, a, h),
	\end{align*}
	where $\hatq_{(s'', m+1)}$ is the occupancy measure with respect to $\pi$, $\hatP$, and initial state $(s'', m+1)$.
	As a result, under the event of \pref{lem:transition-interval},
	\begin{align*}
		\abr{\inner{\hatq-q}{c}} &= \abr{\sum_{(s, a), s', h}q(s, a, h)(\hatP(s'|s, a, h) - P(s'|s, a, h))J^{\hatP, \pi}(s', h+1)}\\
		&\leq \inner{|\hatq-q|}{c} \leq H\sum_{(s, a), s', h\in\unk}q(s, a, h)\epsilon^\star_k(s, a, s').
	\end{align*}
\end{lemma}
\begin{proof}
	We prove the first statement by induction on $h$.
	Denote by $q_{(s, h)}$ the occupancy measure of $\pi_q$ with initial state $(s, h)$.
	When $h=1$, the statement is true by $q(s, a, h)=\hatq(s, a, h)=\pi(a|s, 1)\Ind\{s=s_0\}$.
	For the induction step with $h>1$:
	\begin{align*}
		&\hatq(s, a, h) - q(s, a, h) = \pi(a|s, h)(\hatq(s, h) - q(s, h))\\
		&= \pi(a|s, h)\sumsa[s', a']\rbr{\hatP(s|s', a', h-1)\hatq(s', a', h-1) - P(s|s', a', h-1)q(s', a', h-1)} \tag{$q(s, h)=\sumsa[s', a']P(s|s', a', h-1)q(s', a', h-1)$}\\
		&= \underbrace{\pi(a|s, h)\sumsa[s', a']\hatP(s|s', a', h-1)\rbr{\hatq(s', a', h-1) - q(s', a', h-1)}}_{\chi_1}\\
		&\qquad\qquad\qquad + \underbrace{\pi(a|s, h)\sumsa[s', a']q(s', a', h-1)\rbr{\hatP(s|s', a', h-1)-P(s|s', a', h-1)}}_{\chi_2}\\
	\end{align*}
	For $\chi_1$, by the induction step, we have:
	\[
		\hatq(s', a', h-1) - q(s', a', h-1) = \sum_{(s'', a''), s'''}\sum_{m=1}^{h-2}q(s'', a'', m)(\hatP-P)(s'''|s'', a'', m)\hatq_{(s''', m+1)}(s', a', h-1).
	\]
	Thus, by $\sum_{(s', a')}\hatq_{(s''', m+1)}(s', a', h-1)\hatP(s|s', a', h-1)\pi(a|s, h)=\hatq_{(s''', m+1)}(s, a, h)$:
	\begin{align*}
		\chi_1 &= \pi(a|s, h)\sumsa[s', a']\hatP(s|s', a', h-1)\sum_{(s'', a''), s'''}\sum_{m=1}^{h-2}q(s'', a'', m)(\hatP-P)(s'''|s'', a'', m)\hatq_{(s''', m+1)}(s', a', h-1)\\
		&= \sum_{(s'', a''), s'''}\sum_{m=1}^{h-2}q(s'', a'', m)\rbr{\hatP(s'''|s'', a'', m) - P(s'''|s'', a'', m)}\hatq_{(s''', m+1)}(s, a, h).
	\end{align*}
	For $\chi_2$, note that $\pi(a|s'', h)\Ind\{s''=s\}=q_{(s'', h)}(s, a, h)$.
	Thus,
	\begin{align*}
		&\pi(a|s, h)\rbr{\hatP(s|s', a', h-1)-P(s|s', a', h-1)}\\
		&= \sum_{s''}\pi(a|s'', h)\Ind\{s''=s\}\rbr{\hatP(s''|s', a', h-1)-P(s''|s', a', h-1)}\\
		&= \sum_{s''}q_{(s'', h)}(s, a, h)\rbr{\hatP(s''|s', a', h-1)-P(s''|s', a', h-1)}.\\
	\end{align*}
	and,
	\begin{align*}
		\chi_2 = \sum_{(s', a'), s''}q(s', a', h-1)(\hatP-P)(s''|s', a', h-1)q_{(s'', h)}(s, a, h).
	\end{align*}
	Plugging these back and changing variables $(s'', a''), s'''$ in $\chi_1$ to $(s', a'), s''$, we get:
	\begin{align*}
		\hatq(s, a, h) - q(s, a, h) &= \chi_1 + \chi_2\\ 
		&= \sum_{(s', a'), s''}\sum_{m=1}^{h-1}q(s', a', m)\rbr{\hatP(s''|s', a', m) - P(s''|s', a', m)}\hatq_{(s'', m+1)}(s, a, h).
	\end{align*}
	This completes the proof of the first statement.
	For the second statement,
	\begin{align*}
		\inner{\hatq-q}{c} &= \sum_{(s, a), h}\sum_{(s', a'), s''}\sum_{m=1}^{h-1}q(s', a', m)(\hatP(s''|s', a', m) - P(s''|s', a', m))\hatq_{(s'', m+1)}(s, a, h)c(s, a, h)\\
		&= \sum_{m=1}^H\sum_{(s', a'), s''}q(s', a', m)(\hatP(s''|s', a', m) - P(s''|s', a', m))\sum_{(s, a), h > m}\hatq_{(s'', m+1)}(s, a, h)c(s, a, h)\\
		&= \sum_{m=1}^H\sum_{(s', a'), s''}q(s', a', m)(\hatP(s''|s', a', m) - P(s''|s', a', m))J^{\hatP, \pi}(s'', m+1)\\
		&= \sum_{h=1}^H\sum_{(s, a), s'}q(s, a, h)(\hatP(s'|s, a, h) - P(s'|s, a, h))J^{\hatP, \pi}(s', h+1). \tag{change of variables} 
		%&\leq H\sum_{(s, a), s', h\in\unk}q(s, a, h)\epsilon^\star_k(s, a, s'),
	\end{align*}
	Similarly, $\inner{|\hatq-q|}{c} \leq \sum_{h=1}^H\sum_{(s, a), s'}q(s, a, h)\abr{\hatP(s'|s, a, h) - P(s'|s, a, h)}J^{\hatP, \pi}(s', h+1)$.
	Applying the fact $J^{\hatP, \pi}(s', h+1)\leq H$ and \pref{lem:transition-bound} completes the proof.
\end{proof}

\begin{lemma}\label{lem:transition-sum}
	%Denote by $n_k(s, a)$ the number of visits to $(s, a)\in\SA$ during the first $H_1$ steps of episode $k$ so that $n_k(s,a) = \N_{k+1}(s, a) - \N_k(s, a)$. Then 
	With probability at least $1-\delta$ we have: $\sumsa \sumk\frac{q_k(s, a)}{\Nk(s, a)} = \tilO{SAH}$.
\end{lemma}

\begin{proof}
We first prove that for each $(s,a)$, 
$
\sumk\frac{\tilN_k(s, a)}{\Nk(s, a)} = \order(\ln (HK) + H)
$.
Indeed, we have
\begin{align*}
\sumk\frac{\tilN_k(s, a)}{\Nk(s, a)}
&= \sumk\frac{\tilN_{k}(s, a)}{\N_{k+1}^+(s, a)} + \sumk\tilN_{k}(s, a)\rbr{\frac{1}{\N_{k}^+(s, a)} - \frac{1}{\N_{k+1}^+(s, a)}} \\
&\leq \sumk\frac{\tilN_{k}(s, a)}{\N_{k+1}^+(s, a)} + H\sumk \rbr{\frac{1}{\N_{k}^+(s, a)} - \frac{1}{\N_{k+1}^+(s, a)}} \\
&\leq \sumk\frac{\tilN_{k}(s, a)}{\N_{k+1}^+(s, a)} + H \\
&= \order(\ln (HK) + H).  \tag{$\tilN_{k}(s, a) = \N_{k+1}(s, a) - \N_k(s, a)$}
\end{align*}
%
%	Define $\calI$ the set of episodes where there exists $(s, a)\in\SA$ such that $n_k(s, a)>\Nk(s, a)$.
%	Note that by $\sumsa n_k(s, a)\leq H$, we have $|\calI|\leq SAH$.
%	Hence, the first result can be showed by the following:
%	\begin{align*}
%		\sumsa \sumk\frac{n_k(s, a)}{\Nk(s, a)} &= \sumsa \rbr{\sum_{k\in\calI} \frac{n_k(s, a)}
%		{\Nk(s, a)}+ \sum_{k\notin\calI} \frac{n_k(s, a)}
%		{\Nk(s, a)}} = \bigO{SAH^2 + SA\ln(HK)} = \tilO{SAH^2}.
%	\end{align*}
Therefore, applying the fact $\E_k[\tilN_k(s, a)] = q_k(s,a)$ and  \pref{lem:e2r}, we have with probability at least $1-\delta$:
	\begin{align*}
		\sumsa\sumk\frac{q_k(s, a)}{\Nk(s, a)} = \sumk\E_k\sbr{\sumsa\frac{n_k(s, a)}{\Nk(s, a)}} \leq 2\sumsa\sumk\frac{n_k(s, a)}{\Nk(s, a)} + \bigO{H\ln\frac{2K}{\delta}} = \tilO{SAH},
	\end{align*}
	competing the proof.
\end{proof}

\begin{lemma}
	\label{lem:transition-bias}
	Consider interacting with the environment for $K$ episodes, where in episode $k$ the executed policy is $\tilpi_k$ and the cost function denoted by $c_k: \tilSA\rightarrow[0, 1]$ is arbitrary.
	Also let $P_k$ be any transition function within the confidence set $\calP_k$ defined in \pref{eq:conf-set},
	and define $q_k=q_{P, \tilpi_k}$ and $\hatq_k=q_{P_k, \tilpi_k}$.
	Then with probability at least $1-4\delta$, we have
	\begin{align*}
		\sumk\abr{\inner{q_k-\hatq_k}{c_k}} &\leq 32\sqrt{S^2A\ln^2\rbr{\frac{HKSA}{\delta}} \rbr{\sumk\V_k\sbr{\inner{\tilN_k}{c_k}} + \tilO{H^3\sqrt{K}}}} + \tilO{\Qdtd}\\
		&\leq 32\sqrt{S^2A\ln^2\rbr{\frac{HKSA}{\delta}} \rbr{\sumk\inner{q_k}{\h{c_k}} + \tilO{H^3\sqrt{K}}}} + \tilO{\Qdtd} \\
		&\leq 16\rbr{\lambda'S^2A\rbr{\sumk\inner{q_k}{\h{c_k}} + \tilO{H^3\sqrt{K}} } + \frac{\ln^2\rbr{\frac{HKSA}{\delta}}}{\lambda'}} + \tilO{\Qdtd}.
	\end{align*}
	where $\lambda'>0$ is arbitrary.
\end{lemma}
\begin{proof}
	Define $\mu_k(s, a, h)=\sum_{s'}P(s'|s, a, h)J^{\tilpi_k}_k(s', h+1)$. 
	Then by \pref{lem:om-diff}:
	\begin{align*}
		&\sumk\abr{\inner{q_k-\hatq_k}{c_k}} = \sumk\abr{\sum_{(s, a), s', h}q_k(s, a, h)(P(s'|s, a, h)-P_k(s'|s, a, h))J^{P_k, \tilpi_k}_k(s', h+1)}\\
		&\leq\sumk\abr{\sum_{(s, a), s', h}q_k(s, a, h)(P(s'|s, a, h)-P_k(s'|s, a, h))J^{\tilpi_k}_k(s', h+1)} + \tilO{\Qdtd} \tag{\pref{lem:double transition difference}}\\
		&=\sumk\abr{\sum_{(s, a), s', h}q_k(s, a, h)(P(s'|s, a, h)-P_k(s'|s, a, h))\rbr{J^{\tilpi_k}_k(s', h+1) - \mu_k(s, a, h)} } + \tilO{\Qdtd} \tag{$\sum_{s'}P(s'|s, a, h)-P_k(s'|s, a, h)=0$ and $\mu_k(s, a, h)$ is independent of $s'$}\\
		&\leq\sumk\sum_{(s, a), s', h\in\unk}q_k(s, a, h)\epsilon_k^{\star}(s, a, s')\left|J^{\tilpi_k}_k(s', h+1) - \mu_k(s, a, h)\right| \tag{$P(s'|s, a, h)-P_k(s'|s, a, h)\leq \epsilon_k^{\star}(s, a, s')$ by \pref{lem:transition-bound}} + \tilO{\Qdtd}\\
		&\leq 8\sumk\sum_{(s, a), s', h\in\unk}q_k(s, a, h)\sqrt{P(s'|s, a, h)\Ak(s, a)\rbr{ J^{\tilpi_k}_k(s', h+1) - \mu_k(s, a, h) }^2}\\
		&\qquad\qquad + \tilO{HS\sumk\sumsa\frac{q_k(s, a)}{\Nk(s, a)} + \Qdtd} \tag{definition of $\epsilon_k^{\star}$ from \pref{lem:transition-bound}} \\
		&\leq 8\sumk\E_k\sbr{\sum_{(s, a), s', h\in\unk}\tilN_k(s, a, h)\sqrt{P(s'|s, a, h)\Ak(s, a)\rbr{ J^{\tilpi_k}_k(s', h+1) - \mu_k(s, a, h) }^2}} + \tilO{\Qdtd} \tag{\pref{lem:transition-sum}}\\
		&= 8\sumk \E_kX_k + \tilO{\Qdtd} \tag{$X_k=\sum_{(s, a), s', h\in\unk}\tilN_k(s, a, h)\sqrt{P(s'|s, a, h)\Ak(s, a)\rbr{ J^{\tilpi_k}_k(s', h+1) - \mu_k(s, a, h) }^2}$}.
	\end{align*}
	%Denote $\calI$ the set of episodes where there exists $(s, a)\in\SA$ such that $\tilN_k(s, a)>\Nk(s, a)$.
	%Note that $0 \leq X_k = \tilo{ \sum_{(s, a), s', h\in\unk}\tilN_k(s, a, h)H } = \tilo{H^2S}$, and $\abs{\calI}\leq SAH$.
	%Hence, by \pref{lem:e2r}, with probability $1-\delta$:
	Note that $0 \leq X_k = \tilo{ \sum_{(s, a), s', h\in\unk}\tilN_k(s, a, h)H } = \tilo{H^2S}$.
	Hence, by \pref{lem:e2r}, with probability at least $1-\delta$:
	\begin{align*}
		&\sumk \E_kX_k \leq  2\sumk\sum_{(s, a), s', h\in\unk}\tilN_k(s, a, h)\sqrt{P(s'|s, a, h)\Ak(s, a)\rbr{ J^{\tilpi_k}_k(s', h+1) - \mu_k(s, a, h) }^2} + \tilO{H^2S}\\
		&\leq 2\sumk\sum_{(s, a), s', h\in\unk}\tilN_k(s, a, h)\sqrt{P(s'|s, a, h)\A_{k+1}(s, a)\rbr{ J^{\tilpi_k}_k(s', h+1) - \mu_k(s, a, h) }^2}\\
		&\qquad\qquad + 2H^2\sumk\sum_{(s, a), s', h\in\unk}\rbr{\sqrt{\Ak(s, a)} - \sqrt{\A_{k+1}(s, a)}} + \tilO{ H^2S }\\
		&\leq 2\sqrt{\sumk\sum_{(s, a), s', h\in\unk}\tilN_k(s, a, h)P(s'|s, a, h)\rbr{ J^{\tilpi_k}_k(s', h+1) - \mu_k(s, a, h) }^2}\sqrt{\sumk\sum_{(s, a), s', h\in\unk}\tilN_k(s, a, h)\A_{k+1}(s, a) } \\
		&\qquad\qquad + \tilO{H^3S^2A} \tag{Cauchy-Schwarz inequality}.
	\end{align*}
	Note that %by \citep[Lemma B.18]{cohen2020near},
	\begin{align*}
		&\sumk\sum_{(s, a), s', h\in\unk}\tilN_k(s, a, h)\A_{k+1}(s, a) = \ln\rbr{\frac{HKSA}{\delta}}\sumk\sum_{(s, a), s', h\in\unk}\frac{\tilN_k(s, a, h)}{\N_{k+1}(s, a)}\\
		&= S\ln\rbr{\frac{HKSA}{\delta}}\sumk\sumsa\frac{\N_{k+1}(s, a)-\Nk(s, a)}{\N_{k+1}(s, a)} \leq 2S^2A\ln^2\rbr{\frac{HKSA}{\delta}}.
	\end{align*}
	Moreover, define $\fV_k(s, a, h)=P(s'|s, a, h)\rbr{ J^{\tilpi_k}_k(s', h+1) - \mu_k(s, a, h) }^2$, with probability at least $1-\delta$,
	\begin{align*}
		&\sumk\sum_{(s, a), s', h\in\unk}\tilN_k(s, a, h)P(s'|s, a, h)\rbr{ J^{\tilpi_k}_k(s', h+1) - \mu_k(s, a, h) }^2\\
		&=\sumk\sum_{(s, a), s', h\in\unk}\tilN_k(s, a, h)\fV_k(s, a, h) \leq\sumk\sum_{(s, a), h}\tilN_k(s, a, h)\fV_k(s, a, h)\\
		&=\sumk\sum_{(s, a), h}q_k(s, a, h)\fV_k(s, a, h) + \sumk\sum_{(s, a), h}(\tilN_k(s, a, h) - q_k(s, a, h))\fV_k(s, a, h)\\
		&\leq 2\sumk\V\sbr{\inner{\tilN_k}{c_k}} + \tilO{H^3\sqrt{K}},
	\end{align*}
	where in the last inequality we apply \pref{lem:var} and \pref{lem:azuma} with
	\begin{align*}
		\sum_{(s, a), h}|\tilN_k(s, a, h) - q_k(s, a, h)|\fV_k(s, a, h) \leq \sum_{(s, a), h}\tilN_k(s, a, h)H^2 + \sum_{(s, a), h}q_k(s, a, h)H^2 \leq 2H^3.
	\end{align*}
	Therefore, combining everything we arrive at
	\begin{align*}
		\abr{\sumk\inner{q_k-\hatq_k}{c_k}} &\leq 32\sqrt{S^2A\ln^2\rbr{\frac{HKSA}{\delta}}\rbr{\sumk\V\sbr{\inner{\tilN_k}{c_k}} + \tilO{H^3\sqrt{K}}}} + \tilO{\Qdtd}.\\
	\end{align*}
	When the value of $c_k$ is at most $1$, we have by \pref{lem:deviation_loop_free} and $\V\sbr{\inner{\tilN_k}{c_k}}\leq \E_k\sbr{\inner{\tilN_k}{c_k}^2}$:
	\begin{align*}
		\abr{\sumk\inner{q_k-\hatq_k}{c_k}} &\leq 32\sqrt{S^2A\ln^2\rbr{\frac{HKSA}{\delta}}\rbr{\sumk\inner{q_k}{\h{c_k}} + \tilO{H^3\sqrt{K}}}} + \tilO{\Qdtd}\\
		&\leq 16\rbr{\lambda'S^2A\rbr{\sumk\inner{q_k}{\h{c_k}} + \tilO{H^3\sqrt{K}} } + \frac{\ln^2\rbr{\frac{HKSA}{\delta}}}{\lambda'}} + \tilO{\Qdtd},
	\end{align*}
	where the last step is by AM-GM inequality.
\end{proof}

\begin{lemma}
	\label{lem:double transition difference}
	Under the event of \pref{lem:transition-interval},
	\begin{align*}
		\sumk\abr{\sum_{(s, a), s', h}q_k(s, a, h)(P(s'|s, a, h)-P_k(s'|s, a, h))\rbr{J^{P_k, \tilpi_k}_k(s', h+1) -  J^{P, \tilpi_k}_k(s', h+1) }} = \tilO{\Qdtd}.
	\end{align*}
\end{lemma}
\begin{proof}
	Define $q_{k,(s', h+1)}=q_{P,\tilpi_k,(s',h+1)}$.
	Then,
	\begin{align*}
		&\sumk \abr{\sum_{(s, a), s', h}q_k(s, a, h)(P(s'|s, a, h)-P_k(s'|s, a, h))\rbr{J^{P_k, \tilpi_k}_k(s', h+1) -  J^{P, \tilpi_k}_k(s', h+1) } }\\
		&\leq \sumk\sum_{(s, a), s', h\in \unk}q_k(s, a, h)\epsilon_k^\star(s, a, s')\abr{\inner{q_{P_k, \tilpi_k, (s', h+1)}-q_{P, \tilpi_k, (s', h+1)}}{c_k}} \tag{\pref{lem:transition-bound} and $J^{P', \tilpi_k}_k(s',h+1)=\inner{q_{P', \tilpi_k, (s', h+1)}}{c_k}$}\\
		&\leq \sumk\sum_{(s, a), s', h\in \unk}q_k(s, a, h)\epsilon_k^\star(s, a, s')\sum_{(\tils, \tila), \tils', h'\in\unk}q_{k,(s',h+1)}(\tils, \tila, h')\epsilon_k^\star(\tils, \tila, \tils')H \tag{\pref{lem:om-diff}}\\
		&=\tilO{H\sumk\sum_{\substack{(s, a), s', h\in \unk \\ (\tils, \tila), \tils', h'\in\unk } }q_k(s, a, h)\sqrt{\frac{P(s'|s, a, h)}{\Nk(s, a)}}q_{k,(s', h+1)}(\tils, \tila, h')\sqrt{\frac{P(\tils'|\tils, \tila, h')}{\Nk(\tils, \tila)}} }\\
		&= \tilO{H\sumk\sum_{\substack{(s, a), s', h\in \unk \\ (\tils, \tila), \tils', h'\in\unk } }\sqrt{\frac{q_k(s, a, h)P(\tils'|\tils, \tila, h')q_{k, (s', h+1)}(\tils, \tila, h')}{\Nk(s, a)}}\sqrt{\frac{q_k(s, a, h)P(s'|s, a, h)q_{k, (s', h+1)}(\tils, \tila, h')}{\Nk(\tils, \tila)}} }\\
		&= \tilO{H\sqrt{\sum_{\substack{k, (s, a), s', h\in \unk \\ (\tils, \tila), \tils', h'\in\unk }}\frac{q_k(s, a, h)P(\tils'|\tils, \tila, h')q_{k, (s', h+1)}(\tils, \tila, h')}{\Nk(s, a)}}\sqrt{\sum_{\substack{k, (s, a), s', h\in \unk \\ (\tils, \tila), \tils', h'\in\unk }}\frac{q_k(s, a, h)P(s'|s, a, h)q_{k, (s', h+1)}(\tils, \tila, h')}{\Nk(\tils, \tila)}} } \tag{Cauchy-Schwarz inequality}\\
		&= \tilO{ H\sqrt{HS\sum_{k,(s, a)}\frac{q_k(s, a)}{\Nk(s, a)} }\sqrt{HS\sum_{k, (\tils, \tila)}\frac{q_k(\tils, \tila)}{\Nk(\tils, \tila)} } } = \tilO{\Qdtd}. \tag{\pref{lem:transition-sum}}
	\end{align*}
	where in the first square root of the last line we simply sum over $s', h, (\tils, \tila), \tils', h'$, and in the second square root of the last line we apply $\sum_{(s, a), s'}q_k(s, a, h)P(s'|s, a, h)q_{k,(s', h+1)}(\tils, \tila, h')=q_k(\tils, \tila, h')$ for any $\tils, \tila, h<h'$.
\end{proof}

\subsection{\pfref{thm:full-unknown}}
\label{sec:thm-full-unknown}

\begin{proof}
	First apply \pref{lem:loop-free}: with probability at least $1-\delta$, we have
	\begin{align*}
		R_K \leq \sumk\inner{\tilN_k-\tiloptq}{c_k} + \tilO{D^{3/2}S^2A(\ln\tfrac{1}{\delta})^2}.
	\end{align*}
	Define $P_k=P_{\hatq_k}$, and $q_k=q_{P, \tilpi_k}$.
	We decompose the regret in $\tilM$ into three terms:
	\begin{align}
		\sumk\inner{\tilN_k - \tiloptq}{c_k} = \sumk\inner{\tilN_k-q_k}{c_k} + \sumk\inner{q_k-\hatq_k}{c_k} + \sumk\inner{\hatq_k - \tiloptq}{c_k}\label{eq:full-unknown-reg}
	\end{align}
	For the first term, by \pref{lem:freedman} with $\frac{\lambda}{4} \leq \frac{1}{H}$ (because $K\geq 16S^2AH^2$) and \pref{lem:deviation_loop_free}, we have with probability at least $1-\delta$:
	\begin{align*}
		\sumk\inner{\tilN_k-q_k}{c_k} \leq \frac{\lambda}{2}\sumk\inner{q_k}{\h{c_k}} + \frac{4\ln(1/\delta)}{\lambda} .
	\end{align*}
	For the second term, by \pref{lem:transition-bias}, with probability $1-4\delta$, we have for any $\lambda'>0$:
	\begin{align*}
		\sumk\inner{q_k-\hatq_k}{c_k} &\leq 16\rbr{\lambda'S^2A\rbr{\sumk\inner{q_k}{\h{c_k}} + \tilO{H^3\sqrt{K}} } + \frac{\ln^2\rbr{\frac{HKSA}{\delta}}}{\lambda'} } + \tilO{H^3S^2A}.
	\end{align*}
	For the third term, by standard OMD analysis (see for example Eq.~(12) of~\citep{rosenberg2020adversarial}):
	\begin{align}
		\sumk\inner{\qfeat_k - \qfeat_{\tiloptpi}}{c_k} &\leq D_{\regz}(\qfeat_{\tiloptpi}, \qfeat_1) + \sumk\inner{\qfeat_k - \qfeat'_{k+1}}{c_k},\label{eq:omd}
	\end{align}
	where $\qfeat'_{k+1} = \argmin_{\phi\in\fR^{\tilSA\times\calS\times[H]}}\inner{\phi}{c_k} + D_{\psi}(\phi, \qfeat_k)$.
	It can be shown that $\qfeat'_{k+1}(s, a, s', h)=\qfeat_k(s, a, s', h)e^{-\eta c_k(s, a)}$ with the choice of the entropy regularizer.
	Applying the inequality $1-e^{-x}\leq x$, we get:
	\begin{align*}
		\sumk\inner{\qfeat_k-\qfeat'_{k+1}}{c_k} &\leq \eta\sumk\sum_{(s, a), s', h}\qfeat_k(s, a, s', h)c_k^2(s, a) \leq 2\eta\sumk\sumsa\hatq_k(s, a)c_k(s, a) = 2\eta\sumk\inner{\hatq_k}{c_k},
	\end{align*}
	where in the last inequality we apply $\qfeat_k(s, a, s', h) = (1+\lambda h)\hatq_k(s, a, s', h)\leq 2\hatq_k(s, a, s', h)$.
	To bound $D_{\psi}(\qfeat_{\tiloptpi}, \qfeat_1)$, note that $\inner{\nabla\psi(\qfeat_1)}{\qfeat_{\tiloptpi}-\qfeat_1}\geq 0$ by $\qfeat_1=\argmin_{\phi\in\Omega_1}\psi(\phi)$.
	Thus,
	\begin{align*}
		D_{\psi}(\qfeat_{\tiloptpi}, \qfeat_1) &\leq \psi(\qfeat_{\tiloptpi}) - \psi(\qfeat_1) \\
		&= \frac{1}{\eta}\sum_{(s, a), s', h}\qfeat_{\tiloptpi}(s, a, s', h)\ln\qfeat_{\tiloptpi}(s, a, s', h) - \frac{1}{\eta}\sum_{(s, a), s', h}\qfeat_1(s, a, s', h)\ln\qfeat_1(s, a, s', h)\\
		&\leq \frac{1}{\eta}\sum_{(s, a), s', h}\qfeat_{\tiloptpi}(s, a, s', h)\ln(2T) -\frac{2T}{\eta}\sum_{(s, a), s', h}\frac{\qfeat_1(s, a, s', h)}{2T}\ln\frac{\qfeat_1(s, a, s', h)}{2T}\\
		&\leq \frac{2T\ln(2T)}{\eta} + \frac{2T\ln(S^2AH)}{\eta} \leq \frac{2T\ln(2S^2AHT)}{\eta}.
	\end{align*}
	Substituting these back to \pref{eq:omd} and rearranging terms, we get:
	\begin{align}
		\sumk\inner{\hatq_k-\tiloptq}{c_k} &\leq \frac{1}{1-2\eta}\rbr{ \frac{2T\ln(2S^2AHT)}{\eta} + 2\eta\sumk\inner{\tiloptq}{c_k} + \sumk\lambda\inner{\tiloptq-\hatq_k}{\h{c_k}} }\notag\\
		&\leq \frac{4T\ln(2S^2AHT)}{\eta} + 4\eta DK + 2\lambda DTK - \lambda\sumk\inner{\hatq_k}{\h{c_k}},\label{eq:full unknown omd}
	\end{align}
	where we apply $\frac{1}{1-2\eta}\leq 2, \sumk\inner{\tiloptq}{c_k}\leq DK$ and $\sumk\inner{\tiloptq}{\h{c_k}}=\sumk\inner{\tiloptq}{J_k^{\tiloptpi}}\leq DTK$ in the last inequality.
	Substituting everything back to \pref{eq:full-unknown-reg} and set $\lambda'=\frac{1}{8}\sqrt{\frac{1}{S^2ADTK}}$, we get:
	\begin{align*}
		\sumk\inner{\tilN_k-\tiloptq}{c_k} &\leq \tilO{\frac{1}{\lambda} + \frac{1}{\lambda'} + \frac{T}{\eta} } + \rbr{16\lambda'S^2A - \frac{\lambda}{2}}\sumk\inner{q_k}{\h{c_k}}\\
		&+ \lambda\sumk\inner{q_k-\hatq_k}{\h{c_k}} + 4\eta DK + 2\lambda DTK + \tilO{\lambda'S^2AH^3\sqrt{K} + H^3S^2A}\\
		&= \tilO{\sqrt{S^2ADTK} + \sqrt{DTK} + \lambda\sqrt{S^2AH^2K} + H^3S^2A } = \tilO{\sqrt{S^2ADTK} + H^3S^2A},
	\end{align*}
	where in the last line we apply $\eta\leq \sqrt{\frac{T}{DK}}, \lambda=4\sqrt{\frac{S^2A}{DTK}}$, and
	\begin{align*}
		&\lambda\sumk\inner{q_k-\hatq_k}{\h{c_k}}=\tilO{ \sqrt{\frac{S^2A}{DTK}} \cdot \sqrt{S^2A\rbr{\sumk\inner{q_k}{\h{c_k}}+H^3\sqrt{K} } }}\\
		&=\tilO{\sqrt{ \frac{S^2A}{K}\cdot S^2A\rbr{ H^2K + H^3\sqrt{K} } }} = \tilO{H^3S^2A}.
	\end{align*}
\end{proof}

\section{Omitted details for \pref{sec:bandit-unknown}}
\label{app:bandit-unknown}
% !TEX root = main.tex

In this section, we provide all omitted algorithms and proofs for \pref{sec:bandit-unknown}.
We first prove a bound of the bias induced by our optimistic cost estimator in the bandit setting.
Then we gives the full proof of \pref{thm:bandit-unknown}, which has a similar structure to \citep[Theorem 11]{chen2020minimax}.

\begin{algorithm}[t]
\caption{Log-barrier Policy Search for SSP with unknown transition}
\label{alg:bandit-unknown}
	\textbf{Input:} Upper bound on expected hitting time $T$, horizon parameter $H_1$, and confidence level $\delta$
		
     \textbf{Parameters:} $H_2=\ceil{2D}$, $H=H_1+H_2$, $ C=\ceil{\log_2 (TK^4)}\ceil{\log_2 (T^2K^9)}, \beta=e^{\frac{1}{7\ln K}}, \eta=\sqrt{\frac{SA}{DTK}}, \gamma=280(\ln K)\rbr{2C\sqrt{\ln\rbr{\frac{CSA}{\delta}}}+1}^2\eta, \lambda=40\eta+2\gamma+33\sqrt{\frac{S^3A^2}{DTK}}$
	
	\textbf{Define:} $\regz_k(\qfeat) = \sum_{(s,a)\in\tilSA}\frac{1}{\eta_k(s,a)} \ln \frac{1}{\qfeat(s, a)}$, where $\qfeat(s, a) = \sum_{s'\in\calS\cup\{s_f\}}\sumh\qfeat(s, a, s', h)$
	
	\textbf{Define:} $\Omega_k = \{\qfeat=q + \lambda \h{q} : q\in \tilDelta(T, \calP_{i_k}), \;\; q(s,a) \geq \frac{1}{TK^4}, \;\forall (s,a)\in \tilSA\}$
	
	\textbf{Initialize:} $\qfeat_1= \argmin_{\qfeat\in\Omega_1} \regz_1(\qfeat)$.
	
	\textbf{Initialize:} for all $(s, a)\in\tilSA, \eta_1(s, a)=\eta, \rho_1(s, a)=2T$.
	
	\textbf{Initialize:} $\N_1(s, a)=\M_1(s, a, s')=0$ for all $(s, a, s')\in\SA\times(\calS\cup\{g\})$. An instance of Bernstein-SSP $\calB$.
	
	\For{$k=1,\ldots, K$}{
		Extract $\hatq_k$ from $\qfeat_k=\hatq_k + \lambda \h{\hatq_k}$, and let $\tilpi_k=\pi_{\hatq_k}$.
	
		Execute policy: $\tau_k=\RUN(\tilpi_k, \calB)$, receive $c_k\odot\Ind_k$.
	
		$\calP_{k+1}, \tilN_k =\TransEst(\N, \M, \delta, H_1, H_2, \tau_k)$.
		
		Construct cost estimator $\hatc_k \in \fR_{\geq 0}^{\tilSA}$ such that $\hatc_k(s,a) = \frac{\tilN_{k}(s, a)c_k(s, a)}{u_{k}(s, a)}$, where $u_k(s, a) = \max_{\hatP\in\calP_{k}}q_{\hatP, \tilpi_k}(s, a)$.
         
         Construct bias term $\hatb_k\in \fR_{\geq 0}^{\tilSA}$ such that $\hatb_k(s, a) = \frac{\sum_h h  u'_k(s, a, h)\hatc_k(s, a)}{u_k(s, a)}$ where $u'_k(s, a, h) = q_{P_k^{(s,a)}, \tilpi_k}(s, a, h)$ and $P_k^{(s,a)} = \argmax_{\hatP\in\calP_{k}}\sum_h h \cdot q_{\hatP, \tilpi_k}(s, a, h)$.

		Update \[
			\qfeat_{k+1} = \argmin_{\qfeat\in\Omega_{k+1}} \inner{\qfeat}{\hatc_k-\gamma\hatb_k} + D_{\psi_k}(\qfeat, \qfeat_k).
		\]
		
		\For{$\forall (s, a)\in\tilSA$}{
			\If{$\frac{1}{u_{k+1}(s, a)} > \rho_k(s, a)$}{
				$\rho_{k+1}(s, a)=\frac{2}{u_{k+1}(s, a)}, \eta_{k+1}(s, a)=\beta\cdot\eta_k(s, a)$.
			}
			\Else{
			$\rho_{k+1}(s, a)=\rho_{k}(s, a), \eta_{k+1}(s, a)=\eta_k(s, a)$.
			}
		}
	}
\end{algorithm}

%\subsection{\pfref{thm:bandit-unknown}}

\begin{lemma}\label{lem:optimistic-transition-bound}
	With probability at least $1-4\delta$, we have
	\begin{align*}
		\sumk\inner{u_k-q_k}{c_k} \leq 32\sqrt{S^3A^2\ln^2\rbr{\frac{HKS^2A^2}{\delta}}\rbr{\sumk\inner{q_k}{\h{c_k}} + \tilO{SAH^3\sqrt{K}}}} + \tilO{\Qumq}.
	\end{align*}
\end{lemma}
%\begin{remark}
%	The extra $\sqrt{SA}$ dependency is due to the fact that $u_k(s, a)$ depends on different transition function $P^{(s, a)}_k$ for different $(s, a)$. Hence we cannot group costs of different state-action pairs together to obtain $\sqrt{S^2A\sumk\inner{q_k}{\h{c_k}}}$. Instead, we compute the error separately for each $(s, a)$ and use Cauchy-Schwarz to combine them, which gives an extra $\sqrt{SA}$ factor.
%\end{remark}	
\begin{proof}
	Denote $c^{(s, a)}_k(s', a')=c_k(s', a')\Ind\{s'=s, a'=a\}$. Then by \pref{lem:transition-bias} (with $\delta$ there set to $\delta/|\SA|$) and a union bound over all $(s,a)$, we have with probability $1-4\delta$, for any $(s, a)\in\SA$:
	\begin{align*}
		\sumk(u_k(s, a) - q_k(s, a))c_k(s, a) &= \sumk\inner{u_k-q_k}{c_k^{(s, a)}}\\
		&\leq 32\sqrt{S^2A\ln^2\rbr{\frac{HKS^2A^2}{\delta}}\rbr{\sumk\inner{q_k}{\h{c_k^{(s, a)}}} + \tilO{H^3\sqrt{K}}}} + \tilO{\Qdtd}.
	\end{align*}
	Hence,
	\begin{align*}
		\sumk\inner{u_k-q_k}{c_k} &\leq 32\sumsa\sqrt{S^2A\ln^2\rbr{\frac{HKS^2A^2}{\delta}}\rbr{\sumk\inner{q_k}{\h{c_k^{(s, a)}}} + \tilO{H^3\sqrt{K}}}} + \tilO{\Qumq}\\
		&\leq 32\sqrt{S^3A^2\ln^2\rbr{\frac{HKS^2A^2}{\delta}}\sumsa\rbr{\sumk\inner{q_k}{\h{c_k^{(s, a)}}} + \tilO{H^3\sqrt{K}}}} + \tilO{\Qumq}\\
		&= 32\sqrt{S^3A^2\ln^2\rbr{\frac{HKS^2A^2}{\delta}}\rbr{\sumk\inner{q_k}{\h{c_k}} + \tilO{SAH^3\sqrt{K}}}} + \tilO{\Qumq},
	\end{align*}
	where in the second line we use Cauchy-Schwarz inequality.
\end{proof}

Below we present the proof of \pref{thm:bandit-unknown}.
It decomposes the regret into several terms, each of which is bounded by a lemma included after the proof.
\begin{proof}[\pfref{thm:bandit-unknown}]
By \pref{lem:loop-free}, with probability at least $1-\delta$,
\[
	R_K \leq \sumk\inner{\tilN_{k}-\tiloptq}{c_k} + \bigO{D^{3/2}S^2A(\ln\tfrac{1}{\delta})^2}.
\]
We define a slightly perturbed benchmark $q^\star = (1-\frac{1}{TK})\tiloptq + \frac{1}{TK}q_0 \in \tilDelta(T, \{P\})$ (note that $\tilDelta(T,\{P\})\subseteq\tilDelta(T, \calP_{i_k}),\forall k$ under the event of \pref{lem:transition-interval}) for some $q_0 \in \tilDelta(T, \{P\})$ with $q_0(s,a) \geq \frac{1}{K^3}$ for all $(s,a) \in \tilSA$, so that $\qfeat^\star=q^\star+\lambda\h{q^\star} \in \Omega_k,\forall k$.
Also define %$u'_k(s, a)$, such that $\hatb_k(s, a)=\frac{\sum_{h=1}^Hh\cdot u'_k(s, a, h)\hatc_k(s, a)}{u_k(s, a)}$,
$b_k \in \fR^{\tilSA}$ such that $b_k(s,a) = \frac{\sum_h h  u'_k(s, a, h)c_k(s, a)}{u_k(s, a)}$, which clearly satisfies $\E_k[\hatb_k] = b_k$.
We then decompose $\sumk\inner{\tilN_{k}-q^\star}{c_k}$ as
\begin{align*}
&\sumk\inner{\tilN_{k}-q^\star}{c_k} \\
&= \sumk\inner{u_{k}}{\hatc_k} - \sumk\inner{q^\star}{c_k}
\tag{$\inners{\tilN_{k}}{c_k} = \inner{u_k}{\hatc_k}$} \\
&= \sumk\inner{u_{k}-\hatq_k}{\hatc_k} + \sumk\inner{\hatq_k}{\hatc_k} - \inner{q^\star}{c_k}\\
&= \err_1 + \sumk\inner{\qfeat_{k}-\qfeat^\star}{\hatc_k} + \sumk\inner{\qfeat^\star}{\hatc_k-c_k} + \lambda\sumk\inner{\h{q^\star}}{c_k} - \lambda\sumk\inner{\h{\hatq_k}}{\hatc_k} \tag{define $\err_1=\sumk\inner{u_k-\hatq_k}{\hatc_k}$} \\
&= \err_1 + \sumk\inner{\qfeat_{k}-\qfeat^\star}{\hatc_k} + \sumk\inner{\qfeat^\star}{\hatc_k-c_k} + \tilO{\lambda DTK} - \lambda\sumk\inner{\h{\hatq_k}}{\hatc_k} \\
&= \sumk\inner{\qfeat_{k}-\qfeat^\star}{\hatc_k} + \tilO{\lambda DTK} + \err_1 + \bias_1 + \bias_2 - \lambda\sumk\inner{\h{\hatq_k}}{c_k}
\tag{define $\bias_1 = \sumk\inner{\qfeat^\star}{\hatc_k-c_k}$ and $\bias_2 = \lambda\sumk\inner{\h{\hatq_k}}{c_k - \hatc_k}$} \\
&= \reg_\qfeat + \tilO{\lambda DTK} + \err_1 + \bias_1 + \bias_2 +  \gamma\sumk \inner{\qfeat_{k}-\qfeat^\star}{\hatb_k} - \lambda\sumk\inner{\h{\hatq_k}}{c_k}
\tag{define $\reg_\qfeat = \sumk\inners{\qfeat_{k}-\qfeat^\star}{\hatc_k-\gamma\hatb_k}$} \\
&= \reg_\qfeat + \tilO{\lambda DTK} + \err_1 + \bias_1 + \bias_2 + \bias_3 + \bias_4 + \gamma\sumk \inner{\qfeat_{k}-\qfeat^\star}{b_k} - \lambda\sumk\inner{\h{\hatq_k}}{c_k}
\tag{define $\bias_3 = \gamma\sumk \inners{\qfeat_{k}}{\hatb_k - b_k}$ and  $\bias_4 = \gamma\sumk \inners{\qfeat^\star}{b_k - \hatb_k }$} \\
&\leq \reg_\qfeat + \tilO{\lambda DTK} + \err_1 + \bias_1 + \bias_2 + \bias_3 + \bias_4 \\
&\hspace{4cm}  +2\gamma \sumk \inner{\hatq_k}{b_k} - \gamma \sumk\inner{\qfeat^\star}{b_k} - \lambda\sumk\inner{\h{\hatq_k}}{c_k} \\
&\leq \reg_\qfeat + \tilO{\lambda DTK} + \err_1 + \bias_1 + \bias_2 + \bias_3 + \bias_4 \\ 
&\hspace{4cm} + (2\gamma-\lambda) \sumk \inner{\hatq_k}{\h{c_k}} + 2\gamma \sumk \inner{u'_k-\hatq_k}{\h{c_k}} - \gamma \sumk\inner{\qfeat^\star}{b_k}
\tag{$\inner{\hatq_k}{b_k}\leq\inner{u_k}{b_k} = \inner{u'_k}{\h{c_k}}$}\\
&= \reg_\qfeat + \tilO{\lambda DTK} + \err_1 + \err_2 + \bias_1 + \bias_2 + \bias_3 + \bias_4 \\ 
&\hspace{4cm} + (2\gamma-\lambda) \sumk \inner{\hatq_k}{\h{c_k}} - \gamma \sumk\inner{\qfeat^\star}{b_k}.
\tag{define $\err_2=2\gamma \sumk \inner{u'_k-\hatq_k}{\h{c_k}}$}
\end{align*}

The $\reg_\qfeat$ term can be upper bounded by the OMD analysis (see \pref{lem:OMD-increasing-eta}), the four bias terms $\bias_1, \bias_2, \bias_3$, and $\bias_4$ can be bounded using Azuma's or Freedman's inequality (see \pref{lem:BIAS_1_2} and \pref{lem:BIAS_3_4}), and $\err_1+\err_2$ can be bounded by \pref{lem:optimistic-transition-bound} (see \pref{lem:ERR_1_2}).
Combining everything, we obtain
\begin{align*}
	R_K &\leq \tilO{\frac{SA}{\eta}} - \frac{\inner{\qfeat^\star}{\rho_K}}{140\eta\ln K} + 40\eta\sumk\inner{q_k}{\h{c_k}} + \tilO{\lambda DTK}\\
	&\qquad+ 33\lambda'\rbr{\sumk\inner{q_k}{\h{c_k}} + \tilO{SAH^3\sqrt{K}} } + \tilO{\frac{S^3A^2}{\lambda'}}\\
	&\qquad+ 2C\sqrt{\ln\rbr{\frac{CSA}{\delta}}}\rbr{\frac{\inner{\qfeat^\star}{\rho_K}}{\eta'}+\eta'\inner{\qfeat^\star}{\sumk b_k}} + 2CH\ln\rbr{\frac{CSA}{\delta}}\inner{\qfeat^\star}{\rho_K}\\
	&\qquad+ \rbr{\frac{1}{\eta'}+1}\inner{\qfeat^\star}{\rho_K} + \eta'\inner{\qfeat^\star}{\sumk b_k} + (2\gamma-\lambda) \sumk \inner{\hatq_k}{\h{c_k}} - \gamma \sumk\inner{\qfeat^\star}{b_k} + \tilO{\Qumq}\\
	&= \tilO{\frac{SA}{\eta} + \lambda DTK + \lambda'SAH^3\sqrt{K} + \frac{S^3A^2}{\lambda'} + H^3S^3A^2} + \rbr{40\eta + 33\lambda' + 2\gamma - \lambda}\sumk\inner{\hatq_k}{\h{c_k}}\\
	&\qquad+ \rbr{ \frac{2C\sqrt{\ln\rbr{\frac{CSA}{\delta}}} + 1}{\eta'} + 2CH\ln\rbr{\frac{CSA}{\delta}} + 1 - \frac{1}{140\eta\ln K} }\inner{\qfeat^\star}{\rho_K}\\
	&\qquad+ \rbr{2C\sqrt{\ln\rbr{\frac{CSA}{\delta}}}\eta' + \eta' - \gamma}\inner{\qfeat^\star}{\sumk b_k} + (40\eta + 33\lambda')\sumk\inner{q_k-\hatq_k}{\h{c_k}}.
\end{align*}
Finally, taking $\eta'=\gamma/\rbr{2C\sqrt{\ln\rbr{\frac{CSA}{\delta}}}+1}, \lambda'=\sqrt{\frac{S^3A^2}{DTK}}$, and noticing $\frac{1}{\eta'}\geq 2CH\ln\rbr{\frac{CSA}{\delta}} + 1$, we have the coefficients multiplying $\sumk\inner{q_k}{\h{c_k}}, \inner{\qfeat^\star}{\rho_K}$, and $\inner{\qfeat^\star}{\sumk b_k}$ are all non-positive.
Moreover, by \pref{lem:transition-bias},  $\frac{1}{H}(\h{c_k})(s, a)\leq c_k(s, a)$, and $\sumk\inner{q_k}{\h{c_k}}\leq H^2K$:
\begin{align*}
	(40\eta + 33\lambda')\sumk\inner{q_k-\hatq_k}{\h{c_k}} &= (40\eta + 33\lambda')H\sumk\inner{q_k-\hatq_k}{\frac{\h{c_k}}{H}} \\
	&= \tilO{(40\eta + 33\lambda')H\sqrt{S^2A(H^2K+H^3\sqrt{K})} } = \tilO{H^3S^3A^2}.
\end{align*}
Thus, we arrive at $R_K = \tilO{\sqrt{S^3A^2DTK} + H^3S^3A^2}$.
\end{proof}

\begin{lemma}\label{lem:OMD-increasing-eta}
\pref{alg:bandit-unknown} ensures with probability at least $1-\delta$:
\[
\reg_\qfeat \leq \tilO{\frac{SA}{\eta}} - \frac{\inner{\qfeat^\star}{\rho_K}}{140\eta\ln K} + 40\eta\sumk\inner{q_k}{\h{c_k}} + \tilO{H^2\sqrt{SA}}.
\]
\end{lemma}
\begin{proof}
	Denote by $n(s, a)$ the number of times the learning rate for $(s, a)$ increases, such that $\eta_K(s, a)=\eta\beta^{n(s, a)}$,
	and by $k_1, \ldots, k_{n(s, a)}$ the episodes where $\eta_k(s, a)$ is increased, such that $\eta_{k_{t}+1}(s, a)=\beta\cdot\eta_{k_{t}}(s, a)$.
	Since $\rho_1(s,a) = 2T$ and
	\[
	  \rho_1(s, a)2^{n(s, a)-1}\leq \cdots \leq 
	\rho_{k_{n(s,a)}}(s, a) \leq \frac{1}{u_{k_{n(s,a)}+1}(s,a)} 
	\leq \frac{1}{\hatq_{k_{n(s,a)}+1}(s,a)} \leq TK^4,
	\]
	we have $n(s, a)\leq 1+ \log_2\frac{K^4}{2}\leq 7\log_2K$.
	Therefore, $\eta_K(s, a)\leq \eta e^{\frac{7\log_2K}{7\ln K}}\leq 5\eta$.
	
	Now, notice that $\sum_hu'_k(s, a, h)\leq u_k(s, a)$ by definition and $\gamma\leq\frac{1}{H}$ for large enough $K$ ($K \gtrsim S^3A^2H^2$).
	Thus,
	\[
	\gamma\hatb_k(s,a) \leq \frac{\gamma H \sum_h u'_k(s,a,h)\hatc_k(s,a)}{u_k(s,a)}
	\leq \gamma H \hatc_k(s,a) \leq  \hatc_k(s,a).
	\]
	This means that the cost $\hatc_k - \gamma\hatb_k$ we feed to OMD is always non-negative, and thus by the same argument of~\citep[Lemma~12]{agarwal2017corralling}, we have
	\begin{align*}
		&\reg_\qfeat = \sumk\inner{\qfeat_k-\qfeat^\star}{\hatc_k-\gamma\hatb_k}\\
		&\leq \sumk D_{\psi_k}(\qfeat^\star, \qfeat_k) - D_{\psi_k}(\qfeat^\star, \qfeat_{k+1}) + \sumk\sumsa\eta_k(s, a)\qfeat^2_k(s, a)(\hatc_k(s, a) - \gamma\hatb_k(s, a))^2\\
		&\leq D_{\psi_1}(\qfeat^\star, \qfeat_1) + \sum_{k=1}^{K-1}\rbr{D_{\psi_{k+1}}(\qfeat^\star, \qfeat_{k+1}) - D_{\psi_k}(\qfeat^\star, \qfeat_{k+1})} + 5\eta\sumk\sumsa \qfeat_k^2(s, a)\hatc_k^2(s, a)\\
		&\leq D_{\psi_1}(\qfeat^\star, \qfeat_1) + \sum_{k=1}^{K-1}\rbr{D_{\psi_{k+1}}(\qfeat^\star, \qfeat_{k+1}) - D_{\psi_k}(\qfeat^\star, \qfeat_{k+1})} + 20\eta\sumk\sumsa \hatq_k^2(s, a)\hatc_k^2(s, a)\\
		&\leq D_{\psi_1}(\qfeat^\star, \qfeat_1) + \sum_{k=1}^{K-1}\rbr{D_{\psi_{k+1}}(\qfeat^\star, \qfeat_{k+1}) - D_{\psi_k}(\qfeat^\star, \qfeat_{k+1})} + 20\eta\sumk\sumsa\tilN_k^2(s, a)c_k^2(s, a). \tag{$u_k(s, a)\geq \hatq_k(s, a)$}
	\end{align*}
	For the first term, since $\qfeat_1$ minimizes $\psi_1$ and thus $\inner{\nabla\regz_1(\qfeat_1)}{\qfeat^\star-\qfeat_1}\geq 0$, we have
	\begin{align*}
		D_{\psi_1}(\qfeat^\star, \qfeat_1) \leq \psi_1(\qfeat^\star) - \psi_1(\qfeat_1) = \frac{1}{\eta}\sumsa\ln\frac{\qfeat_1(s, a)}{\qfeat^\star(s, a)} 
		\leq \frac{1}{\eta}\sumsa\ln\frac{2H}{q^\star(s,a)} = \tilO{\frac{SA}{\eta}}.
	\end{align*}
	For the second term, we define $\aux(y) = y-1-\ln y$ and proceed similarly to~\citep{agarwal2017corralling}:
	\begin{align*}
		&\sum_{k=1}^{K-1} D_{\psi_{k+1}}(\qfeat^\star, \qfeat_{k+1}) - D_{\psi_k}(\qfeat^\star, \qfeat_{k+1})\\
		&= \sum_{k=1}^{K-1}\sumsa\rbr{\frac{1}{\eta_{k+1}(s, a)} - \frac{1}{\eta_k(s, a)}}\aux\rbr{\frac{\qfeat^\star(s, a)}{\qfeat_{k+1}(s, a)}}\\
		&\leq\sumsa\frac{1-\beta}{\eta\beta^{n(s, a)}} \aux\rbr{\frac{\qfeat^\star(s, a)}{\qfeat_{k_{n(s, a)}+1}(s, a)}}\\
		&=\sumsa\frac{1-\beta}{\eta\beta^{n(s, a)}}\rbr{\frac{\qfeat^\star(s, a)}{\qfeat_{k_{n(s, a)}+1}(s, a)} - 1 - \ln\frac{\qfeat^\star(s, a)}{\qfeat_{k_{n(s, a)}+1}(s, a)}}\\
		&\leq -\frac{1}{35\eta\ln K}\sumsa\rbr{ \qfeat^\star(s, a)\frac{\rho_K(s, a)}{4} - 1 - \ln\frac{\qfeat^\star(s, a)}{\qfeat_{k_{n(s, a)}+1}(s, a)} }\\
		&\leq \frac{SA(1 + 6\ln K)}{35\eta\ln K} - \frac{\inner{\qfeat^\star}{\rho_K}}{140\eta\ln K} = \tilO{\frac{SA}{\eta}}- \frac{\inner{\qfeat^\star}{\rho_K}}{140\eta\ln K},
	\end{align*}
	where in the last two lines we use the facts $1-\beta \leq -\frac{1}{7\ln K}, \beta^{n(s, a)} \leq 5$, $\rho_K(s, a)=\frac{2}{u_{k_{n(s, a)}+1}(s, a)}$, and $\ln\frac{\qfeat^\star(s, a)}{\qfeat_{k_{n(s, a)}+1}(s, a)} \leq \ln(HTK^4) \leq 6\ln K$.
	
	Finally, for the third term, 
	since $\sumsa \tilN_k^2(s, a)c_k^2(s, a) \leq \left(\sumsa \tilN_k(s, a)\right)^2 \leq H^2$, we apply Azuma's inequality (\pref{lem:azuma}) and obtain, with probability at least $1-\delta$:
	\begin{align*}
		\eta\sumk\sumsa \tilN_k^2(s, a)c_k^2(s, a) &\leq \eta\sumk\E_k\sbr{\sumsa \tilN_k^2(s,a )c_k^2(s, a)} + \tilO{\eta H^2\sqrt{K}} \\
		&\leq  \eta\sumk\E_k\sbr{\inner{\tilN_k}{c_k}^2} + \tilO{H^2\sqrt{SA}} \tag{$\sum_i a_i^2\leq (\sum_i a_i)^2$ for $a_i>0$ and $\eta=\sqrt{\frac{SA}{DTK}}$}\\
		&\leq 2\eta\sumk\inner{q_k}{\h{c_k}} + \tilO{H^2\sqrt{SA}} \tag{\pref{lem:deviation_loop_free}}.
	\end{align*}
	Combining everything shows 
	\begin{align*}
		\reg_\qfeat \leq \tilO{\frac{SA}{\eta}} - \frac{\inner{\qfeat^\star}{\rho_K}}{140\eta\ln K} + 40\eta\sumk\inner{q_k}{\h{c_k}} + \tilO{H^2\sqrt{SA}}.
	\end{align*}
	finishing the proof.
\end{proof}

\begin{lemma}\label{lem:BIAS_1_2}
For any $\eta' > 0$, 
with probability at least $1-\delta$, 
\[
\bias_1 \leq 2C\sqrt{\ln\rbr{\frac{CSA}{\delta}}}\rbr{\frac{\inner{\qfeat^\star}{\rho_K}}{\eta'}+\eta'\inner{\qfeat^\star}{\sumk b_k}} + 2CH\ln\rbr{\frac{CSA}{\delta}}\inner{\qfeat^\star}{\rho_K}.
\]
Also, with probability at least $1-5\delta$, 
$
\bias_2 = \tilO{ S^3A^2H^2 }.
$
\end{lemma}
\begin{proof}
	Define $X_k(s, a) = \hatc_k(s, a) - \E_k[\hatc_k(s, a)]$.
	Note that $X_k(s, a)\leq \frac{H}{u_k(s,a)} \leq HTK^4$, and
	\begin{align*}
		\sumk\E_k \sbr{X_k^2(s, a)} &\leq \sumk \frac{\E_k\sbr{\tilN_k^2(s, a)c^2_k(s, a)}}{u^2_k(s, a)}  \\
		&\leq \rho_k(s,a) \sumk \frac{\E_k\sbr{\tilN_k^2(s, a)c^2_k(s, a)}}{u_k(s, a)}  \\
		&\leq 2\rho_k(s,a) \sumk b_k(s,a). \tag{\pref{lem:simpler_form_N_times_c}}
	\end{align*}
	Therefore, by applying a strengthened Freedman's inequality (\pref{lem:extended-freedman}) with $b = HTK^4$, $B_k=H\rho_k(s, a), \max_kB_k=H\rho_K(s, a)$, and $V = 2\rho_k(s,a) \sumk b_k(s,a)$, we have with probability at least $1-\delta/(SA)$, 
	\begin{align*}
		&\sumk \hatc_k(s, a) - \E_k[\hatc_k(s, a)] \\ 
		&\leq C\rbr{ 4\sqrt{\rho_K(s, a)\sumk b_k(s, a)\ln\rbr{\frac{CSA}{\delta}}} + 2H\rho_K(s, a)\ln\rbr{\frac{CSA}{\delta}} } \\
		&\leq 2C\sqrt{\ln\rbr{\frac{CSA}{\delta}}}\rbr{\frac{\rho_K(s,a)}{\eta'}+\eta'\sumk b_k(s,a)} + 2CH\rho_K(s, a)\ln\rbr{\frac{CSA}{\delta}}, 
	\end{align*}
	where the last step is by AM-GM inequality.
	Further using a union bound shows that the above holds for all $(s, a)\in\tilSA$ with probability at least $1-\delta$.
	Thus, by $\E_k[\hatc_k(s, a)]\leq c_k(s, a)$,
	\begin{align*}
		\bias_1 &= \sumk\inner{\qfeat^\star}{\hatc_k-c_k} \leq \sumk\inner{\qfeat^\star}{\hatc_k-\E_k[\hatc_k]} \\
		&\leq 2C\sqrt{\ln\rbr{\frac{CSA}{\delta}}}\rbr{\frac{\inner{\qfeat^\star}{\rho_K}}{\eta'}+\eta'\inner{\qfeat^\star}{\sumk b_k}} + 2CH\ln\rbr{\frac{CSA}{\delta}}\inner{\qfeat^\star}{\rho_K}.
	\end{align*}
	To bound $\bias_2$, simply note that $\left|\inner{\h{\hatq_k}}{\E_k[\hatc_k]-\hatc_k}\right| \leq 2H^2$ and apply Azuma's inequality (\pref{lem:azuma}): with probability $1-5\delta$,
	\begin{align*}
		\bias_2 &= \lambda\sumk\inner{\h{\hatq_k}}{c_k - \E_k[\hatc_k]} + \lambda\sumk\inner{\h{\hatq_k}}{\E_k[\hatc_k] - \hatc_k} \\
		&\leq \lambda H\sumk\inner{u_k-q_k}{c_k} + \tilO{\lambda H^2\sqrt{K}} \\
		&= \tilO{\lambda H\sqrt{S^3A^2\sumk\inner{q_k}{\h{c_k}}} + H^2\sqrt{S^3A^2} } = \tilO{S^3A^2H^2},\tag{$\lambda=\tilO{\sqrt{\frac{S^3A^2}{DTK}}}, \sumk\inner{q_k}{\h{c_k}}=\bigO{H^2K}$ and \pref{lem:optimistic-transition-bound}}
	\end{align*}
	where in the second line we use:
	\begin{align*}
		\sumk\inner{\h{\hatq_k}}{c_k-\E_k[\hatc_k]} &= \sumk\sum_{(s, a), h}h\cdot\hatq_k(s, a, h)c_k(s, a)\rbr{1 - \frac{q_k(s, a)}{u_k(s, a)} }\\
		&\leq H\sumk\sum_{(s, a), h}\hatq_k(s, a, h)c_k(s, a)\frac{u_k(s, a)-q_k(s, a)}{u_k(s, a)}\\
		&\leq H\sumk\inner{u_k-q_k}{c_k}. \tag{$u_k(s, a)\geq \max\{q_k(s, a), \hatq_k(s, a)\}$}
	\end{align*}
	%This completes the proof.
\end{proof}

\begin{lemma}\label{lem:BIAS_3_4}
With probability at least $1-\delta$, we have $\bias_3 = \tilO{H^2(SA)^{3/2}}$.
Also, for any $\eta' > 0$, 
with probability at least $1-\delta$, we have
\[
\bias_4 \leq \rbr{\frac{1}{\eta'}+1}\inner{\qfeat^\star}{\rho_K} + \eta'\inner{\qfeat^\star}{\sumk b_k} + \tilO{1}.
\]
\end{lemma}
\begin{proof}
To bound $\bias_3$, simply note that $\E_k[\hatb_k(s, a)]\leq b_k(s, a)$,
\begin{align*}
&\left|\inner{\phi_k}{\hatb_k-\E_k[\hatb_k]}\right| \leq \abr{\inner{\phi_k}{\hatb_k}} + \abr{\inner{\phi_k}{\E_k\hatb_k}}\\
&\leq 2\abr{\sumsa\sum_hh\cdot u'_k(s, a, h)\hatc_k(s, a)} + 2\abr{\sumsa\sum_hh\cdot u'_k(s, a, h)\E_k[\hatc_k(s, a)]} \tag{$\phi_k(s, a)\leq 2\hatq_k(s, a)\leq 2u_k(s, a)$} \\
&\leq 2H\abr{\sumsa\tilN_k(s, a)c_k(s, a)} + 2H\abr{\sumsa\sum_hu'_k(s, a, h)c_k(s, a)} \leq 4SAH^2, \tag{$\sum_hu'_k(s, a, h)\leq u_k(s, a), q_k(s, a)\leq u_k(s, a), \sumsa \tilN_k(s, a)\leq H$, and $\sum_hu'_k(s, a, h)\leq H$}  
\end{align*}
$\gamma=\tilO{\sqrt{\frac{SA}{DTK}}}$, and apply Azuma's inequality (\pref{lem:azuma}): with probability at least $1-\delta$,
	\begin{align*}
		\bias_3 = \gamma\sumk \inners{\qfeat_{k}}{\hatb_k - b_k} \leq \gamma\sumk\inner{\phi_k}{\hatb_k - \E_k[\hatb_k]} = \tilO{\gamma SAH^2\sqrt{K}} = \tilO{H^2(SA)^{3/2}}.
	\end{align*}
	To bound $\bias_4 = \gamma\sumk\inner{\qfeat^\star}{b_k - \E_k[\hatb_k]} + \gamma\sumk\inner{\qfeat^\star}{\E_k[\hatb_k] - \hatb_k}$, first note that
	\begin{align*}
		\gamma\sumk\inner{\qfeat^\star}{b_k - \E_k[\hatb_k]} &\leq \gamma \sumk\sumsa\inner{\qfeat^\star}{b_k} = \gamma\sumk\sumsa \qfeat^\star(s, a)\frac{\sum_{h=1}^Hh\cdot u'_k(s, a, h)}{u_k(s, a)}c_k(s, a)\\
		&\leq \gamma H\sumk\sumsa\qfeat^\star(s, a) \leq \sumk\inner{\qfeat^\star}{\rho_K}. \tag{$\gamma H\leq 1$ and $\rho_K(s, a)\geq\rho_1(s, a)\geq 1$}
	\end{align*}
%	\begin{align*}
%		\gamma\sumk\inner{\qfeat^\star}{b_k - \E_k[\hatb_k]} &= \gamma\sumk\sumsa\qfeat^\star(s, a)\frac{\sum_hh\cdot u'_k(s, a, h)}{u_k(s, a)}\rbr{c_k(s, a) - \E_k[\hatc_k(s, a)]}\\
%		&\leq \gamma H\sumsa\qfeat^\star(s, a)\sumk\frac{u_k(s, a)-q_k(s, a)}{u_k(s, a)}c_k(s, a) \\
%		&\leq \gamma H\sumsa \qfeat^\star(s, a)\rho_K(s, a)\sqrt{S^2AH^2K} \leq H^2\sqrt{S^3A^2}\inner{\qfeat^\star}{\rho_K}.
%	\end{align*}
	Then, we note that $\E_k[\hatb_k(s, a)] - \hatb_k(s,a) \leq \E_k[\hatb_k(s,a)] \leq H$, and
	\begin{align*}
		\E_k\sbr{\rbr{\E_k[\hatb_k(s, a)] - \hatb_k(s, a)}^2} &\leq \E_k[\hatb_k^2(s, a)] \leq \frac{\E_k\sbr{(\sum_h h\cdot u'_k(s, a, h)\hatc_k(s, a))^2}}{u_k^2(s, a)}\\
		&\leq \frac{H^2\E_k\sbr{\tilN_k^2(s, a)c_k^2(s, a)}}{u^2_k(s, a)} \tag{$h\leq H$ and $\sum_hu'_k(s, a, h)\leq u_k(s, a)$}\\
		&\leq H^2\rho_K(s,a)\frac{\E_k\sbr{\tilN_k^2(s, a)c_k^2(s, a)}}{u_k(s, a)} \\
		& \leq 2H^2\rho_K(s, a)b_k(s, a) \tag{\pref{lem:simpler_form_N_times_c}}.
	\end{align*}
	Hence, applying a strengthened Freedman's inequality (\pref{lem:extended-freedman}) with $b = B_k = H$, $V = 2H^2\rho_K(s, a)\sumk b_k(s, a)$, and $C' = \lceil \log_2 H \rceil \lceil \log_2 (H^2K) \rceil$, we have
 with probability at least $1-\delta/(SA)$,
	\begin{align*}
		&\sumk b_k(s, a) - \hatb_k(s, a)   \\
		&\leq  4C'H\sqrt{\ln\rbr{\frac{C'SA}{\delta}}}\sqrt{\rho_K(s,a) \sumk b_k(s,a)} + 2C'H\ln\rbr{\frac{C'SA}{\delta}} \\
		&=  2C'H\sqrt{\ln\rbr{\frac{C'SA}{\delta}}}\rbr{\frac{\rho_K(s,a)}{\eta'} + \eta' \sumk b_k(s,a)} + 2C'H\ln\rbr{\frac{C'SA}{\delta}},
	\end{align*}
	where the last step is by AM-GM inequality.
	Finally, applying a union bound shows that the above holds for all $(s,a)\in\tilSA$ with probability at least $1-\delta$ and thus
	\begin{align*}
		\bias_4 \leq \gamma\sumk\inner{\qfeat^\star}{b_k - \hatb_k} \leq \rbr{\frac{1}{\eta'}+1}\inner{\qfeat^\star}{\rho_K} +\eta'\inner{\qfeat^\star}{\sumk b_k} + \tilO{1},
	\end{align*}
	where we bound $2\gamma C'H\sqrt{\ln\rbr{\frac{C'SA}{\delta}}}$ by $\tilO{1}$ since $\gamma H\leq 1$ when $K$ is large enough ($K \gtrsim S^3A^2H^2$).
\end{proof}

\begin{lemma}\label{lem:ERR_1_2}
	For any $\lambda' \in (0, \frac{2}{H}]$, with probability at least $1-9\delta$ we have 
	\begin{align*}
		\err_1 \leq 33\lambda'\rbr{\sumk\inner{q_k}{\h{c_k}} + \tilO{SAH^3\sqrt{K}} } + \frac{34S^3A^2\ln^2\rbr{\frac{HKS^2A^2}{\delta}}}{\lambda'} + \tilO{\Qumq}.
	\end{align*}
	Also, with probability at least $1-8\delta$, we have $\err_2=\tilO{ H^3S^3A^2 }$,
\end{lemma}
\begin{proof}
	We write $\err_1$ as
	\begin{align*}
		\err_1 &= \sumk\inner{u_k-\hatq_k}{\hatc_k - \E_k[\hatc_k]} + \sumk\inner{u_k-\hatq_k}{\E_k[\hatc_k]}.
	\end{align*}
	For the first term, note that $\inner{u_k-\hatq_k}{\hatc_k}\leq\sumsa\tilN_k(s, a)\leq H$, and
	\begin{align*}
		\E_k\sbr{\rbr{\sumsa (u_k(s, a)-\hatq_k(s, a))\frac{\tilN_k(s, a)}{u_k(s, a)}c_k(s, a) }^2} \leq \E_k\sbr{\rbr{\sumsa \tilN_k(s, a)c_k(s, a)}^2 } \leq 2\inner{q_k}{\h{c_k}}.
	\end{align*}
	Hence, by Freedman's inequality (\pref{lem:freedman}), for any $0<\lambda'\leq\frac{2}{H}$, with probability at least $1-\delta$ we have:
	\begin{align*}
		\sumk\inner{u_k-\hatq_k}{\hatc_k - \E_k[\hatc_k]} = \lambda'\sumk\inner{q_k}{\h{c_k}} + \frac{2\ln(1/\delta)}{\lambda'}.
	\end{align*}
	For the second term, with probability at least $1-8\delta$:
	\begin{align*}
		\sumk\inner{u_k-\hatq_k}{\E_k[\hatc_k]} &\leq \sumk\inner{u_k-\hatq_k}{c_k} = \sumk\inner{u_k-q_k}{c_k} + \inner{q_k-\hatq_k}{c_k}\\
		&\leq 64\sqrt{S^3A^2\ln^2\rbr{\frac{HKS^2A^2}{\delta}}\rbr{\sumk\inner{q_k}{\h{c_k}} + \tilO{SAH^3\sqrt{K}}}} + \tilO{\Qumq}. \tag{\pref{lem:optimistic-transition-bound} and \pref{lem:transition-bias}}\\
		&\leq 32\lambda'\rbr{\sumk\inner{q_k}{\h{c_k}} + \tilO{SAH^3\sqrt{K}} } + \frac{32S^3A^2\ln^2\rbr{\frac{HKS^2A^2}{\delta}}}{\lambda'} + \tilO{\Qumq}. \tag{AM-GM inequality}
	\end{align*}
	Putting everything together, we have:
	\begin{align*}
		\err_1 \leq 33\lambda'\rbr{\sumk\inner{q_k}{\h{c_k}} + \tilO{SAH^3\sqrt{K}} } + \frac{34S^3A^2\ln^2\rbr{\frac{HKS^2A^2}{\delta}}}{\lambda'} + \tilO{\Qumq}.
	\end{align*}
	For $\err_2$, use the bound for $\sumk\inner{u_k-\hatq_k}{c_k}$ from above, we have with probability at least $1-8\delta$:
	\begin{align*}
		\err_2 &= 2\gamma \sumk \inner{u'_k-\hatq_k}{\h{c_k}} \leq 2\gamma H\sumk\inner{u_k-\hatq_k}{c_k}\\ 
		&= \tilO{\gamma H\sqrt{S^3A^2\rbr{\sumk\inner{q_k}{\h{c_k}} + \tilO{SAH^3\sqrt{K}} }} + \gamma H^4S^3A^2 } \tag{\pref{lem:optimistic-transition-bound} and \pref{lem:transition-bias}}\\
		&= \tilO{\gamma H\sqrt{S^3A^2(H^2K + SAH^3\sqrt{K})}} = \tilO{ H^3S^3A^2 }. \tag{$\gamma=\tilO{\sqrt{\frac{SA}{DTK}}}$ and $\gamma H\leq 1$}
	\end{align*}
%	Therefore,
%	\begin{align*}
%		\err_1 + \err_2 \leq 33\lambda'\rbr{\sumk\inner{q_k}{\h{c_k}} + \tilO{SAH^3\sqrt{K}} } + \frac{32S^3A^2\ln^2\rbr{\frac{HKS^2A^2}{\delta}}}{\lambda'} + \tilO{H^4S^3A^2}.
%	\end{align*}
\end{proof}

\begin{lemma}
	\label{lem:simpler_form_N_times_c}
	For any episode $k$ and $(s, a)\in\tilSA$, we have $\E_k \sbr{\tilN_k(s, a)^2c_k(s, a)^2} \leq 2u_k(s, a)b_k(s, a)$.
\end{lemma}
\begin{proof}
We use the inequality $(\sum_{i=1}^I a_i)^2\leq 2\sum_{i}a_i(\sum_{i'= i}^I a_{i'})$:
	\begin{align*}
		\E_k\sbr{ \tilN_k(s, a)^2c_k(s, a)^2} &\leq \E_k\sbr{\rbr{ \sum_{h=1}^H \tilN_k(s, a, h)}^2c_k(s, a) }\\
		&\leq 2\E_k\sbr{\rbr{\sum_{h=1}^H \tilN_k(s, a, h)}\rbr{\sum_{h'\geq h}^H \tilN_k(s, a, h')c_k(s, a)}}\\
		&\leq 2\E_k\sbr{\sum_{h=1}^H\sum_{h'\geq h}^H \tilN_k(s, a, h')c_k(s, a)} \tag{$\tilN_k(s, a, h)\leq 1$}\\
		&= 2 \sum_{h=1}^H\sum_{h'\geq h}^H q_k(s, a, h')c_k(s, a) = 2\sum_{h=1}^Hh \cdot q_k(s, a, h)c_k(s, a)\\
		&\leq 2\sum_{h=1}^H h\cdot u'_k(s, a, h)c_k(s, a) = 2u_k(s, a)b_k(s, a),
	\end{align*}
	where the last step is by the definition of $b_k$.
\end{proof}

\section{Omitted details for \pref{sec:iid}}
\label{app:iid}
% !TEX root = main.tex

In this section, we provide all omitted proofs for \pref{sec:iid}.
We first introduce modifications to the loop-free SSP which allows the learner to switch to Bernstein-SSP at any state and any time.
We add a new state $s'_f$ in $\tilM$ (that is, add $s'_f$ into $\lfs$) and add $a_f$ to $\calA_{(s, h)}$ for all $(s, h)\in (\calS\cup\{s'_f\})\times[H_1]$, such that $\tilP((s'_f, h+1)|(s, h), a_f)=1,\forall s\in\calS\cup\{s'_f\}, h\in[H_1]$ and $\tilc_k(s, a_f, h)=0, \forall s\in\calS\cup\{s'_f\}, h\in[H]$.
Also add $(s, a_f)$ to $\tilSA$ for all $s \in \calS$.
Executing $a_f$ in state $(s, h)$ leads the agent into state $(s'_f, h+1)$ and stay at $s'_f$ until it reaches $(s_f, H_1+1)$.
This is equivalent to directly transit to $(s_f, H_1+1)$ when taking action $a'_f$ at $(s, h)$, but we pad states $(s'_f, h+1),\ldots,(s'_f,H_1)$ to preserve the layer structure (which simplifies notation).
Therefore, executing $a_f$ at $(s, h)$ with $h\leq H_1$ corresponds to directly switching to Bernstein-SSP at current state in $M$ (see \pref{alg:unknown-run} and \pref{fn:break}).
Eventually, this modification only amounts to the following slightly different definition of $\tilDelta(T, \calP)$ for the algorithm (so that taking action $a_f$ does not count towards actual time steps in $M$):
\begin{align*}
	\tilDelta(T, \calP) &= \Bigg\{ q \in [0,1]^{\tilSA\times\lfs\times [H]}: \sumh\sum_{(s, a)\in\tilSA\setminus((\calS\cup\{s'_f\})\times\{a_f\})} q(s, a, h)\leq T, \notag\\
&\sum_{a\in\tilA_{(s,h)}} q(s, a, h) - \sumtilsaf[s', a'] q(s', a', s, h-1) = \Ind\{(s,h)=\tils_0\}, \forall h > 1;\\
&\sum_{a\in \tilA_{(s,1)}} q(s, a, 1) = \Ind\{s=s_0\}, \;\forall s\in\calX;  \; P_q \in \calP \Bigg\}.
\end{align*}
With this new definition, it is not hard to obtain something similar to \pref{lem:loop-free} for the pseudo-regret.
\begin{lemma}\label{lem:loop-free-iid}
Suppose $H_1 \geq 8\Tmax\ln K, H_2 = \ceil{2D}$ and $K\geq D$, and $\tilpi_1,\ldots,\tilpi_k$ are policies for $\tilM$.
Then with probability at least $1-\delta$, the regret of executing $\sigma(\tilpi_1), \ldots, \sigma(\tilpi_K)$ in $M$ satisfies,
	\begin{align*}
	\tilR_K \leq \sumk\inner{\tilN_{k}}{c_k} - \inner{\tiloptq}{c} + \tilO{D^{3/2}S^2A(\ln\tfrac{1}{\delta})^2}.
	\end{align*}
	where $\optpi\in\argmin_{\pi\in\propers}J^{\pi, c}(s_0)$ and $\tiloptpi(s, h)=\optpi(s)$.
\end{lemma}
\begin{proof}
By similar arguments in the proof of \pref{lem:loop-free}, we have with probability at least $1-\delta$:
\begin{align*}
	\sumk\inner{N_k}{c_k} \leq \sumk\inner{\tilN_k}{c_k} + \tilO{D^{3/2}S^2A(\ln\tfrac{1}{\delta})^2}.
\end{align*}
and (the same cost for all $K$ episodes)
\begin{align*}
J^{\tiloptpi, c}(\tils_0)  \leq J^{\optpi, c}(s_0) + \frac{2H_2}{K^2} =  
J^{\optpi, c}(s_0) + \tilO{\frac{1}{K}}.
\end{align*}
	Putting everything together, we get:
	\begin{align*}
		\tilR_K = \sum_{k=1}^K\inner{N_k}{c_k} - J^{\optpi, c}(s_0) \leq \sumk\inner{\tilN_k}{c_k} - \inner{\tiloptq}{c} + \tilO{D^{3/2}S^2A(\ln\tfrac{1}{\delta})^2}.
	\end{align*}
\end{proof}

With these modifications, the state value w.r.t optimistic transition/cost is of $\bigO{D}$. %it is also not hard to see from the definition that the cost starting from any state is now bounded by $D$:

\begin{lemma}
	\label{lem:fast-action}
	Assume $T\geq\T$.
	For $\hatq_k$ obtained from \pref{eq:evi}, we have $\inner{\hatq_k}{\hatc_k}\leq\inner{\tiloptq}{\hatc_k}$ and $J^{P_{\hatq_k}, \tilpi_k, \hatc_k}(s, h)\leq H_2$ for all $(s,h)$ with $\hatq_k(s, h)>0$.
\end{lemma}
\begin{proof}
	The first inequality is by the definition of $\hatq_k=\argmin_{q\in\tilDelta(T, \calP_k)}\inner{q}{\hatc_k}$ and $\tiloptq\in\tilDelta(T, \calP_k)$.
	%We only need to prove the second inequality for $(s, h)$ with $\hatq_k(s, h)>0$, since otherwise we can set $J^{P_{\hatq_k}, \tilpi_k, \hatc_k}(s, h)$ arbitrarily.
	The second inequality can be proven by contradiction: suppose $J^{P_{\hatq_k}, \tilpi_k, \hatc_k}(s, h)>H_2$.
	Define $\pi'(a'_f|s, h)=1$ and $\pi'=\tilpi_k$ for all other entries.
	Note that $q_{P_{\hatq_k}, \pi'}\in\tilDelta(T, \calP_{i_k})$, since the expected hitting time of $\pi'$ should be no more than that of $\tilpi_k$.
	Moreover, by the structure of $\tilM$, we have $J^{P_{\hatq_k}, \pi', \hatc_k}(s, h)=H_2 < J^{P_{\hatq_k}, \tilpi_k, \hatc_k}(s, h)$.
	Since, $\pi'$ differs from $\tilpi_k$ in only one entry, we have $J^{P_{\hatq_k}, \pi', \hatc_k}(s_0, 1) < J^{P_{\hatq_k}, \tilpi_k, \hatc_k}(s_0, 1)$, a contradiction to the definition of $\hatq_k$.
\end{proof}

We next present a lemma similar to \pref{lem:transition-bound} but for cost estimation.
\begin{lemma}
	\label{lem:cost-bound}
	With probability at least $1-\delta$, for any $(s, a)\in\SA, k\in[K]$, $0 \leq c(s, a)-\hatc_k(s, a)\leq 8\sqrt{\A^c_k(s, a)c(s, a)} + 34\A^c_k(s, a)$.
\end{lemma}
\begin{proof}
	Applying \pref{lem:bernstein} with $X_k=c_k(s, a)$ for each $(s, a)\in\SA$ and then by a union bound over all $(s, a)\in\SA$, we have with probability at least $1-\delta$, for all $k\in[K]$:
	\begin{align*}
		|\bar{c}_k(s, a) - c(s, a)| \leq 2\sqrt{\A^c_k(s, a)\bar{c}_k(s, a)} + 7\A^c_k(s, a).
	\end{align*}
	Hence, $c(s, a) \geq \hatc_k(s, a)$ by the definition of $\hatc_k$.
	Applying $x\leq a\sqrt{x}+b \implies x\leq (a+\sqrt{b})^2$ with $x=\bar{c}_k(s, a)$ to the inequality above (ignoring the absolute value operator), we obtain
	\begin{align*}
		\bar{c}_k(s, a) \leq c(s, a) + 4\sqrt{\A^c_k(s, a)c(s, a)} + 23\A^c_k(s, a) \leq 3c(s, a) + 25\A^c_k(s, a),
	\end{align*}
	where we apply $\sqrt{ab}\leq\frac{a+b}{2},\forall a, b>0$ for the last inequality. Therefore, $2\sqrt{\A^c_k(s, a)\bar{c}_k(s, a)}+7\A^c_k(s, a)\leq 4\sqrt{\A^c_k(s, a)c(s, a)}+17\A^c_k(s, a)$, and
	\begin{align*}
		 c(s, a) - \hatc_k(s, a) &= c(s, a) - \bar{c}_k(s, a) + \bar{c}_k(s, a) - \hatc_k(s, a) \leq 8\sqrt{\A^c_k(s, a)c(s, a)} + 34\A^c_k(s, a).
	\end{align*}
\end{proof}

We are now ready to prove \pref{thm:iid}.
We first introduce some notations used in the proof below.
Define $P_k=P_{\hatq_k}$, and $x_k(s, a)=\E_k[\Ind_k(s, a)]$ as the probability that $(s, a)$ is ever visited in episode $k$.
Also denote by $q_{k, (s, a, h)}$ (or $\hatq_{k, (s, a, h)}$) the occupancy measure of $\tilpi_k$ with transition $\tilP$ (or $P_k$) and initial state-action pair $((s, h), a)$.
Without loss of generality, we assume $c_k$ is sampled right before the beginning of episode $k$ instead of before learning starts.
\begin{proof}[\pfref{thm:iid}]
	We condition on the event of \pref{lem:transition-interval} and \pref{lem:cost-bound}, which happens with probability at least $1-2\delta$.
	We first decompose the regret in $\tilM$ using the fact $\inner{\tiloptq}{c}\geq\inner{\tiloptq}{\hatc_k}\geq\inner{\hatq_k}{\hatc_k}$ derived from \pref{lem:cost-bound} and \pref{lem:fast-action}:
	\begin{align}
		\sumk\inner{\tilN_k}{c_k} - \inner{\tiloptq}{c} &\leq \sumk\inner{\tilN_k}{c_k} - \inner{\hatq_k}{\hatc_k} \notag \\
		&\leq \sumk\inner{\tilN_k}{c_k} - \inner{q_k}{c} + \sumk\inner{q_k}{c-\hatc_k} + \sumk\inner{q_k - \hatq_k}{\hatc_k}. \label{eq:reg_decomposition}
	\end{align}
	The first term in \pref{eq:reg_decomposition} is a martingale difference sequence where the randomness in episode $k$ is w.r.t the learner's trajectory in episode $k$ and sampling of $c_k$. 
	Thus, it suffices to bound its second moment according to Freedman's inequality.
	Note that
	\begin{align*}
		&\sumk\E_{c_k}\sbr{\E_k\sbr{\inner{\tilN_k}{c_k}^2}} \leq 2\sumk\E_{c_k}\sbr{\inner{q_k}{\h{c_k}}} = 2\inner{q_k}{\h{c}}\tag{\pref{lem:deviation_loop_free}}\\
		&= 2\underbrace{\sumk\inner{q_k}{\h{(c-\hatc_k)}}}_{\zeta_1} + 2\underbrace{\sumk\inner{q_k-\hatq_k}{\h{\hatc_k}}}_{\zeta_2} + 2\underbrace{\sumk\inner{\hatq_k}{\h{\hatc_k}}}_{\zeta_3}.
	\end{align*}
	We can bound $\zeta_1$ as follows with probability at least $1-\delta$:
	\begin{align*}
		\zeta_1 &\leq H\sumk\sumsa q_k(s, a)(c(s, a)-\hatc_k(s, a)) \tag{$\hatc_k(s, a)\leq c(s, a)$ by \pref{lem:cost-bound}}\\
		&= \tilO{H^2\sumk\sumsa x_k(s, a)\rbr{ \sqrt{\frac{c(s, a)}{\Nc(s, a)}} + \frac{1}{\Nc(s, a)} } } \tag{$q_k(s, a)\leq Hx_k(s, a)$ and \pref{lem:cost-bound}}\\
		&= \tilO{ H^2\sumk\sumsa\Ind_k(s, a) \rbr{ \sqrt{\frac{c(s, a)}{\Nc(s, a)}} + \frac{1}{\Nc(s, a)} } + H^2SA  } \tag{\pref{lem:e2r}}\\
		&= \tilO{ H^2SA\sqrt{K} } = \tilO{H^4S^2A^2 + K}, \tag{AM-GM inequality}
	\end{align*}
	where the last inequality is by $\sumk\sumsa\frac{\Ind_k(s, a)}{\sqrt{\Nc(s, a)}} = \tilO{\sumsa\sqrt{\N^c_K(s, a)}} = \tilO{SA\sqrt{K}}$.
	To bound $\zeta_2$, we apply \pref{lem:transition-bias} with costs $\{\frac{\h{\hatc_k}}{H}\}_k$ to obtain with probability at least $1-4\delta$:
	\begin{align*}
		\zeta_2 &= \tilO{ H\sqrt{S^2A\rbr{\sumk \inner{q_k}{\h{c_k}} +H^3\sqrt{K} } } + H^4S^2A } \tag{$\frac{1}{H}(\h{\hatc_k})(s, a) \leq c_k(s, a)$}\\
		&= \tilO{ \sqrt{S^2AH^5K} + H^4S^2A } = \tilO{K + H^5S^2A} \tag{AM-GM inequality}.
	\end{align*}
	Finally, by \pref{lem:deviation_loop_free}, \pref{lem:fast-action} and $\sumsa \hatq_k(s, a)\leq T$, we have $\zeta_3=\sumk\inner{\hatq_k}{J^{P_k, \tilpi_k, \hatc_k}} = \tilo{DTK}$.
	Putting everything together, we have: $\sumk\E_{c_k}\sbr{\E_k\sbr{\inner{\tilN_k}{c_k}^2}} = \tilO{ H^5S^2A^2 + DTK } $.
	Hence, by Freedman's inequality with $\lambda=\frac{1}{\sqrt{DTK+H^5S^2A^2}}\leq\frac{1}{H}$, we have with probability at least $1-\delta$:
	\begin{align*}
		\sumk\inner{\tilN_k}{c_k} - \inner{q_k}{c} = \tilO{\sqrt{DTK} + H^3SA }.
	\end{align*}
	For the second term in \pref{eq:reg_decomposition}, by \pref{lem:cost-bound} and the definition of $\A^c_k$:
	\begin{align*}
		\sumk\inner{q_k}{c - \hatc_k} &=\tilO{ \sumk\sumsa q_k(s, a)\sqrt{\frac{c(s, a)}{\Nc(s, a)}} + \sumk\sumsa\frac{q_k(s, a)}{\Nc(s, a)} }\\
		&=\tilO{ \sumk\sumsa q_k(s, a)\sqrt{\frac{\hatc_k(s, a)}{\Nc(s, a)}} + \sumk\sumsa q_k(s, a)\sqrt{\frac{c(s, a)-\hatc_k(s, a)}{\Nc(s, a)}} + H\sumk\sumsa\frac{x_k(s, a)}{\Nc(s, a)} }\tag{$\sqrt{x+y}\leq\sqrt{x}+\sqrt{y}$, $q_k(s, a)\leq Hx_k(s, a)$}\\
		&=\tilO{ \sumk\sumsa q_k(s, a)\sqrt{\frac{\hatc_k(s, a)}{\Nc(s, a)}} + \sqrt{K} + H^2S^2A^2 },
	\end{align*}
	where in the last line we use with probability at least $1-\delta$:
	\begin{align*}
		&\sumk\sumsa q_k(s, a)\sqrt{\frac{c(s, a)-\hatc_k(s, a)}{\Nc(s, a)}} + H\sumk\sumsa\frac{x_k(s, a)}{\Nc(s, a)}\\
		 &=\tilO{H\sumk\sumsa\frac{x_k(s, a)}{\Nc(s, a)^{3/4}}} \tag{$q_k(s, a)\leq Hx_k(s, a)$, \pref{lem:cost-bound}, and $\frac{1}{\Nc(s, a)}\leq \frac{1}{\Nc(s, a)^{3/4}}$} \\
		 &=\tilO{H\sumk\sumsa\frac{\Ind_k(s, a)}{\Nc(s, a)} + HSA } \tag{\pref{lem:e2r}}\\
		 &=\bigO{SAHK^{1/4}} = \bigO{ H^2S^2A^2 + \sqrt{K} }. \tag{AM-GM inequality}
	\end{align*}
	It is left to bound $\sumk\sumsa q_k(s, a)\sqrt{\frac{\hatc_k(s, a)}{\Nc(s, a)}}$, for which we need to discuss the type of the feedback model.
	In the full information setting, we have $\Nc(s, a)=\max\{1, k-1\}$. %and with probability at least $1-4\delta$:
	We decompose it into two terms:
	\begin{align*}
		\sumk\sumsa q_k(s, a)\sqrt{\frac{\hatc_k(s, a)}{\Nc(s, a)}} = \sumk\sumsa \hatq_k(s, a)\sqrt{\frac{\hatc_k(s, a)}{\Nc(s, a)}} + \sumk\sumsa (q_k(s, a) - \hatq_k(s, a))\sqrt{\frac{\hatc_k(s, a)}{\Nc(s, a)}}.
	\end{align*}
	For the first term, by Cauchy-Schwarz inequality, $\sumsa\hatq_k(s, a)\leq T$, and $\inner{\hatq_k}{\hatc_k}\leq H_2=\bigO{D}$ (\pref{lem:fast-action}):
	\begin{align*}
		\sumk\sumsa \hatq_k(s, a)\sqrt{\frac{\hatc_k(s, a)}{\Nc(s, a)}} &= \sqrt{\sumk\sumsa\hatq_k(s, a)}\sqrt{\sumk \frac{\inner{\hatq_k}{\hatc_k}}{\max\{1, k-1\}} } = \tilO{\sqrt{DTK}}.
	\end{align*}
	For the second term, we apply \pref{lem:transition-bias} with costs $\{\sqrt{\hatc_k/\N^c_k}\}_k$: with probability at least $1-4\delta$,
	\begin{align*}
		&\sumk\sumsa (q_k(s, a) - \hatq_k(s, a))\sqrt{\frac{\hatc_k(s, a)}{\Nc(s, a)}} \\
		&= \tilO{ \sqrt{S^2A\rbr{ \sumk\sumsa h\cdot q_k(s, a)\sqrt{\frac{\hatc_k(s, a)}{\Nc(s, a)}} + H^3\sqrt{K} }} + H^3S^2A }\\
		&= \tilO{ \sqrt{ S^2A\rbr{ \sumk \frac{H^2}{\sqrt{k}} + H^3\sqrt{K}} } + H^3S^2A } \tag{$h\leq H$ and $\sumsa q_k(s, a)\leq H$}\\
		&= \tilO{ \sqrt{H^3S^2A\sqrt{K}} + H^3S^2A } = \tilO{\sqrt{K} + H^3S^2A}. \tag{$\sumk\frac{1}{\sqrt{k}}=\bigO{\sqrt{K}}$ and AM-GM inequality}
	\end{align*}
	Therefore, $\sumk\sumsa q_k(s, a)\sqrt{\frac{\hatc_k(s, a)}{\Nc(s, a)}} = \tilO{\sqrt{DTK} + H^3S^2A}$ in the full information setting.
	
	In the bandit feedback setting, denote by $x_k(s, a, h)$ the probability that the first visit to $(s, a)$ is at time step $h$ under transition $\tilP$ and policy $\tilpi_k$.
	Then, we have $q_k(s, a)=\sum_{h=1}^Hx_k(s, a, h)q_{k, (s, a, h)}(s, a)$, and
	\begin{align*}
		&\sumk\sumsa q_k(s, a)\sqrt{\frac{\hatc_k(s, a)}{\Nc(s, a)}}=\sumk\sumsa\sqrt{\frac{q_k(s, a)}{\Nc(s, a)}}\sqrt{\sum_{h=1}^Hx_k(s, a, h)q_{k, (s, a, h)}(s, a)\hatc_k(s, a)}\\
		&= \sumk\sumsa\sqrt{\frac{q_k(s, a)}{\Nc(s, a)}}\sqrt{\sum_{h=1}^Hx_k(s, a, h)\rbr{\hatq_{k, (s, a, h)}(s, a)\hatc_k(s, a) + (q_{k, (s, a, h)}(s, a) - \hatq_{k, (s, a, h)}(s, a))\hatc_k(s, a) }}\\
		&\leq \underbrace{\sumk\sumsa\sqrt{q_k(s, a)}\sqrt{\frac{\sum_{h=1}^Hx_k(s, a, h)\hatq_{k, (s, a, h)}(s, a)\hatc_k(s, a)}{\Nc(s, a)}} }_{\zeta_4} \\
		&\qquad\qquad + \underbrace{\sumk\sumsa\sqrt{\frac{q_k(s, a)}{\Nc(s, a)}}\sqrt{\sum_{h=1}^Hx_k(s, a, h)\abr{q_{k, (s, a, h)}(s, a) - \hatq_{k, (s, a, h)}(s, a)}\hatc_k(s, a) }}_{\zeta_5} .
	\end{align*}
	For $\zeta_4$, first notice that
	\begin{align}
		\sumk\sumsa q_k(s, a) &= \sumk\sumsa \hatq_k(s, a) + \sumk\sumsa q_k(s, a) - \hatq_k(s, a) \notag\\
		&=\tilO{TK} + H\sum_{(s, a), s', h\in \unk}q_k(s, a, h)\epsilon^{\star}_k(s, a, s') \tag{$\sumsa\hatq_k(s, a)\leq T$ and \pref{lem:om-diff}}\\
		%&= \tilO{TK + HS\sumsa\frac{q_k(s, a)}{\sqrt{\Nk(s, a)}}} \tag{by definition of $\epsilon^\star_k$ in \pref{lem:transition-bound}}\\
		&= \tilO{TK + HS\sum_{(s, a)} \frac{q_k(s, a)}{\sqrt{\Nk(s, a)}} } \tag{by the definition of $\epsilon^\star_k$ in \pref{lem:transition-bound}}\\
		&\overset{\text{(i)}}{=} \tilO{TK + HS\sqrt{SAHK} + H^2SA} \overset{\text{(ii)}}{=} \tilO{ TK + H^3S^3A }, \label{eq:sum qk}
	\end{align}
	where in (ii) we apply AM-GM inequality, and in (i) we apply with probability at least $1-\delta$:
	\begin{align}
		&\sumk\sumsa\frac{q_k(s, a)}{\sqrt{\Nk(s, a)}} = \tilO{ \sumk\sumsa\frac{\tilN_k(s, a)}{\sqrt{\Nk(s, a)}} + H} \tag{\pref{lem:e2r}}\\ 
		&= \tilO{ \sumk\sumsa\frac{\tilN_k(s, a)}{\sqrt{\N^+_{k+1}(s, a)}} + H\sumk\sumsa\rbr{\frac{1}{\sqrt{\Nk(s, a)}} - \frac{1}{\sqrt{\N_{k+1}^+(s, a)}}} + H } \tag{$\tilN_k(s, a)\leq H$}\\
		&= \tilO{\sumsa\sqrt{\N^+_{K+1}(s, a)} + HSA} = \tilO{ \sqrt{SAHK} + HSA }. \label{eq:q/N}
	\end{align}
	Moreover, by \pref{lem:e2r}, with probability at least $1-\delta$,
	\begin{equation}
		\label{eq:x/Nc}
		\sumk\sumsa\frac{x_k(s, a)}{\Nc(s, a)}=\bigO{\sumk\sumsa\frac{\Ind_k(s, a)}{\Nc(s, a)} + 1}=\tilO{ SA }.
	\end{equation}
	Therefore,
	\begin{align*}
		\zeta_4 &= \sumk\sumsa\sqrt{q_k(s, a)}\sqrt{\frac{\sum_{h=1}^Hx_k(s, a, h)\hatq_{k, (s, a, h)}(s, a)\hatc_k(s, a)}{\Nc(s, a)}}\\
		&= \bigO{\sqrt{\sumk\sumsa q_k(s, a)}\sqrt{D\sumk\sumsa\frac{x_k(s, a)}{\Nc(s, a)} } } \tag{Cauchy-Schwarz inequality, $\hatq_{k, (s, a, h)}(s, a)\hatc_k(s, a)\leq Q^{\tilpi_k, P_k, \hatc_k}(s,a)=\tilo{D}$, and $\sum_{h=1}^Hx_k(s, a, h)=x_k(s, a)$}\\
		&= \tilO{ \sqrt{TK + H^3S^3A}\sqrt{DSA} } = \tilO{ \sqrt{DTSAK} + H^2S^2A }. \tag{\pref{eq:sum qk}, \pref{eq:x/Nc} and $D\leq H$}
	\end{align*}
	
	For $\zeta_5$, we have
	\begin{align*}
		&\sumk\sumsa\sqrt{\frac{q_k(s, a)}{\Nc(s, a)}}\sqrt{\sum_{h=1}^Hx_k(s, a, h) \abr{q_{k, (s, a, h)}(s, a) - \hatq_{k, (s, a, h)}(s, a)}\hatc_k(s, a) }\\
		&\leq \sumk\sumsa\sqrt{\frac{q_k(s, a)}{\Nc(s, a)}}\sqrt{\sum_{h=1}^Hx_k(s, a, h)\inner{|q_{k, (s, a, h)}-\hatq_{k, (s, a, h)}|}{\hatc_k} }\\
		&\leq \sumk\sumsa\sqrt{\frac{q_k(s, a)}{\Nc(s, a)}}\sqrt{ H\sum_{h=1}^Hx_k(s, a, h)\sum_{s', a', s'', h'\in\unk}q_{k, (s, a, h)}(s', a', h')\epsilon^\star_k(s', a', s'') } \tag{\pref{lem:om-diff}}\\
		&= \tilO{ \sumk\sumsa\sqrt{\frac{q_k(s, a)}{\Nc(s, a)}}\sqrt{ H\sum_{h=1}^Hx_k(s, a, h)\sum_{s', a', s'', h'\in\unk}\frac{q_{k, (s, a, h)}(s', a', h')}{\sqrt{\Nk(s', a')}} } } \tag{by definition of $\epsilon^\star_k$ in \pref{lem:transition-bound}}\\
		&\leq \sqrt{ H\sumk\sumsa\frac{x_k(s, a)}{\Nc(s, a)} }\sqrt{H\sumk\sumsa\sum_{s', a', s'', h'\in\unk}\frac{q_k(s', a', h')}{\sqrt{\Nk(s', a')}} } \tag{Cauchy-Schwarz inequality, $q_k(s, a)\leq Hx_k(s, a)$, and $\sum_{h=1}^H x_k(s, a, h)q_{k,(s, a, h)}(s', a', h') \leq q_k(s', a', h')$ }\\
		&=\tilO{ \sqrt{HSA}\sqrt{HS^2A\sumk\sumsa[s', a']\frac{q_k(s', a')}{\sqrt{\Nk(s', a')}}} } = \tilO{ \sqrt{H^2S^3A^2\rbr{\sqrt{SAHK} + HSA }} }. \tag{\pref{eq:x/Nc} and \pref{eq:q/N}}\\
		&= \tilO{\sqrt{SAK + H^5S^6A^4} } = \tilO{\sqrt{SAK} + H^3S^3A^2}. \tag{AM-GM inequality and $\sqrt{x+y}\leq\sqrt{x}+\sqrt{y}$}
	\end{align*}
	Putting everything together, we have the second term in \pref{eq:reg_decomposition}  bounded by:
	\begin{align*}
		\sumk\inner{q_k}{c - \hatc_k} &=\tilO{ \sumk\sumsa q_k(s, a)\sqrt{\frac{\hatc_k(s, a)}{\Nc(s, a)}} + \sqrt{K} + H^2S^2A^2 }\\ 
		&= \begin{cases}
			\tilO{\sqrt{DTK} + H^3S^2A^2},& \text{full information setting,} \\
			\tilO{\sqrt{DTSAK} + H^3S^3A^2},& \text{bandit feedback setting.}
		\end{cases} 
	\end{align*}
	For the third term in \pref{eq:reg_decomposition}, by \pref{lem:transition-bias} with costs $\{\hatc_k\}_k$ and \pref{lem:deviation_loop_free}, we have with probability at least $1-4\delta$:
	\begin{align*}
		\sumk\inner{q_k-\hatq_k}{\hatc_k} &= \tilO{ \sqrt{S^2A\sumk\E_k\sbr{\inner{\tilN_k}{\hatc_k}^2} + H^3S^2A\sqrt{K} } + H^3S^2A }\\
		&= \tilO{ \sqrt{S^2A\sumk\inner{q_k}{\hatc_k \odot Q^{\tilpi_k, \hatc_k}} } + \sqrt{K} + H^3S^2A }. \tag{$\sqrt{x+y}\leq\sqrt{x}+\sqrt{y}$ and AM-GM inequality}
	\end{align*}
	Note that:
	\begin{align*}
		&\sumk\inner{q_k}{\hatc_k \odot Q^{\tilpi_k, \hatc_k}} = \sumk\inner{q_k}{\hatc_k\odot(Q^{\tilpi_k, \hatc_k}-Q^{P_k, \tilpi_k, \hatc_k})} + \sumk\inner{q_k}{\hatc_k\odot Q^{P_k, \tilpi_k, \hatc_k}}.
	\end{align*}
	For the first term, note that by \pref{lem:om-diff}:
	\begin{align*}
		&Q^{\tilpi_k, \hatc_k}(s, a, h) - Q^{P_k, \tilpi_k, \hatc_k}(s, a, h) = \inner{q_{k, (s, a, h)} - \hatq_{k, (s, a, h)}}{\hatc_k} \leq \sum_{(s', a'), s'', h'\in\unk}q_{k, (s, a, h)}(s', a', h')\epsilon^{\star}_k(s', a', s'')H.
	\end{align*}
	Therefore,
	\begin{align*}
		&\sumk\inner{q_k}{\hatc_k\odot(Q^{\tilpi_k, \hatc_k}-Q^{P_k, \tilpi_k, \hatc_k})} \leq \sumk\sum_{(s, a), h}q_k(s, a, h)\sum_{(s', a'), s'', h'\in\unk}q_{k, (s, a, h)}(s', a', h')\epsilon^{\star}_k(s', a', s'')H\\
		&\leq H^2\sumk\sum_{(s', a'), s'', h'\in\unk}q_k(s', a', h')\epsilon^{\star}_k(s', a', s'') \tag{$\sumsa q_k(s, a, h)q_{k, (s, a, h)}(s', a', h')=q_k(s', a', h')$ for $h\leq h'$}\\
		&= \tilO{ H^2S\sumk\sumsa[s', a']\frac{q_k(s', a')}{\sqrt{\Nk(s', a')}} } = \tilO{H^2S(\sqrt{SAHK} + HSA)} = \tilO{ H^5S^3A + K }. \tag{\pref{lem:transition-bound}, \pref{eq:q/N}, and AM-GM inequality}
	\end{align*}
	For the second term, with probability at least $1-4\delta$,
	\begin{align*}
		\sumk\inner{q_k}{\hatc_k\odot Q^{P_k, \tilpi_k, \hatc_k}} &= \bigO{D\sumk\inner{q_k}{\hatc_k}} \tag{$Q^{P_k,\tilpi_k,\hatc_k}(s, a, h)\leq H_2 + 1 =\bigO{D}$ by \pref{lem:fast-action}}\\
		&= \bigO{D\sumk\inner{\hatq_k}{\hatc_k} + D\sumk\inner{q_k-\hatq_k}{\hatc_k}}\\
		&= \tilO{ D^2K + D\sqrt{S^2A\rbr{\sumk\inner{q_k}{\h{\hatc_k}} + H^3\sqrt{K} }} + DH^3S^2A } \tag{$\inner{\hatq_k}{\hatc_k}\leq D$ (\pref{lem:fast-action}) and \pref{lem:transition-bias} with costs $\{\hatc_k\}_k$}\\
		&= \tilO{ D^2K + D\sqrt{H^3S^2AK} + DH^3S^2A } = \tilO{D^2K + DH^3S^2A}. \tag{$\sumk\inner{q_k}{\h{c_k}}\leq H^2K$ and AM-GM inequality}
	\end{align*}
	Putting everything together, we obtain: 
	\begin{align*}
		\sumk\inner{q_k-\hatq_k}{\hatc_k} = \tilO{\sqrt{S^2A(H^5S^3A^2 + K + D^2K + DH^3S^2A)} + \sqrt{K} + H^3S^2A} = \tilO{DS\sqrt{AK} + H^3S^3A^2}.
	\end{align*}
%	\begin{align*}
%		&\leq \sumk\sum_{(s, a), h}q_k(s, a, h)\sum_{(s', a'), s'', h'\in\unk}q_{k, (s, a, h)}(s', a', h')\epsilon^{\star}_k(s', a', s'')H + \bigO{D\sumk\inner{q_k}{\hatc_k}} \tag{\pref{lem:om-diff} and $Q^{P_k,\tilpi_k,\hatc_k}(s, a, h)\leq H_2=\bigO{D}$}\\
%		&\leq H^2\sumk\sum_{(s', a'), s'', h'\in\unk}q_k(s', a', h')\epsilon^{\star}_k(s', a', s'') + \bigO{D\sumk\inner{\hatq_k}{\hatc_k} + D\sumk\inner{q_k-\hatq_k}{\hatc_k}} \tag{$q_k(s', a', h')=\sumsa q_k(s, a, h)q_{k, (s, a, h)}(s', a', h')$ for $h\leq h'$}\\
%		&= \tilO{ H^2S\sumk\sumsa[s', a']\frac{q_k(s', a')}{\sqrt{\Nk(s', a')}} + D^2K + D\sqrt{S^2A\sumk\inner{q_k}{\h{\hatc_k}}} }. \tag{$\inner{\hatq_k}{\hatc_k}\leq D$ and \pref{lem:transition-bias}}\\
%		&\leq \tilO{ H^2S\sqrt{SAHK} + D^2K + D\sqrt{S^2AH^2K}} = \tilO{D^2K},
%	\end{align*}
%	where in the last inequality we apply \pref{eq:q/N}.
	Substituting everything back to \pref{eq:reg_decomposition} and applying \pref{lem:loop-free-iid}, we have with probability at least $1-\delta$,
	\begin{align*}
		\tilR_K = \begin{cases} 
			\tilO{\sqrt{DTK} + DS\sqrt{AK} + H^3S^3A^2}, & \text{full information setting,}\\
			\tilO{\sqrt{DTSAK} + DS\sqrt{AK} + H^3S^3A^2}, & \text{bandit feedback setting.}
		\end{cases}
	\end{align*}
	%	\begin{align*}
%		\sumk\inner{\tiloptq}{\hatc_k-c} &= \tilO{ \sumk\sumsah\optq(s, a, h)\rbr{ \sqrt{\frac{c(s, a)}{k}}\ + \frac{1}{k} } }\\
%		&= \tilO{ \sqrt{\sumk\sumsah\optq(s, a, h)c(s, a) }\sqrt{ \sumk\sumsah\frac{\optq(s, a, h)}{k} } + \T\ln K}\\
%		&= \tilO{ \sqrt{D\T K\ln K} + \T\ln K }.
%	\end{align*}
This completes the proof.
\end{proof}

\section{Learning without knowing SSP-diameter}
\label{app:tune-d}
% !TEX root = main.tex

\label{sec:tund-d}

In this section, we present a general idea on how to learn without knowing the SSP-diameter $D$.
To give a concrete example, we apply this idea to \pref{alg:full-unknown} and obtain an algorithm (\pref{alg:full-unknown-tune-d}) that achieves the same regret without the knowledge of $D$. The same idea can also be applied to other algorithms proposed in this paper (details omitted).
The main ideas of our proposed algorithm are as follows:
\begin{enumerate}
	\item We partition $K$ episodes into $M_p$ phases. In phase $m$, we learn on a virtual MDP $\Pi(M, \calS_m)$.
	We call states in $\calS_m\subseteq\calS$ known states, and states in $\calS\setminus\calS_m$ unknown states which are all treated as goal states.
	That is, $\Pi(M, \calS_m)$ is obtained by modifying the transition and cost function of $M$ in unknown states so that $P(s|s, a)=1$ and $c(s, a)=0$ for any $s\in\calS\setminus\calS_m$.
	The correspondence between $M$ and $\Pi(M, \calS_m)$ is as follows: every time we reach a state $s\in\calS\setminus\calS_m$ in $\Pi(M, \calS_m)$, we run Bernstein-SSP until the goal state is reached.
	In phase $m$, we run \pref{alg:full-unknown} on $\tilPi(M, \calS_m)$ with $H_2=\ceil{2\max_{s\in\calS_m}\tilD_s}$, where $\tilPi(M, \calS_m)$ is the loop-free reduction of $\Pi(M, \calS_m)$ and $\tilD_s$ is defined below.
	
	\item In \citep[Appendix I.3]{rosenberg2020adversarial}, they show that we can get an estimate $\tilD_s$ of $T^{\pi^f}(s)$ such that $T^{\pi^f}(s)\leq\tilD_s=\bigO{D}$ for any $s\in\calS$ as long as we run Bernstein-SSP starting from state $s$ for some $L\geq \max\big\{\frac{2400D^2}{T^{\pi_f}(s_0)^2}\ln^3\frac{4K}{\delta}, S^2A\sqrt{D}\ln^2\frac{KDSA}{\delta}\big\}$ episodes.
	We thus concretely define known states as follows: define $N_f(s)$ as the number of episodes where state $s$ is the first visited unknown state in $\Pi(M, \calS_m)$ for any $m$. A state $s$ is known if $N_f(s)\geq L$.
	By the definition of $\Pi(M, \calS_m)$ and known states, if state $s$ is known, then we have run Bernstein-SSP for at least $L$ episodes starting from $s$ and thus $\tilD_s$ satisfying $T^{\pi^f}(s)\leq\tilD_s$ can be computed.
	When a new known state is found, the algorithm enters into the next phase (that is, increment $m$).
	This construction also implies that $\calS_m\subset\calS_{m+1}, M_p\leq S$, and the diameter of $\Pi(M, \calS_m)$ is upper bounded by $\max_{s\in\calS_m}\tilD_s$.
	%Our algorithm implement these steps as follows: at the beginning of learning, we first run Bernstein-SSP for $L$ episodes to get $\tilD_{s_0}$.
	%Then, we let $\calS_1=\{s_0\}$ and begin the phase $1$ of learning.
	%If at the end of an episode in phase $m$, $N_f(s)=L$, then we start a new phase $m+1$ with $\calS_{m+1}=\calS_m\cup\{s\}$.
	
	\item When running \pref{alg:full-unknown}, we set the learning rate $\eta$ using $\tilD_{s_0}$ and the parameter $\lambda$ (for the skewed occupancy measure) using a doubling trick.
	Specifically, we replace the diameter $D$ in $\lambda$ by an estimate $D_j$ to obtain $\lambda_j$, and run \pref{alg:full-unknown} with $\lambda_j$ in place of $\lambda$ (starting with $D_1=2\tilD_{s_0}$). We double the estimate every time we realize that $D_j$ is not an upper bound of $D$: suppose at the end of an episode, state $s$ becomes a known state, and $\tilD_s>D_j$, then we set $D_{j+1}=2\tilD_s$ and increase $j$ by $1$. Denote by $I_p$ the value of $j$ at the end of episode $K$. Clearly, $I_p=\bigO{\log_2D}$.

	\item The last important change is that instead of reinitializing the confidence sets and counters in each phase, we directly inherit them from the previous phase. Denote by $\hatq^m_k$ the occupancy measure computed by \pref{alg:full-unknown} in $\tilPi(M, \calS_m)$ in episode $k$, and by $q^m_k$ the occupancy measure of executing $\pi_{\hatq^m_k}$ in $\tilPi(M, \calS_m)$. Following the proof of \pref{lem:transition-bias}, it is straightforward to verify that for any cost functions $\{c_k\}_{k=1}^K$ in $[0, 1]$, the following holds,
	\begin{align}
		\label{eq:full unknown q-hatq}
		\sumk\abr{\inner{q^m_k-\hatq^m_k}{c_k}} &\leq 32\sqrt{S^2A\ln(HK)\rbr{\sumk\inner{q^m_k}{\h{c_k}} + \tilO{H^3\sqrt{K}}}} + \tilO{H^3S^2A}.
	\end{align}
\end{enumerate}

\begin{algorithm}[t]
\caption{Learning SSP with unknown diameter}
\label{alg:full-unknown-tune-d}
	
	\textbf{Input:} Upper bound on expected hitting time $T$, horizon parameter $H_1$, confidence level $\delta$. 
	
	\textbf{Define:} $\eta=\min\cbr{\frac{1}{8}, \sqrt{\frac{ST}{\tilD_{s_0}K}}}, \lambda_j = 4\sqrt{\frac{S^2A}{D_jTK}}, L=2400\sqrt{AK}\ln^3\frac{4KH_1SA}{\delta}$.
	
	\textbf{Initialization:} $\N_1(s, a)=\M_1(s, a, s')=0$ for all $(s, a, s')\in\SA\times(\calS\cup\{g\})$.

	\textbf{Initialization:} for each state $s$, a Bernstein-SSP instance $\calB_s$ that uses $s$ as the initial state and treats all costs as $1$.
	
	\textbf{Initialization:} Compute $\tilD_{s_0}$ by executing $\calB_{s_0}$ for the first $L$ episodes.
	
	\textbf{Initialization:} $\calS_1=\{s_0\}, m=1, j=1$, $D_1=2\tilD_{s_0}$.
	
	\textbf{Initialization:} $A = \calA(T, H_1, \delta; \eta, \lambda_1, \N, \M, \calS_m)$, where $\calA$ is a variant of \pref{alg:full-unknown} which also takes $\eta, \lambda, \N, \M$ and $\calS_m$ as inputs.
	
	\textbf{Initialization:} $N_f(s)=L\Ind\{s=s_0\},\forall s\in\calS$.
	
	\For{$k=L+1,\ldots,K$}{
		
		Execute $A$ on $\Pi(M, \calS_m)$ for episode $k$. %, and denote by $e$ the first unknown state visited in $\Pi(M, \calS_m)$.
		
		\If{$A$ stops at an unknown state $e$}{
%		\If{$e$ exists}{
			Invoke $\calB_e$ (as a new episode for it) and follow its decision until reaching $g$. 

			$N_f(e)\leftarrow N_f(e) + 1$.
			
			\If{$N_f(e)=L$}{
			
				Compute $\tilD_{e}$ using previous data~\citep[Appendix I.3]{rosenberg2020adversarial}.
			
				$\calS_{m+1}\leftarrow \calS_m\cup\{e\}$.
			
				$m\leftarrow m+1$.
				
				\If{$\tilD_e > D_j$}{
					$D_{j+1}\leftarrow 2\tilD_e$
				
					$j\leftarrow j+1$.
				}
				
				$A = \calA(T, H_1, \delta;\, \eta, \lambda_j, \N, \M, \calS_m)$.
			}
		}

	}
\end{algorithm}

We summarize the ideas above in \pref{alg:full-unknown-tune-d}.
Now we proceed to show that \pref{alg:full-unknown-tune-d} ensures the same regret guarantee as \pref{alg:full-unknown} without knowing $D$.
\begin{theorem}
	\label{thm:full-unknown-tune-d}
	If $T\geq\T+1, H_1\geq 8\Tmax\ln K$ and $K\geq 16S^2AH^2$, then with probability at least $1-6\delta$, \pref{alg:full-unknown-tune-d} ensures $R_K=\tilo{\sqrt{S^2ADTK} + H^3S^3A + D^4S^4AH}$.
\end{theorem}
\begin{proof}
	When $K<\max\Big\{DS^4A, \frac{D^4}{T^{\pi_f}(s_0)^4A} \Big\}$, by the regret guarantee of Bernstein-SSP~\cite{cohen2020near}, we have:
	\begin{align*}
		R_K \leq KH_1 + KD + \tilO{ DS\sqrt{AK} + D^{\frac{3}{2}}S^2A } = \tilO{ D^4S^4AH }.
	\end{align*}
	Otherwise, we have $L\geq \max\big\{\frac{2400D^2}{T^{\pi_f}(s_0)^2}\ln^3\frac{4K}{\delta}, S^2A\sqrt{D}\ln^2\frac{KDSA}{\delta}\big\}$ as desired.
	Denote by $N^m_k, \tilN^m_k$ the number of steps taken in $\Pi(M, \calS_m), \tilPi(M, \calS_m)$ in episode $k$ respectively.
	We have $N^m_k<N_k$ only if the learner reaches an unknown state.
	Hence, There are at most at most $SL$ episodes such that $N^m_k<N_k$, and $N_k-N^m_k$ is the number of steps of executing Bernstein-SSP after reaching an unknown state in episode $k$.
	Denote by $q^m_{\pi}$ the occupancy measure of executing policy $\pi$ in $\Pi(M, \calS_m)$ or $\tilPi(M, \calS_m)$ based on the policy $\pi$.
	%Treat $m, j$ as a function of $k$ when necessary. 
	We decompose the regret as follows:
	\begin{align*}
		\sumk\inner{N_k - \optq}{c_k} &= \sumk\inner{N^m_k - \optq}{c_k} + \sumk\inner{N_k - N^m_k}{c_k}\\
		&\leq \sumk\inner{N^m_k - \optq^m}{c_k} + \tilO{S(DL + DS\sqrt{AL} + D^{3/2}S^2A)} \tag{by $\optq^m(s, a)\leq \optq(s, a)$ and the regret guarantee of Bernstein-SSP}\\
		&\leq \sumk\inner{\tilN^m_k - \tiloptq^m}{c_k} + \tilO{DSL + D^2S^3A}, \tag{by loop-free reduction (\pref{lem:loop-free}) on $\Pi(M, \calS_m)$}
	\end{align*}
	Note that using the last inequality above, we can already get a parameter-free algorithm with an extra $S$ dependency by completely restarting \pref{alg:full-unknown} in every new phase.
	To obtain a tighter bound, we need a more careful analysis on the sum of regret of all phases.
	Denote by $\calI_m$ the set of episodes in phase $m$, and by $\calI'_j$ the set of episodes using parameter $\lambda_j$.
	Note that we can treat $j$ as a function of phase $m$, and $m$ as a function of episode $k$.
	Following the arguments in the proof of \pref{thm:full-unknown}, with probability at least $1-6\delta$,
	\begin{align*}
		&\sum_{m=1}^{M_p}\sum_{k\in\calI_m}\inner{\tilN^m_k - \tiloptq^m}{c_k}\\
		&= \sumk\inner{\tilN^m_k - q^m_k}{c_k} + \sumk\inner{q^m_k-\hatq^m_k}{c_k} + \sum_{m=1}^{M_p}\sum_{k\in\calI_m}\inner{\hatq^m_k - \tiloptq^m}{c_k}\\
		&\leq \lambda_{I_p}\sumk\inner{q^m_k}{\h{c_k}} + \tilO{\frac{1}{\lambda_{I_p}}} + 32\sqrt{S^2A\ln(HK)\rbr{\sumk\inner{q^m_k}{\h{c_k}} + \tilO{H^3\sqrt{K}} }} + \tilO{H^3S^2A} \tag{Freedman's inequality and \pref{eq:full unknown q-hatq}} \\
		&\qquad + \sum_{m=1}^{M_p}\rbr{ \tilO{\frac{T}{\eta} + \eta\sum_{k\in\calI_m}\inner{\tiloptq^m}{c_k} } + 2\lambda_j\sum_{k\in\calI_m}\inner{\tiloptq^m}{\h{c_k}} - \lambda_j\sum_{k\in\calI_m}\inner{\hatq^m_k}{\h{c_k}} } \tag{by \pref{eq:full unknown omd}}\\
		&\leq \tilO{\sqrt{SDTK}} + \sum_{j=1}^{I_p}\rbr{2\lambda_j\sum_{k\in\calI'_j}\inner{\tiloptq^m}{\h{c_k}} + \tilO{\frac{S^2A}{\lambda_j}} + \lambda_j\sum_{k\in\calI'_j}\inner{q^m_k-\hatq^m_k}{\h{c_k}} } + \tilO{H^3S^2A} \tag{$|M_p|\leq S$,$\sumk\inner{q^m_{\tiloptpi}}{c_k}\leq \sumk\inner{q_{\tiloptpi}}{c_k}\leq DK$, and AM-GM inequality} \\
		&\overset{\text{(i)}}{=} \tilO{\sqrt{SDTK} + \sum_{j=1}^{I_p} \sqrt{S^2AD_jTK} + H^3S^2A} = \tilO{\sqrt{S^2ADTK} + H^3S^2A}.
	\end{align*}
	In (i), we denote by $J_k^{m,\pi}(s, h)$ the expected cost of policy $\pi$ starting from state $(s, h)$ w.r.t cost $c_k$ and $\tilPi(M, \calS_m)$.
	Then, $\sumk J_k^{m,\tiloptpi}(s, h) \leq \sumk J_k^{\tiloptpi}(s, h)\Ind\{s\in\calS_m\}\leq \sumk (J_k^{\optpi}(s) + 3D_j)\Ind\{s\in\calS_m\} \leq 4D_jK$ for all $(s, h)$, and by \pref{lem:deviation_loop_free},
	\begin{align*}
		\sum_{k\in\calI'_j}\inner{q^m_{\tiloptpi}}{\h{c_k}} \leq \sumk \inner{q^m_{\tiloptpi}}{\h{c_k}} = \sumk\inner{q^m_{\tiloptpi}}{J_k^{m,\tiloptpi}} = \bigO{D_jTK}.
	\end{align*}
	Moreover, by \pref{eq:full unknown q-hatq},
	\begin{align*}
		&\lambda_j\sum_{k\in\calI'_j}\inner{q^m_k-\hatq^m_k}{\h{c_k}} = \lambda_jH\sum_{k\in\calI'_j}\inner{q^m_k-\hatq^m_k}{\frac{\h{c_k}}{H}}\\ 
		&= \tilO{ \lambda_jH\sqrt{S^2A\rbr{\sumk\inner{q^m_k}{\h{c_k}} + H^3\sqrt{K} }} + \lambda_jH^4S^2A}\\
		&= \tilO{ \lambda_jH\sqrt{S^2AH^3K} + \lambda_jH^4S^2A } = \tilO{H^3S^2A }. \tag{$\sumk\inner{q^m_k}{\h{c_k}}\leq H^2K$, $\lambda_j=4\sqrt{\frac{S^2A}{D_jTK}}$, and $K\geq 16S^2AH^2$}
	\end{align*}
	Combining both cases, we have:
	\begin{align*}
		R_K &= \tilO{ \sqrt{S^2ADTK} + H^3S^2A + DSL + D^2S^3A + D^4S^4AH }\\
		&= \tilO{ \sqrt{S^2ADTK} + H^3S^3A + DS\sqrt{AK} + D^4S^4AH }\\
		&= \tilO{ \sqrt{S^2ADTK} + H^3S^3A + D^4S^4AH }. \tag{$D \leq T$}
	\end{align*}
	This completes the proof.
\end{proof}

\section{Learning without knowing $\T$ or $\Tmax$}
\label{app:tune T}

In this section, we discuss how to instantiate our proposed algorithms without knowing $\T$ or $\Tmax$ (or both).
We apply our ideas to \pref{alg:full-unknown} to give concrete examples, and they are applicable to other proposed algorithms similarly.
The modifications described below can be applied separately or jointly depending on the knowledge we have.
Moreover, they are compatible with ideas in \pref{sec:tund-d} for learning without knowing the SSP-diameter $D$.

\paragraph{Learning without knowing $\T$} We assume knowledge of $\cmin=\min_{s, a, k}c_k(s, a)$ and $\cmin>0$, which is the assumption made in \cite{rosenberg2020adversarial}.
Then by the inequality $\T\leq\frac{T^{\pi^f}(s_0)}{\cmin}$, it suffices to obtain an upper bound of $T^{\pi^f}(s_0)$ to obtain an upper bound of $\T$.
To this end, we first run a Bernstein-SSP instance for $L=\tilo{\sqrt{AK}}$ episodes with uniform costs equal to $1$ and obtain $\tilD_{s_0}$, where both $L$ and $\tilD_{s_0}$ are defined in \pref{sec:tund-d}. Then, we simply run \pref{alg:full-unknown} with $T=\tilD_{s_0}/\cmin$.
Following the arguments in \pref{sec:tund-d}, we know that the extra costs of estimating $\tilD_{s_0}$ in the regret is $\tilo{DL+D^{3/2}S^2A+D^4S^4AH}=\tilo{D\sqrt{AK}+D^4S^4AH}$.
Thus, we obtain the following result:
\begin{theorem}
	If $H_1\geq 8\Tmax\ln K$, and $K\geq 16S^2AH^2$, then with probability at least $1-6\delta$, the algorithm described above ensures $R_K=\tilo{D\sqrt{\frac{S^2AK}{\cmin}} + H^3S^2A + D^4S^4AH}$.
\end{theorem}
Compared to \cite{rosenberg2020adversarial}, the bound above is $\sqrt{\frac{1}{\cmin}}$ better in the dominating term.
When $\cmin=0$, similarly to \cite{rosenberg2020adversarial}, we can solve a modified MDP with perturbed cost functions $c'_k(s, a)=\max\{c_k(s, a),\epsilon\}$, where $\epsilon=K^{-1/3}$.
The bias brought by the perturbation is of order $\tilO{\epsilon\T K}$, and thus the overall regret is $\tilo{K^{2/3}}$ ignoring other parameters and constant terms.
This is asymptomatically better than the $\tilO{K^{3/4}}$ regret in \cite{rosenberg2020adversarial} for $\cmin=0$.
We conclude that \pref{alg:full-unknown} improves over previous work even without knowledge of $\T$.

\paragraph{Learning without knowing $\Tmax$} Similarly to \cite{chen2020minimax}, we run \pref{alg:full-unknown} with $H_1=8(K/S^2A)^{1/6}\ln K$.
Note that when $K\leq \Tmax^6S^2A$, by the regret guarantee of Bernstein-SSP, we have:
\begin{align*}
	R_K \leq KH_1 + KD + \tilO{ DS\sqrt{AK} + D^{\frac{3}{2}}S^2A } = \tilO{ \Tmax^7S^2A }.
\end{align*}
Otherwise, $H_1\geq\Tmax$, and by the regret guarantee of \pref{alg:full-unknown}, we have:
$$R_K=\tilO{\sqrt{S^2ADTK} + H^3S^2A} = \tilO{\sqrt{S^2ADTK}}.$$
Combining these two cases, we have the following result:
\begin{theorem}
	If $T\geq\T+1$, and $K\geq 16S^2AH^2$, then with probability at least $1-6\delta$, running \pref{alg:full-unknown} with $H_1=8(K/S^2A)^{1/6}\ln K$ ensures $R_K=\tilo{\sqrt{S^2ADTK} + \Tmax^7S^2A}$.
\end{theorem}
Applying both modifications to \pref{alg:full-unknown}, we obtain: $R_K = \tilO{ D\sqrt{\frac{S^2AK}{\cmin}} + D^4S^4AK^{1/6} + \Tmax^7S^2A }$, which is still asymptomatically better compared to \cite{rosenberg2020adversarial}.

\section{Concentration Inequalities}
\label{app:concentration}
% !TEX root = main.tex

In this section, for a sequence of random variables $\{X_i\}_{i=1}^{\infty}$ adapted to a filtration $\{\calF_i\}_{i=1}^{\infty}$, we define $\E_i[X_i]=\E[X_i|\calF_i]$.

\begin{lemma}{(Azuma's inequality)}
	\label{lem:azuma}
	Let $X_{1:n}$ be a martingale difference sequence and $\abs{X_i} \leq B$ holds for $i=1,\ldots, n$ and some fixed $B > 0$.
	Then, with probability at least $1-\delta$:
	\begin{align*}
		\left|\sum_{i=1}^nX_i\right| \leq B\sqrt{2n\ln\frac{2}{\delta}}.
	\end{align*}
\end{lemma}

\begin{lemma}{(A version of Freedman's inequality from~\citep{beygelzimer2011contextual})}
	\label{lem:freedman}
	Let $X_{1:n}$ be a martingale difference sequence and $X_i \leq B$ holds for $i=1,\ldots, n$ and some fixed $B > 0$. Denote $V=\sum_{i=1}^n\E_i[X_i^2]$.
	Then, for any $\lambda\in [0, 1/B]$, with probability at least $1-\delta$:
	\begin{align*}
		\sum_{i=1}^nX_i \leq \lambda V + \frac{\ln(1/\delta)}{\lambda}.
	\end{align*}
\end{lemma}

\begin{lemma}{(A version of anytime Bernstein's inequality from~\citep[Theorem D.3]{cohen2020near})}
	\label{lem:bernstein}
	Let $X_{1:n}$ be a sequence of i.i.d. random variables with expectation $\mu$. Assume $X_n\in [0, B]$ almost surely. Then with probability $1-\delta$, the following holds for all $n\geq 1$ simultaneously:
	\begin{align*}
		\left| \sum_{i=1}^n(X_i - \mu) \right| &\leq 2\sqrt{B\mu n\ln\frac{2n}{\delta}} + B\ln\frac{2n}{\delta}.\\
		\left| \sum_{i=1}^n(X_i - \mu) \right| &\leq 2\sqrt{B\sum_{i=1}^nX_i\ln\frac{2n}{\delta}} + 7B\ln\frac{2n}{\delta}.
	\end{align*}
\end{lemma}

\begin{lemma}{(Strengthened Freedman's inequality from~\citep[Theorem 2.2]{lee2020bias})}
	\label{lem:extended-freedman}
	Let $X_{1:n}$ be a martingale difference sequence with respect to a filtration $\calF_1 \subseteq \cdots \subseteq \calF_n$ such that $\E[X_i |\calF_i] = 0$.
	Suppose $B_i \in [1,b]$ for a fixed constant $b$ is $\calF_i$-measurable and such that $X_i \leq B_i$ holds almost surely.
	Then with probability at least $1 -\delta$ we have 
	\[
	    \sum_{i=1}^n  X_i \leq  C\big(\sqrt{8V\ln\left(C/\delta\right)} + 2B^\star  \ln\left(C/\delta\right)\big),
	\]
	where $V = \max\big\{1, \sum_{i=1}^n \E[X_i^2 | \calF_i]\big\}$, $B^\star =\max_{i\in[n]} B_i$, and 
	$C = \ceil{\log_2(b)}\ceil{\log_2(nb^2)}$.
\end{lemma}

\begin{lemma}
	\label{lem:e2r} % expectation to random variable
	Let $\{X_i\}_{i=1}^\infty$ be a sequence of random variables adapted to the filtration $\{\calF_i\}_{i=1}^{\infty}$, and $0\leq X_i\leq B$ almost surely.
	Then with probability at least $1-\delta$, for all $n\geq 1$ simultaneously:
	\begin{align*}
		\sum_{i=1}^n\E_i[X_i] \leq 2\sum_{i=1}^nX_i + 4B\ln\frac{2n}{\delta}.
	\end{align*}
\end{lemma}

%\begin{lemma}{(Fano's inequality~\citep{gerchinovitz2020fano})}\label{lem:fano}
%	For all sequences of $N\geq 2$ probability distributions $P_{1:N}$ on the same measurable space $(\Omega, \calF)$, and all events $A_{1:N}$ forming a partition of $\Omega$,
%	\begin{align*}
%		\frac{1}{N}\sum_{i=1}^NP_i(A_i) \leq \frac{\frac{1}{N}\inf_Q\sum_{i=1}^N\KL(P_i, Q) + \ln 2}{\ln N}.
%	\end{align*}
%\end{lemma}

\end{document}